\documentclass{article} 
\usepackage{iclr2027_conference,times}

\usepackage{amsmath,amsfonts,bm}

\def\eqref#1{equation~\ref{#1}}

\def\1{\bm{1}}

\def\vx{{\bm{x}}}

\DeclareMathAlphabet{\mathsfit}{\encodingdefault}{\sfdefault}{m}{sl}
\SetMathAlphabet{\mathsfit}{bold}{\encodingdefault}{\sfdefault}{bx}{n}

\newcommand{\E}{\mathbb{E}}

\newcommand{\KL}{D_{\mathrm{KL}}}
\newcommand{\Var}{\mathrm{Var}}

\usepackage[
    colorlinks=true,
    linkcolor=red,   
    citecolor=blue,  
    urlcolor=red      
]{hyperref}
\usepackage{url}
\usepackage{thmtools}
\usepackage{amsthm}
\usepackage{graphicx}
\usepackage{booktabs}
\usepackage{multirow}
\usepackage{array}
\usepackage{makecell}
\usepackage{algorithm}
\usepackage{algpseudocode}
\usepackage{subcaption}
\usepackage{xstring}
\usepackage{graphicx}
\usepackage[table]{xcolor}
\usepackage{bm}
\usepackage{titletoc}
\usepackage[dvipsnames]{xcolor}

\newcommand{\shortminus}{\scalebox{0.7}[1.0]{$-$}}
\newcommand{\boldshortminus}{\scalebox{0.7}[1.0]{$\mathbf{-}$}}

\newcommand{\sci}[1]{%
  \StrLeft{#1}{1}[\firstchar]%
  \IfStrEq{\firstchar}{-}
    {\mathrm{e}\shortminus\StrGobbleLeft{#1}{1}}
    {\mathrm{e}#1}%
}

\newcommand{\bsci}[1]{%
  \StrLeft{#1}{1}[\firstchar]%
  \IfStrEq{\firstchar}{-}
    {\mathbf{e}\boldshortminus\mathbf{\StrGobbleLeft{#1}{1}}}
    {\mathbf{e#1}}%
}

\newcommand{\mstd}[3]{(#1\!\pm\!#2)\sci{#3}}

\newcommand{\bmstd}[3]{%
  \mathbf{(#1\!\pm\!#2)}\bsci{#3}%
}
\newcommand{\ie}{i.e.,\ }
\newcommand{\eg}{e.g.,\ }

\title{A PAC-Bayesian view of Generalisation for Physics-Informed Machine Learning}

\author{Thien V. Nguyen \\
Université Jean Monnet Saint-Étienne, CNRS, Institut d'Optique Graduate School,\\
Laboratoire Hubert Curien UMR 5516, Inria, F-42023, Saint-Étienne, France\\
\texttt{van.thien.nguyen@univ-st-etienne.fr} \\
\And
Amaury Habrard \\
Université Jean Monnet Saint-Étienne, CNRS, Institut d'Optique Graduate School,\\
Laboratoire Hubert Curien UMR 5516, Inria, F-42023, Saint-Étienne, France\\
Institut Universitaire de France \\
\texttt{amaury.habrard@univ-st-etienne.fr} \\
\AND
Benjamin Guedj \\
Inria and University College London, \\
France and United Kingdom \\
\texttt{b.guedj@ucl.ac.uk}
}

\declaretheorem[name=Theorem, numberwithin=section]{theorem}

\newtheorem{theorem}{Theorem}
\newtheorem{proposition}[theorem]{Proposition} 
\newtheorem{lemma}[theorem]{Lemma}
\newtheorem{assumption}[theorem]{Assumption}

\newtheorem{corollary}[theorem]{Corollary}
\newtheorem{remark}[theorem]{Remark}

\DeclareMathOperator*{\Prob}{\mathbb{P}}
\newcommand{\kl}{\textrm{kl}}

\definecolor{slightblue}{HTML}{EAF2F8}
\definecolor{slightorange}{HTML}{FDF2E9}
\definecolor{slightgreen}{HTML}{EAF5F1}

\iclrfinalcopy 
\begin{document}

\maketitle

\begin{abstract}
Physics-informed machine learning (PIML) integrates mechanistic knowledge, typically in the form of partial differential equations (PDEs), into data-driven models. Despite strong empirical performance, its statistical generalisation properties remain poorly understood, particularly in the regression setting with unbounded losses. Existing analyses rely largely on approximation or stability arguments and do not fully capture how physical structure influences generalisation from finite data. In this work, we develop a PAC-Bayesian framework for PIML that provides high-probability generalisation guarantees in the presence of potentially unbounded losses. Our analysis leverages the structure of physics-informed objectives to derive component-wise bounds whose complexity scales with the input-gradient energy of each loss, revealing a direct link between physical regularity and generalisation. We instantiate this framework under population-level regularity conditions and introduce a PAC-Bayesian calibration procedure that yields computable input-gradient complexities while explicitly controlling rare large-gradient events through a residual-tail correction. We then adopt a multi-task perspective that jointly treats data fidelity, PDE residuals, initial conditions, and boundary conditions, leading to a refined certificate with a single PAC-Bayesian complexity penalty, thereby avoiding the looseness induced by controlling each component independently. Based on the resulting certificate, we further develop a bound-aware learning procedure that jointly promotes empirical accuracy, proximity to a physics-informed prior, and low input-gradient complexity. Experiments across six PDE benchmarks demonstrate substantially tighter certificates than conservative bounded-loss and sub-Gaussian alternatives, while targeted ablations characterise the effects of posterior-training data, calibration data, gradient envelopes, and hypothesis localisation on the resulting guarantees.
\end{abstract}

\section{Introduction} \label{intro}
Physics-informed machine learning \citep[PIML,][]{Karniadakis_PIML21} integrates prior scientific knowledge, typically expressed through partial differential equations (PDEs), into data-driven models. By constraining learning through physical laws, PIML has shown strong empirical performance in forward and inverse problems, scientific simulation, and hybrid modelling. However, the statistical mechanisms by which these constraints affect generalisation remain poorly understood, especially in realistic settings where the losses are unbounded and several heterogeneous data and physics objectives are optimised jointly.

Existing theoretical analyses mainly study approximation, convergence, optimisation, or Rademacher-complexity based generalisation \citep{Shin_MLMC22, DoumecheBernouilli25, DBLP:conf/nips/RyckM22, Mishra22_IJNMA, Jiao_CP22, LuICLRB22, XuLiHuangICML2025}. In this work, we instead study PIML through PAC-Bayes theory \citep{Alquier24-PB,guedj2019primer,hellstrom2023generalisation}, which controls the generalisation of randomised predictors by trading empirical performance against the information complexity $\KL(\rho\|\pi)$ between a posterior $\rho$ and a prior $\pi$.

Applying PAC-Bayes theory to PIML is challenging because common regression losses, such as squared data errors and differential residuals, are generally unbounded. Existing PAC-Bayes extensions therefore require additional tail or CGF control \citep{haddouche2023pacbayes,NEURIPS2019_3a20f62a,casado2024pacbayeschernoff,JMLR:v25:23-1360}, but the resulting constants are often difficult to interpret or estimate and are only weakly connected to the physical structure of the problem.

Our key idea is to control the concentration of the original, potentially unbounded PIML losses through their input-space regularity. We relate the posterior-averaged CGF of each loss component to its population variance, and then connect this variance to the corresponding input-gradient energy. To make the resulting population quantities computable, we introduce a bounded calibration surrogate and an exact decomposition of the original gradient energy into a clipped contribution and a residual tail. This yields a computable PAC-Bayesian complexity composed of an empirical gradient term, a finite-sample correction, and an explicit gradient-tail term. The resulting certificate is posterior-dependent and can be incorporated into bound-aware optimisation. The regularity constants required by the theory are empirically assessed from calibration data, while the guarantee is stated conditionally on the corresponding population assumptions. After deriving the component-wise certificate, we further exploit the shared-model structure of the heterogeneous PIML losses through a joint PAC-Bayes-Chernoff refinement that avoids repeating KL and confidence penalties across loss components.

Our contributions are three-fold. First, we derive PAC-Bayesian generalisation bounds for PIML with potentially unbounded losses by linking loss concentration to input-gradient regularity. Second, we obtain an end-to-end computable tail-aware certificate for the original unbounded risks through bounded calibration and explicit control of rare large-gradient events, together with a practical bound-aware learning procedure. Third, we derive a joint multi-loss refinement and evaluate the resulting certificates against conventional bounded-loss and sub-Gaussian PAC-Bayes alternatives on PDE benchmarks. 

\section{Preliminaries} \label{preliminary}

\subsection{Physics-informed machine learning} \label{ss:piml}

A PDE with equation constraints, initial conditions (ICs), and boundary conditions (BCs) can be formulated as $\mathcal{D}[u](\boldsymbol{\vx})=0$ for $\boldsymbol{\vx}\in\Omega$, $\mathcal{I}[u](0,\boldsymbol{\vx})=0$ for $\boldsymbol{\vx}\in\Omega_0$, and $\mathcal{B}[u](\boldsymbol{\vx})=0$ for $\boldsymbol{\vx}\in\partial\Omega$, where $\mathcal{D}$, $\mathcal{I}$, and $\mathcal{B}$ are the residual differential, initial, and boundary operators, respectively. The aim is to learn the target function $u:\mathbb{R}^d\rightarrow\mathbb{R}$, which maps an input coordinate $\boldsymbol{\vx}\in\Omega\subset\mathbb{R}^d$ to an output value $y_{\mathrm{obs}}\in\mathbb{R}$. Here, $\boldsymbol{\vx}$ may contain both temporal and spatial coordinates, $\Omega_0$ denotes the initial-time domain, and $\partial\Omega$ denotes the spatial boundary.

Our goal is to learn a parametric model $u_{\boldsymbol{\theta}}$ that approximates the target function $u$, with $\boldsymbol{\theta}\in\mathbb{R}^{d_{\boldsymbol{\theta}}}$ belonging to a model class $\Theta$. As an abuse of notation, $\boldsymbol{\theta}\in\Theta$ denotes both a model and its parameters. We further assume that $u_{\boldsymbol{\theta}}$ lies in the Sobolev space, \ie $u_{\boldsymbol{\theta}}\in H^{p+1}(\Omega)$, where $p$ is the highest-order derivative appearing in the PDE, and that the derivatives involved in the PDE residuals and their input gradients admit pointwise representatives, \ie $u_{\boldsymbol{\theta}}\in C^{p+1}(\Omega)$.

Typically, learning is performed by minimising the following physics-informed empirical risk:
\begin{align}
\widehat{\mathcal{R}}_{\textrm{PIML},\lambda}(\boldsymbol{\theta})
&=\lambda_d\widehat{\mathcal{R}}_d(\boldsymbol{\theta})+\lambda_p\widehat{\mathcal{R}}_p(\boldsymbol{\theta})+\lambda_{ic}\widehat{\mathcal{R}}_{ic}(\boldsymbol{\theta})+\lambda_{bc}\widehat{\mathcal{R}}_{bc}(\boldsymbol{\theta}) \nonumber\\
&=\frac{\lambda_d}{|S_d|}\sum_{(\boldsymbol\vx, y_{\mathrm{obs}})\in S_d}\ell_d\!\left(u_{\boldsymbol{\theta}}(\boldsymbol\vx),y_{\mathrm{obs}}\right)
+\frac{\lambda_p}{|S_p|}\sum_{\boldsymbol\vx\in S_p}\ell_p\!\left(\mathcal{D}[u_{\boldsymbol{\theta}}](\boldsymbol\vx)\right) \nonumber\\
&\quad+\frac{\lambda_{ic}}{|S_{ic}|}\sum_{\boldsymbol\vx\in S_{ic}}\ell_{ic}\!\left(\mathcal{I}[u_{\boldsymbol{\theta}}](\boldsymbol\vx)\right)
+\frac{\lambda_{bc}}{|S_{bc}|}\sum_{\boldsymbol\vx\in S_{bc}}\ell_{bc}\!\left(\mathcal{B}[u_{\boldsymbol{\theta}}](\boldsymbol\vx)\right).
\label{r_piml}
\end{align}
Here, $\ell_d,\ell_p,\ell_{ic},\ell_{bc}$ denote the data-fidelity, PDE, IC, and BC losses, respectively, with the latter three corresponding to physics-based loss components. The hyperparameters $(\lambda_d,\lambda_p,\lambda_{ic},\lambda_{bc})$ control their relative importance during optimisation.

We let $S_d,S_p,S_{ic},S_{bc}$ denote the corresponding sets of observations and PDE, IC, and BC collocation points. \textbf{The samples are assumed to be drawn i.i.d.\ within each dataset, while the datasets are mutually independent}. We allow each dataset to be generated from a different probability distribution, thereby accommodating different observation and collocation procedures. In particular, $S_d$ is generated over $\Omega\times\mathbb{R}$, whereas the physics-based samples are generated over their corresponding physical domains.

To streamline the subsequent analysis, we write the PIML objective in terms of $N_L$ (\eg $N_L=4$ in later experiments) generic loss components. The $i$-th loss $\ell_i(\boldsymbol{\theta},\boldsymbol{\vx})$ is associated with a dataset $S_i=\{\boldsymbol\vx_{i,j}\}_{j=1}^{m_i}$ with $\boldsymbol\vx_{i,j}\overset{\mathrm{i.i.d.}}{\sim}D_i$, and we denote the collection of datasets by $S=\{S_i\}_{i=1}^{N_L}$. Here, $\boldsymbol\vx$ denotes either an observation including its target when required, or a point from the corresponding physical domain. For the data-fidelity loss, the input gradient $\nabla_{\boldsymbol\vx}\ell_i(\boldsymbol{\theta},\boldsymbol\vx)$ is understood to be taken only with respect to the coordinate variables in $\Omega$.

For each loss component, its population and empirical risks are defined as $\mathcal{R}_i(\boldsymbol{\theta}):=\E_{\boldsymbol\vx\sim D_i}[\ell_i(\boldsymbol{\theta},\boldsymbol\vx)]$ and $\widehat{\mathcal{R}}_i(\boldsymbol{\theta}):=\frac{1}{m_i}\sum_{j=1}^{m_i}\ell_i(\boldsymbol{\theta},\boldsymbol\vx_{i,j})$, such that $\widehat{\mathcal{R}}_{\textrm{PIML},\lambda}(\boldsymbol{\theta})=\sum_{i=1}^{N_L}\lambda_i\widehat{\mathcal{R}}_i(\boldsymbol{\theta})$.

For the theoretical analysis, we assume the basic regularity conditions $\ell_i(\boldsymbol{\theta},\boldsymbol\vx)\geq0$, $\mathcal{R}_i(\boldsymbol{\theta})<\infty$, and $\operatorname*{ess\,inf}_{\boldsymbol\vx\sim D_i}\ell_i(\boldsymbol{\theta},\boldsymbol\vx)=0$ for every $i$ and $\boldsymbol{\theta}\in\Theta$. Whenever strict convexity of the corresponding cumulant generating function is required, we additionally assume that $\mathcal{R}_i(\boldsymbol{\theta})-\ell_i(\boldsymbol{\theta},\boldsymbol\vx)$ is non-degenerate under $D_i$.

Training physics-informed models is challenging and commonly requires adapting the loss weights $\{\lambda_i\}$ to improve optimisation \citep{SoftAdapt,PINN_NTK}. These weights may therefore vary across training runs and need not provide a meaningful definition of the target generalisation risk. We consequently study the unweighted population and empirical PIML risks $\mathcal{R}_{\mathrm{PIML}}(\boldsymbol{\theta}):=\sum_{i=1}^{N_L}\mathcal{R}_i(\boldsymbol{\theta})$ and $\widehat{\mathcal{R}}_{\mathrm{PIML}}(\boldsymbol{\theta}):=\sum_{i=1}^{N_L}\widehat{\mathcal{R}}_i(\boldsymbol{\theta})$, with corresponding generalisation gap $\mathrm{gen}(\boldsymbol{\theta},S) = \mathcal{R}_{\mathrm{PIML}}(\boldsymbol{\theta})-\widehat{\mathcal{R}}_{\mathrm{PIML}}(\boldsymbol{\theta})$. This component-wise formulation serves as the basic PIML setting throughout the main analysis. Later, after developing the input-gradient certificate for the individual loss components, we additionally exploit their shared-model structure to derive a joint aggregation that avoids repeating PAC-Bayesian complexity penalties across losses.

\subsection{PAC-Bayes bounds for unbounded regression losses}
\label{ss:pb_regression}
PAC-Bayes theory studies the generalisation of randomised predictors by considering a prior distribution $\pi\in\mathcal{M}(\Theta)$ and a data-dependent posterior distribution $\rho\in\mathcal{M}(\Theta)$ over the hypothesis class. Throughout the paper, we assume $\rho\ll\pi$ and that $\pi$ is independent of the posterior-training data. The resulting guarantees typically balance the empirical performance under $\rho$ against the information complexity $\KL(\rho\|\pi)$.

A central difficulty in applying PAC-Bayes theory to PIML is that its natural regression losses, such as squared data errors and squared differential residuals, are generally unbounded. Classical bounded-loss PAC-Bayes inequalities are therefore not directly applicable unless a potentially conservative uniform bound on the loss is introduced. Recent developments extend PAC-Bayesian analysis beyond this setting by controlling the concentration behaviour of the loss rather than requiring a tight uniform bound on its range \citep{Alquier_2017,haddouche2023pacbayes,HYPE}.

A general route follows from the variational representation of the KL divergence. For a single loss $\ell$, with population risk $\mathcal{R}$ and empirical risk $\widehat{\mathcal{R}}$, an oracle-style PAC-Bayes bound may contain the complexity $\frac{1}{\lambda m}\left[\KL(\rho\|\pi)+\ln\frac{\Psi_{\pi,D}(\lambda)}{\delta}\right]$, where $\Psi_{\pi,D}(\lambda):=\E_{\boldsymbol{\theta}\sim\pi}\E_{S\sim D^m}\!\left[\exp\!\left(\lambda m(\mathcal{R}(\boldsymbol{\theta})-\widehat{\mathcal{R}}(\boldsymbol{\theta}))\right)\right]$ captures the exponential-moment behaviour of the generalisation gap \citep{JMLR:v17:15-290}. Hence, obtaining a computable guarantee requires additional control of this exponential moment. Existing approaches obtain such control through assumptions on the tails or moments of the loss \citep{Alquier_2017,catoni:hal-00104952,NIPS2016_84d2004b}, by directly constraining its cumulant generating function (CGF) \citep{JMLR:v25:23-1360}, or through assumptions on the underlying data distribution, often specialised to Gaussian settings \citep{guo2025pacbayes,Shalaeva2019ImprovedPB}. Therefore, removing a uniform loss bound does not eliminate the need for information about the loss distribution; rather, it replaces range control by an appropriate distributional or moment condition.

A further practical issue is the choice of the exponential tilting parameter $\lambda$. In several PAC-Bayes bounds, $\lambda$ must be selected independently of the training data, commonly through discretisation and an additional union bound. PAC-Bayes-Chernoff techniques instead combine Cramér-Chernoff arguments with generalised inverses to optimise directly over admissible $\lambda>0$~\citep{casado2024pacbayeschernoff}. Defining the loss CGF as $\Lambda_{\boldsymbol{\theta}}(\lambda):=\log\E_{\boldsymbol\vx\sim D} \!\left[\exp\!\left(\lambda(\mathcal{R}(\boldsymbol{\theta})-\ell(\boldsymbol{\theta},\boldsymbol\vx))\right)\right]$, this yields complexities of the form $\inf_{\lambda>0}\left[\frac{\KL(\rho\|\pi)+\ln\frac{m}{\delta}}{\lambda(m-1)}+\frac{\E_{\boldsymbol{\theta}\sim\rho}[\Lambda_{\boldsymbol{\theta}}(\lambda)]}{\lambda}\right]$, shifting the main challenge to obtaining a useful and computable control of the posterior-averaged CGF.

A promising route is to exploit functional inequalities and model regularity to control the CGF through input-gradient quantities~\citep{gat2022on,casado2024pacbayeschernoff}. This perspective is particularly natural for PIML, where data and physics losses are defined over structured input domains and their input gradients quantify the local variation of the corresponding residuals. Motivated by this connection, we control potentially unbounded PIML losses through their input-space regularity rather than a global bound on their magnitude, and develop the calibration and tail-aware machinery needed to make the resulting quantities computable.

Nevertheless, directly using such a relation requires population quantities such as a functional-inequality coefficient and the population input-gradient energy, which are generally unavailable in practical PIML problems. Our analysis therefore follows the chain \colorbox{slightgreen}{$\text{CGF}\rightarrow\text{loss variance}\rightarrow\text{input regularity}\rightarrow\text{computable input-gradient complexity.}$} We first derive this control separately for each original loss component, then introduce bounded auxiliary quantities only to calibrate the unknown regularity coefficient and control the population gradient energy from finite samples. The clipping operations used for these auxiliary quantities do not modify the empirical or population risks appearing in the final guarantee. Finally, after obtaining the component-wise certificate, we return to the shared structure of the PIML losses and derive a joint PAC-Bayes aggregation as an additional tightening step.

\section{Input-gradient PAC-Bayes bounds for PIML}
\label{sec:bounds}

Our goal is to derive a computable PAC-Bayes guarantee for the original, potentially unbounded PIML risk by exploiting input regularity. Rather than bounding the loss, we relate each loss CGF to its population variance and control the latter through its input-gradient energy. We then make these population quantities computable via calibration and a tail-aware gradient decomposition. Crucially, bounded quantities enter only the regularity control; the target population and empirical risks remain those of the original, unclipped losses.

\subsection{From loss concentration to input regularity}
\label{sec:input_regularity}

For every loss component $i$, its CGF is defined as $\Lambda_{\boldsymbol{\theta}}^{(i)}(\lambda):=\log\E_{\vx\sim D_i}\!\left[\exp\!\left(\lambda\left(\mathcal R_i(\boldsymbol{\theta})-\ell_i(\boldsymbol{\theta},\vx)\right)\right)\right]$. Our first assumption relates its posterior average to the population variance of the corresponding loss.

\begin{assumption}[\textbf{Posterior-averaged CGF--variance shape}]
\label{ass:kappa}
For every loss component $i$, there exists $\kappa_{p,i}>0$ such that, for every posterior $\rho\ll\pi$ considered below and every admissible $\lambda>0$, $\E_{\boldsymbol{\theta}\sim\rho}\!\left[\Lambda_{\boldsymbol{\theta}}^{(i)}(\lambda)\right]\leq\frac{\kappa_{p,i}\lambda^2}{2}\E_{\boldsymbol{\theta}\sim\rho}\!\left[\Var_{\vx\sim D_i}[\ell_i(\boldsymbol{\theta},\vx)]\right]$.
\end{assumption}

Assumption~\ref{ass:kappa} is a population-level condition whose coefficient $\kappa_{p,i}$ cannot be certified exactly from finite samples; our guarantees are therefore conditional on its validity. Empirically, we examine the CGF-to-variance ratio across relevant $\lambda$ and choose conservative $\kappa_{p,i}$ values dominating the observed ratios (Appendix~\ref{apdx:ass_kappa}). These diagnostics support the practical plausibility of the assumption but do not certify it. Conditional on Assumption~\ref{ass:kappa}, the remaining data-dependent quantities are controlled by the PAC-Bayesian results below.


To connect the population variance in Assumption~\ref{ass:kappa} to the input regularity of the learned model, for each loss component we define its posterior-averaged population variance-gradient coefficient as $C_i(\rho):=\frac{\E_{\boldsymbol{\theta}\sim\rho}\Var_{\vx\sim D_i}[\ell_i(\boldsymbol{\theta},\vx)]}{\E_{\boldsymbol{\theta}\sim\rho}\E_{\vx\sim D_i}[g_i(\vx,\boldsymbol{\theta})]}$, where $g_i(\vx,\boldsymbol{\theta}):=\|\nabla_{\vx}\ell_i(\boldsymbol{\theta},\vx)\|_2^2$ denotes the input-gradient energy. Whenever the denominator is positive, this definition gives
$
\E_{\boldsymbol{\theta}\sim\rho}\Var_{\vx\sim D_i}[\ell_i(\boldsymbol{\theta},\vx)]
=
C_i(\rho)\E_{\boldsymbol{\theta}\sim\rho}\E_{\vx\sim D_i}[g_i(\vx,\boldsymbol{\theta})].
$
Combining this identity with Assumption~\ref{ass:kappa} shows that controlling the CGF reduces to obtaining computable upper bounds on $C_i(\rho)$ and on the population input-gradient energy. The latter quantity may itself be highly variable. We therefore assume a fixed finite envelope over the localised hypothesis space considered in the analysis, whose plausibility is subsequently assessed through stress-test diagnostics.

\begin{assumption}[\textbf{Input-gradient admissibility}]
\label{ass:gradient_envelope}
For every loss component $i\in\{1,\ldots,N_L\}$, we prescribe a finite envelope $U_i>0$ such that $0\leq g_i(\vx,\boldsymbol{\theta})\leq U_i$ for every admissible input $\vx$ and every $\boldsymbol{\theta}\in\Theta_{\mathrm{loc}}$ considered in the analysis.
\end{assumption}

Assumption~\ref{ass:gradient_envelope} is a population-level condition, so our certificates are conditional on the prescribed envelopes $\{U_i\}_{i=1}^{N_L}$. We fix $U_i$ prior to posterior learning and stress-test the condition over available data and models from the localised hypothesis region; these diagnostics (see Appendix~\ref{apdx:ass_grad_enveloppe}) provide empirical evidence, not uniform certification. To avoid using the potentially conservative $U_i$ as the concentration scale, for $G_i\in(0,U_i)$ we decompose $g_i=g_{i,G_i}+r_{i,G_i}$, where $g_{i,G_i}:=\min\{g_i,G_i\}$ and $r_{i,G_i}:=(g_i-G_i)_+$. This permits concentration at the smaller scale $G_i$ and mitigates dependence on $U_i$ by confining it to the residual-gradient correction.


\subsection{Computable calibration of the population variance-gradient coefficient}
\label{sec:variance_gradient_calibration}

The coefficient $C_i(\rho)$ depends on the population variance and input-gradient energy of the original loss and cannot generally be evaluated from finite samples. We therefore introduce bounded auxiliary quantities solely for its calibration. To simplify notation, we first state the result for a generic loss $\ell$ under a distribution $D$, and subsequently apply it component-wise to PIML.

For prescribed calibration levels $K,G>0$, define $\ell_K(\vx,\boldsymbol{\theta}):=\min\{\ell(\vx,\boldsymbol{\theta}),K\}$ and $g_{K,G}(\vx,\boldsymbol{\theta}):=\min\{\|\nabla_{\vx}\ell_K(\vx,\boldsymbol{\theta})\|_2^2,G\}$. Hence, $0\leq\ell_K\leq K$ and $0\leq g_{K,G}\leq G$. We define the corresponding bounded calibration coefficient as $C_{K,G}(\rho):=\frac{\E_{\boldsymbol{\theta}\sim\rho}\Var_{\vx\sim D}[\ell_K(\vx,\boldsymbol{\theta})]}{\E_{\boldsymbol{\theta}\sim\rho}\E_{\vx\sim D}[g_{K,G}(\vx,\boldsymbol{\theta})]}$.

\begin{assumption}[\textbf{Calibration dominance}]
\label{ass:calibration_dominance}
The calibration levels $K$ and $G$ are chosen such that, for every posterior $\rho\ll\pi$ considered in the analysis, $C(\rho)\leq C_{K,G}(\rho)$.
\end{assumption}

Assumption~\ref{ass:calibration_dominance} formalises the role of $(K,G)$ as conservative calibration levels rather than modifications of the target loss. Its interpretation is particularly transparent when $K$ is sufficiently large that loss clipping is inactive on the relevant population. Indeed, if $\ell_K(\vx,\boldsymbol{\theta})=\ell(\vx,\boldsymbol{\theta})$ almost surely under $D\times\rho$, their population variances coincide, whereas $g_{K,G}(\vx,\boldsymbol{\theta})\leq g(\vx,\boldsymbol{\theta})$ almost surely, and the calibration-dominance relation follows immediately. More generally, we choose $K$ sufficiently large so that clipping is relatively rare, and regard Assumption~\ref{ass:calibration_dominance} as the corresponding population-level regularity condition (see Appendix~\ref{apdx:ass_calib_dominance}).

Now, let $S^{\mathrm{cal}}=(X_1,\ldots,X_{M_c})\sim D^{M_c}$ be a calibration set independent of the posterior-training data. Define the second-order U-statistic $\widehat V_{K,S}(\boldsymbol{\theta}):=\frac{1}{M_c(M_c-1)}\sum_{1\leq j<k\leq M_c}\left(\ell_K(X_j,\boldsymbol{\theta})-\ell_K(X_k,\boldsymbol{\theta})\right)^2$ and the empirical clipped-gradient energy $\widehat J_{K,G,S}(\boldsymbol{\theta}):=\frac{1}{M_c}\sum_{j=1}^{M_c}g_{K,G}(X_j,\boldsymbol{\theta})$. Then $\E_{S^{\mathrm{cal}}}[\widehat V_{K,S}(\boldsymbol{\theta})]=\Var_{\vx\sim D}[\ell_K(\vx,\boldsymbol{\theta})]$. Although $\widehat V_{K,S}$ contains $\binom{M_c}{2}$ pairs, its standard second-order U-statistic reduction has effective sample size $m_c:=\lfloor M_c/2\rfloor$. Moreover, $2\widehat V_{K,S}(\boldsymbol{\theta})/K^2\in[0,1]$ and $\widehat J_{K,G,S}(\boldsymbol{\theta})/G\in[0,1]$.

For $q\in[0,1]$ and $\varepsilon\geq0$, define $\kl_+^{-1}(q,\varepsilon):=\sup\{p\in[q,1]:\kl(q\|p)\leq\varepsilon\}$ and $\kl_-^{-1}(q,\varepsilon):=\inf\{p\in[0,q]:\kl(q\|p)\leq\varepsilon\}$, where $\kl(q\|p):=q\log\frac{q}{p}+(1-q)\log\frac{1-q}{1-p}$. For any $\delta_C\in(0,1)$ and $\rho\ll\pi$, further define $\varepsilon_V(\rho,\delta_C):=\frac{\KL(\rho\|\pi)+\log\frac{4\sqrt{m_c}}{\delta_C}}{m_c}$ and $\varepsilon_J(\rho,\delta_C):=\frac{\KL(\rho\|\pi)+\log\frac{4\sqrt{M_c}}{\delta_C}}{M_c}$.

\begin{restatable}[\textbf{PAC-Bayes calibration of the population variance-gradient coefficient}]
{theorem}{PBCal}
\label{thm:variance_gradient_calibration}
Suppose that Assumption~\ref{ass:calibration_dominance} holds. For any prior $\pi\in\mathcal M(\Theta)$ independent of the calibration sample and any confidence level $\delta_C\in(0,1)$, with probability at least $1-\delta_C$ over $S^{\mathrm{cal}}$, simultaneously for all $\rho\ll\pi$, we have $C(\rho)\leq C_{K,G}(\rho)\leq\overline C_{K,G}(\rho,\delta_C),$
where
\begin{equation}
\overline C_{K,G}(\rho,\delta_C):=\frac{\displaystyle\frac{K^2}{2}\kl_+^{-1}\!\left(\frac{2\E_{\boldsymbol{\theta}\sim\rho}[\widehat V_{K,S}(\boldsymbol{\theta})]}{K^2},\varepsilon_V(\rho,\delta_C)\right)}{\displaystyle G\kl_-^{-1}\!\left(\frac{\E_{\boldsymbol{\theta}\sim\rho}[\widehat J_{K,G,S}(\boldsymbol{\theta})]}{G},\varepsilon_J(\rho,\delta_C)\right)}.
\label{eq:clipped_variance_gradient_bound}
\end{equation}
The bound is finite whenever the lower confidence bound in the denominator is strictly positive.
\end{restatable}

Theorem~\ref{thm:variance_gradient_calibration} therefore uses bounded clipping only as a calibration device: the resulting confidence bound controls the population variance-gradient coefficient of the original, unclipped loss. In particular, neither $K$ nor $\ell_K$ enters the target population risk.
\begin{corollary}[\textbf{Population CGF control by the input-gradient}]
\label{cor:cgf_grad}
For each loss $i$, let $(K_i,G_i)$ denote its prescribed calibration levels and suppose that Assumption~\ref{ass:calibration_dominance} holds component-wise. Applying Theorem~\ref{thm:variance_gradient_calibration} with confidence $\delta_C/N_L$ and taking a union bound gives, with probability at least $1-\delta_C$, simultaneously for every $i$ and every $\rho\ll\pi$, $C_i(\rho)\leq\overline C_{i,K_i,G_i}\!\left(\rho,\frac{\delta_C}{N_L}\right)$. For convenience, define $\overline C_i(\rho,\delta_C):=\kappa_{p,i}\overline C_{i,K_i,G_i}\!\left(\rho,\frac{\delta_C}{N_L}\right)$. Combining this result with Assumption~\ref{ass:kappa} gives
\begin{equation}
\E_{\boldsymbol{\theta}\sim\rho}\!\left[\Lambda_{\boldsymbol{\theta}}^{(i)}(\lambda)\right]\leq\frac{\overline C_i(\rho,\delta_C)\lambda^2}{2}\E_{\boldsymbol{\theta}\sim\rho}\E_{\vx\sim D_i}[g_i(\vx,\boldsymbol{\theta})],
\label{eq:calibrated_cgf_gradient}
\end{equation}
simultaneously for every $i$, every $\rho\ll\pi$, and every admissible $\lambda>0$. 
\end{corollary}
After calibration, the remaining obstacle to a computable PAC-Bayes bound is the population input-gradient energy of each original loss.
\subsection{Component-wise input-gradient generalisation bound}
\label{sec:component_bound}

We control the population input-gradient energy in~\eqref{eq:calibrated_cgf_gradient} using the decomposition $g_i=g_{i,G_i}+r_{i,G_i}$, where $0\leq g_{i,G_i}\leq G_i$ and $0\leq r_{i,G_i} \leq H_i:=U_i-G_i$. We apply PAC-Bayes concentration to the bounded component and a binary-$\kl$ inequality to the normalised residual. For every component $i$, define the empirical input-gradient contribution $\widehat{\mathcal G}_{i}(\rho,\delta_C):=\overline C_i(\rho,\delta_C)\E_{\boldsymbol{\theta}\sim\rho}\sum_{j=1}^{m_i}g_i(\vx_{i,j},\boldsymbol{\theta})$ and the finite-sample correction $\Delta_{i,G_i}(\rho,\delta_C,\delta_J):=m_i\overline C_i(\rho,\delta_C)G_i\sqrt{\frac{\KL(\rho\|\pi)+\ln\frac{N_Lm_i}{\delta_J}}{2(m_i-1)}}$. For the residual contribution, define $\widehat q_{i,G_i}(\rho):=\frac{1}{m_i}\sum_{j=1}^{m_i}\E_{\boldsymbol{\theta}\sim\rho}\!\left[\frac{r_{i,G_i}(\vx_{i,j},\boldsymbol{\theta})}{H_i}\right]$ and $\mathcal T_{i,G_i}(\rho,\delta_i):=H_i\kl_+^{-1}\!\left(\widehat q_{i,G_i}(\rho),\frac{\KL(\rho\|\pi)+\log\frac{2\sqrt{m_i}}{\delta_i}}{m_i}\right)$. Note that in practice, the gradient envelope is usually large enough such that $\widehat q_{i,G_i}(\rho) \ll 1$, hence the term $\mathcal T_{i,G_i}(\rho,\delta_i)$ strongly depends on the KL term and acts as a regulariser. The resulting component-wise input-gradient complexity is
\begin{equation}
\mathcal G_{i,G_i}(\rho,\delta_C,\delta_J,\delta_T):=\colorbox{slightblue}{$\widehat{\mathcal G}_{i}(\rho,\delta_C)$}+\underbrace{\colorbox{slightorange}{$\Delta_{i,G_i}(\rho,\delta_C,\delta_J)+m_i\overline C_i(\rho,\delta_C)\mathcal T_{i,G_i}\!\left(\rho,\frac{\delta_T}{N_L}\right)$}}_{\mathrm{Reg}(\rho, \pi)}.
\end{equation}

\begin{restatable}[\textbf{Component-wise PIML bound with input-gradient complexity}]
{theorem}{ComponentBound}
\label{thm:component_bound}
Suppose that Assumptions~\ref{ass:kappa}, \ref{ass:gradient_envelope}, and \ref{ass:calibration_dominance} hold. Then, with probability at least $1-\delta$ over the calibration and posterior-training samples, simultaneously for every $\rho\ll\pi$ and every loss component $i$,
\begin{equation}
\mathcal R_i(\rho)\leq\widehat{\mathcal R}_i(\rho)+\frac{1}{m_i}\sqrt{2\,\mathcal G_{i,G_i}\left(\rho,\frac{\delta}{3},\frac{\delta}{6},\frac{\delta}{6}\right)K_i^{\mathrm{PB}}\left(\rho,\pi,\frac{\delta}{3N_L}\right)},
\label{eq:component_bound}
\end{equation}
where $\mathcal R_i(\rho):=\E_{\boldsymbol{\theta}\sim\rho}[\mathcal R_i(\boldsymbol{\theta})]$, $\widehat{\mathcal R}_i(\rho):=\E_{\boldsymbol{\theta}\sim\rho}[\widehat{\mathcal R}_i(\boldsymbol{\theta})]$, and $K_i^{\mathrm{PB}}(\rho,\pi,\delta_i):=\frac{m_i\left(\KL(\rho\|\pi)+\ln\frac{m_i}{\delta_i}\right)}{m_i-1}$.
\end{restatable}

Theorem~\ref{thm:component_bound} is the building block of our PIML analysis. It controls the loss through the posterior-average input-gradient \colorbox{slightblue}{$\widehat{\mathcal G}_{i}(\rho,\delta_C)$}, involving the input gradients of both the model $u_{\boldsymbol{\theta}}$ and its physical residual operators, which are well-defined as $u_{\boldsymbol{\theta}}\in C^{p+1}(\Omega)$ (Section~\ref{ss:piml}). The terms \colorbox{slightorange}{$\mathrm{Reg}(\rho, \pi)$} and $K_i^{\mathrm{PB}}(\rho,\pi,\delta_i)$ act as regularisers, favoring posteriors close to the physics-informed prior. Summing the above component-wise guarantees via a union bound yields the following result. 


\begin{corollary}[\textbf{Task-wise PIML union bound}]
\label{cor:piml_union}
Under the conditions of Theorem~\ref{thm:component_bound}, with probability at least $1-\delta$ over the calibration and posterior-training samples, simultaneously for every $\rho\ll\pi$,
\begin{equation}
\mathcal R_{\mathrm{PIML}}(\rho)\leq\widehat{\mathcal R}_{\mathrm{PIML}}(\rho)+\sum_{i=1}^{N_L}\frac{1}{m_i}\sqrt{2\,\mathcal G_{i,G_i}(\rho)K_i^{\mathrm{PB}}(\rho,\pi)},
\label{eq:piml_union_bound}
\end{equation}
where the confidence arguments are allocated with total failure probability at most $\delta$.
\end{corollary}

This bound already provides an end-to-end computable certificate for the PIML risk using the proposed input-gradient complexity. However, it applies the PAC-Bayes-Chernoff argument independently to each loss component and consequently repeats the KL and confidence penalties across the $N_L$ terms. We next exploit their shared-model structure to obtain a joint sample-weighted certificate with a single PAC-Bayesian complexity penalty.

\subsection{Joint multi-task refinement}
\label{sec:joint_bound}

The component-wise analysis above does not exploit that all data and physics losses are evaluated on the same randomised predictor. PIML can therefore be viewed as a particular multi-task problem in which the loss components share the same learnable model and differ only through the fixed operators defining the data, PDE, IC, and BC residuals. This structure allows all $M:=\sum_{i=1}^{N_L}m_i$ samples to be treated jointly. To this end, define the sample-weighted population and empirical risks as $\mathcal R_{\mathrm{PIML}}^{\mathcal S}(\boldsymbol{\theta}):=\frac{1}{M}\sum_{i=1}^{N_L}m_i\mathcal R_i(\boldsymbol{\theta})$ and $\widehat{\mathcal R}_{\mathrm{PIML}}^{\mathcal S}(\boldsymbol{\theta}):=\frac{1}{M}\sum_{i=1}^{N_L}\sum_{j=1}^{m_i}\ell_i(\boldsymbol{\theta},\vx_{i,j})$, with corresponding generalisation gap $\mathrm{gen}^{\mathcal S}(\boldsymbol{\theta},S):=\mathcal R_{\mathrm{PIML}}^{\mathcal S}(\boldsymbol{\theta})-\widehat{\mathcal R}_{\mathrm{PIML}}^{\mathcal S}(\boldsymbol{\theta})$. Further define the sample-weighted CGF $\Lambda_{\boldsymbol{\theta}}^{\mathcal S}(\lambda):=\frac{1}{M}\sum_{i=1}^{N_L}m_i\Lambda_{\boldsymbol{\theta}}^{(i)}(\lambda)$.

\begin{restatable}[\textbf{Joint PAC-Bayes-Chernoff bound}]
{theorem}{GeneralPBChernoeffMTL}
\label{thm:generalpbc_mtl}
For any $\delta\in(0,1)$ and any prior $\pi\in\mathcal M(\Theta)$ independent of the posterior-training samples, with probability at least $1-\delta$ over $S$, simultaneously for all $\rho\ll\pi$,
\begin{equation}
\mathrm{gen}^{\mathcal S}(\rho,S)\leq\inf_{\lambda>0}\left\{\frac{\KL(\rho\|\pi)+\ln\frac{M}{\delta}}{\lambda(M-1)}+\frac{\E_{\boldsymbol{\theta}\sim\rho}[\Lambda_{\boldsymbol{\theta}}^{\mathcal S}(\lambda)]}{\lambda}\right\}.
\label{eq:general_pbc_mtl}
\end{equation}
\end{restatable}

Unlike the task-wise construction in Corollary~\ref{cor:piml_union}, Theorem~\ref{thm:generalpbc_mtl} introduces only one KL and one confidence penalty, while the task-specific concentration behaviour is aggregated through $\Lambda_{\boldsymbol{\theta}}^{\mathcal S}$. The input-gradient machinery developed above applies directly because $\E_{\boldsymbol{\theta}\sim\rho}[\Lambda_{\boldsymbol{\theta}}^{\mathcal S}(\lambda)]\leq\frac{\lambda^2}{2M}\sum_{i=1}^{N_L}m_i\overline C_i(\rho,\delta_C)\E_{\boldsymbol{\theta}\sim\rho}\E_{\vx\sim D_i}[g_i(\vx,\boldsymbol{\theta})]$.

Define the joint computable input-gradient complexity as $\mathcal G_{\mathbf G}^{\mathcal S}(\rho,\delta_C,\delta_J,\delta_T):=\colorbox{slightblue}{$\widehat{\mathcal G}^{\mathcal S}(\rho,\delta_C)$}+\colorbox{slightorange}{$\Delta_{\mathbf G}^{\mathcal S}(\rho,\delta_C,\delta_J)+\mathcal T_{\mathbf G}^{\mathcal S}(\rho,\delta_C,\delta_T)$},$ where $\widehat{\mathcal G}^{\mathcal S}(\rho,\delta_C):=\sum_{i=1}^{N_L}\overline C_i(\rho,\delta_C)\E_{\boldsymbol{\theta}\sim\rho}\sum_{j=1}^{m_i}g_i(\vx_{i,j},\boldsymbol{\theta})$, $\Delta_{\mathbf G}^{\mathcal S}(\rho,\delta_C,\delta_J):=\sum_{i=1}^{N_L}m_i\overline C_i(\rho,\delta_C)G_i\sqrt{\frac{\KL(\rho\|\pi)+\ln\frac{N_Lm_i}{\delta_J}}{2(m_i-1)}}$, and $\mathcal T_{\mathbf G}^{\mathcal S}(\rho,\delta_C,\delta_T):=\sum_{i=1}^{N_L}m_i\overline C_i(\rho,\delta_C)\mathcal T_{i,G_i}\!\left(\rho,\frac{\delta_T}{N_L}\right)$.

\begin{restatable}[\textbf{Joint PIML bound with input-gradient complexity}]
{theorem}{FinalBound}
\label{thm:main}
Suppose that Assumptions~\ref{ass:kappa}, \ref{ass:gradient_envelope}, and \ref{ass:calibration_dominance} hold for every loss component. Then, with probability at least $1-\delta$ over the calibration and posterior-training samples, simultaneously for every $\rho\ll\pi$,
\begin{equation}
\mathcal R_{\mathrm{PIML}}^{\mathcal S}(\rho)\leq\widehat{\mathcal R}_{\mathrm{PIML}}^{\mathcal S}(\rho)+\frac{1}{M}\sqrt{2\,\mathcal G_{\mathbf G}^{\mathcal S}\left(\rho,\frac{\delta}{3},\frac{\delta}{6},\frac{\delta}{6}\right)K\left(\rho,\pi,\frac{\delta}{3}\right)},
\label{eq:main}
\end{equation}
where $K(\rho,\pi,\delta_G):=\frac{M\left(\KL(\rho\|\pi)+\ln\frac{M}{\delta_G}\right)}{M-1}$.
\end{restatable}

Theorem~\ref{thm:main} is the joint counterpart of Corollary~\ref{cor:piml_union}. Both bounds use the same calibrated input-gradient control for the original PIML losses, while the joint construction aggregates the component-wise CGFs before the PAC-Bayes-Chernoff optimisation and therefore avoids repeating the KL and confidence penalties.

\begin{corollary}[\textbf{Joint PIML bounds for balanced and data-scarce regimes}]
\label{cor:joint_balanced_unbalanced}
Under the conditions of Theorem~\ref{thm:main}, consider $\mathcal R_{\mathrm{PIML}}(\rho):=\sum_{i=1}^{N_L}\mathcal R_i(\rho)$. If $m_i=m$ for all components, multiplying~\eqref{eq:main} by $N_L$ yields a valid certificate for $\mathcal R_{\mathrm{PIML}}(\rho)$. If instead $m_d<m_{\mathrm{phy}}$, with $m_i=m_{\mathrm{phy}}$ for all physical components, applying Theorem~\ref{thm:component_bound} to $\mathcal R_d$ and Theorem~\ref{thm:main} jointly to the physical block yields $\mathcal R_{\mathrm{PIML}}(\rho)\leq U_d(\rho,\delta_d)+U_{\mathrm{phy}}(\rho,\delta_{\mathrm{phy}})$ with probability at least $1-\delta$, for any $\delta_d+\delta_{\mathrm{phy}}=\delta$, where the joint physical bound is multiplied by $N_{\mathrm{phy}}:=|\mathcal I_{\mathrm{phy}}|$.
\end{corollary}

\section{Bound-aware algorithm}
\label{sec:bound_aware_algorithm}

We now turn the proposed certificate into a practical bound-aware learning procedure. To facilitate the assessment of the regularity assumptions, we localise the hypothesis space around a physics-informed reference model $\boldsymbol{\theta}_{\pi}$ as $\Theta_{\mathrm{loc}}:=\{\boldsymbol{\theta}:\|\boldsymbol{\theta}-\boldsymbol{\theta}_{\pi}\|_2\leq R_{\mathrm{loc}}\}$, and consider Gaussian prior and posterior distributions truncated to this neighbourhood, $\pi=\mathcal N_{\Theta_{\mathrm{loc}}}(\boldsymbol{\theta}_{\pi},\sigma_\pi^2I)$ and $\rho=\mathcal N_{\Theta_{\mathrm{loc}}}(\boldsymbol{\theta}_{\rho},\sigma_\rho^2I)$. The prior mean is pretrained on an independent dataset $S^{\mathrm{prior}}$ using only the physics losses, after which $\pi$ is fixed. The same independent data are used to determine conservative loss, gradient, and envelope scales $(K_i,G_i,U_i)$, while an independent posterior set $S^{\mathrm{post}}$ provides the surrogate population variance-gradient coefficients $\{\overline C_i^{\mathrm{train}}\}_{i=1}^{N_L}$ for training. Details of the localisation, calibration, and empirical assessment of Assumptions~\ref{ass:kappa}-\ref{ass:calibration_dominance} are given in Appendix~\ref{apdx:algo}.

Starting from $\boldsymbol{\theta}_{\rho}=\boldsymbol{\theta}_{\pi}$, we optimise the posterior using a stochastic surrogate of the bound. At each update, a model $\boldsymbol{\theta}\sim\rho$ is sampled through reparameterisation to estimate the sample-weighted empirical PIML risk $\widehat R_{\mathrm{PIML}}^B(\boldsymbol{\theta})$ and the input-gradient complexity $\mathcal G_{\mathbf G,\mathrm{tr}}^B(\boldsymbol{\theta};\pi)$ on mini-batches $B=\{B_i\}_{i=1}^{N_L}$. Our surrogate training objective is
\begin{equation}
\mathcal L_{\mathrm{BA}}^B(\boldsymbol{\theta},\rho):=\widehat R_{\mathrm{PIML}}^B(\boldsymbol{\theta})+\frac{1}{M}\sqrt{2\mathcal G_{\mathbf G,\mathrm{tr}}^B(\boldsymbol{\theta}; \rho, \pi)K_{\mathrm{tr}}(\rho,\pi)},\qquad \boldsymbol{\theta}\sim\rho,
\label{eq:bound_aware_objective}
\end{equation}
where $\mathcal G_{\mathbf G,\mathrm{tr}}^B$ contains the clipped-gradient, finite-sample, and gradient-tail contributions derived in Section~\ref{sec:bounds}, while $K_{\mathrm{tr}}$ contains the exact $\KL(\rho\|\pi)$ complexity of the current posterior. Hence, posterior learning jointly promotes empirical accuracy, proximity to the physics-informed prior, and low input-gradient complexity. After optimisation, the posterior-dependent calibration quantities are recomputed on $S^{\mathrm{cal}}$ and the theoretical certificate is evaluated using the full expressions of Section~\ref{sec:bounds}.

\paragraph{Balanced and unbalanced regimes.}
Training and evaluation follows Corollary~\ref{cor:joint_balanced_unbalanced} (joint bound) or Corollary~\ref{cor:piml_union} (union bound). In the balanced regime $m_i=m$ of joint bound, we optimise the stochastic objective in~\eqref{eq:bound_aware_objective}, multiplied by $N_L$ to recover the equally weighted PIML risk. In the data-scarce regime $m_d<m_{\mathrm{phy}}$, we instead optimise the sum of the component-wise observational objective and the joint physical objective, with $\delta_d=\delta/N_L$ and $\delta_{\mathrm{phy}}=\delta-\delta_d$. 

\section{Experiment}
We evaluate the proposed PAC-Bayesian framework on six PDE benchmarks: Wave-1D, Convection, Burgers, Heat-2D MS, Helmholtz and Flow Mixing-2D. For each physical loss, collocation samples are split into disjoint prior-training, posterior-training, test, and calibration sets. We consider two observational-data regimes: \emph{balanced} (10k/5k posterior/calibration samples for Burgers and 20k/10k otherwise) and \emph{data-scarce}, using only 100 posterior-training and 100 independent calibration observations while keeping the full physics-based collocation sets. \textbf{The prior is trained exclusively on physics-based losses}. See Appendix~\ref{apdx:problem_formulation},~\ref{apdx:implementation} for details. For convenience, \textbf{Ours-joint} denotes our multi-task certificate (Corollary~\ref{cor:joint_balanced_unbalanced}), while \textbf{Ours-union} stands for the component-wise construction (Corollary~\ref{cor:piml_union}). We compare against \textbf{Bounded-$U$}, inspired by \citep{zakerinia2025fast}, using a fast-rate multi-task bound with a conservative uniform loss envelope, and \textbf{Sub-Gaussian}, using sub-Gaussian control through the same gradient envelope.

Table~\ref{tab:benchmark} shows that our input-gradient certificates are consistently tighter than both baselines, while maintaining comparable or lower test risks, indicating that bound reduction does not generally come at the expense of predictive performance. In the balanced regime, \emph{Ours-joint} yields bounds ranging from $6.26\times10^{-3}$ (Heat-2D MS) to $9.16$ (Helmholtz). The gain is largest on Heat-2D MS, where the bound is over an order of magnitude tighter than Sub-Gaussian and nearly two orders tighter than Bounded-$U$. Joint and union bounds are comparable on Wave-1D and Helmholtz, while joint control yields modest gains on Convection and Burgers and clearer improvements on Flow Mixing and Heat-2D MS, tightening the latter from $9.60\times10^{-3}$ to $6.26\times10^{-3}$. In the data-scarce regime, certificates deteriorate despite largely stable test risks, reflecting the statistical cost of limited observational evidence; nevertheless, our bounds remain substantially tighter than both baselines across all problems. Table~\ref{tab:ablation_n_d_posterior} further shows that the certificates tighten systematically with posterior sample size. On Heat-2D MS, \emph{Ours-joint} decreases from $1.73$ with only $10$ observations to $6.26\times10^{-3}$ with $20k$ observations, remaining tighter than both baselines throughout. Joint and union bound are nearly indistinguishable under severe data scarcity but separate as the sample size grows, reaching $6.26\times10^{-3}$ versus $9.60\times10^{-3}$ at $20$k samples. This highlights both the benefit of additional observational evidence and, once finite-sample effects diminish, joint multi-loss control. See Appendix~\ref{apdx:full_results} for complete results and detailed analysis of all benchmarks. Furthermore, additional ablations in Appendix~\ref{apdx:ablation} examine the impact of the observational calibration sample size, the gradient envelope and the localisation radius to the bound, showing the importance of increasing calibration data, and identifying the large gradient envelope as the main source of inflation.

\begin{table*}[h]
\centering
\small
\caption{PAC-Bayes bounds and empirical risks across PDE benchmarks and data regimes.}
\label{tab:benchmark}
\resizebox{\linewidth}{!}{
\begin{tabular}{lllcccc}
\toprule
\textbf{Problem} & \textbf{Data regime} & \textbf{Quantity} & \textbf{Ours-joint} & \textbf{Ours-union} & \textbf{Bounded-$U$} & \textbf{Sub-Gaussian} \\
\midrule
\multirow{4}{*}{\textbf{Wave-1D}} &  \multirow{2}{*}{Balanced} &  Test risk & $\mstd{4.30}{0.52}{-1}$ & $\mstd{4.30}{0.52}{-1}$ & $\mstd{4.43}{0.53}{-1}$ & $\mstd{4.61}{0.54}{-1}$ \\
  &     &  Bound & $\bmstd{6.46}{0.75}{-1}$ & $\bmstd{6.46}{0.75}{-1}$ & $\mstd{1.40}{0.09}{0}$ & $\mstd{8.05}{0.73}{0}$  \\
  \cmidrule(lr){2-7}
  &  \multirow{2}{*}{Data-scarce}   &  Test risk & $\mstd{4.31}{0.52}{-1}$ & $\mstd{4.31}{0.52}{-1}$ & $\mstd{4.68}{0.54}{-1}$ & $\mstd{4.61}{0.54}{-1}$ \\
  &     &  Bound & $\bmstd{6.57}{0.77}{-1}$ & $\bmstd{6.57}{0.77}{-1}$ & $\mstd{6.61}{0.01}{1}$ & $\mstd{8.08}{0.74}{0}$ \\
\bottomrule
\multirow{4}{*}{\textbf{Convection}}&  \multirow{2}{*}{Balanced}   &  Test risk & $\mstd{1.91}{0.23}{-2}$ & $\mstd{1.92}{0.23}{-2}$ & $\mstd{2.11}{0.35}{-2}$ & $\mstd{2.12}{0.35}{-2}$ \\
  &     &  Bound & $\bmstd{3.12}{0.43}{-2}$ & $\mstd{3.18}{0.44}{-2}$ & $\mstd{5.21}{0.12}{-1}$ & $\mstd{4.62}{0.55}{-1}$ \\
  \cmidrule(lr){2-7}
  &  \multirow{2}{*}{Data-scarce}   &  Test risk & $\mstd{2.06}{0.31}{-2}$ & $\mstd{2.06}{0.31}{-2}$ & $\mstd{2.16}{0.38}{-2}$ & $\mstd{2.13}{0.36}{-2}$ \\
  &     &  Bound & $\bmstd{7.77}{0.72}{-2}$ & $\mstd{7.81}{0.73}{-2}$ & $\mstd{6.52}{0.00}{1}$ & $\mstd{6.14}{0.60}{-1}$ \\
\bottomrule
\multirow{4}{*}{\textbf{Burgers}}&  \multirow{2}{*}{Balanced}   &  Test risk & $\mstd{3.22}{1.01}{-3}$ & $\mstd{3.22}{1.01}{-3}$ & $\mstd{3.24}{1.02}{-3}$ & $\mstd{3.24}{1.02}{-3}$ \\
  &     &  Bound &  $\bmstd{6.01}{0.44}{-1}$ & $\mstd{6.14}{0.44}{-1}$ & $\mstd{8.51}{0.09}{-1}$ & $\mstd{1.75}{0.14}{1}$ \\
  \cmidrule(lr){2-7}
  &  \multirow{2}{*}{Data-scarce} &  Test risk & $\mstd{3.23}{1.02}{-3}$ & $\mstd{3.23}{1.02}{-3}$ & $\mstd{3.25}{1.03}{-3}$ & $\mstd{3.24}{1.02}{-3}$ \\
  &     &  Bound & $\bmstd{1.89}{0.58}{0}$ & $\mstd{1.90}{0.58}{0}$ & $\mstd{6.55}{0.00}{1}$ & $\mstd{2.18}{0.26}{1}$ \\
\bottomrule
\multirow{4}{*}{\textbf{\makecell{Heat-2D\\MS}}}&  \multirow{2}{*}{Balanced}  &  Test risk & $\mstd{2.05}{0.54}{-3}$ & $\mstd{2.10}{0.57}{-3}$ & $\mstd{2.23}{0.66}{-3}$ & $\mstd{2.22}{0.66}{-3}$ \\
  &     &  Bound & $\bmstd{6.26}{1.67}{-3}$ & $\mstd{9.60}{2.33}{-3}$ & $\mstd{4.46}{0.04}{-1}$ & $\mstd{1.76}{0.49}{-1}$ \\
  \cmidrule(lr){2-7}
  &  \multirow{2}{*}{Data-scarce}   &  Test risk & $\mstd{2.24}{0.68}{-3}$ & $\mstd{2.24}{0.68}{-3}$ & $\mstd{2.24}{0.68}{-3}$ & $\mstd{2.24}{0.68}{-3}$  \\
  &     &  Bound & $\bmstd{5.74}{0.78}{-1}$ & $\mstd{5.76}{0.78}{-1}$ & $\mstd{6.51}{0.00}{1}$ & $\mstd{2.08}{0.29}{0}$ \\
\bottomrule
\multirow{4}{*}{\textbf{\makecell{Helmholtz}}}&  \multirow{2}{*}{Balanced}  &  Test risk & $\mstd{4.62}{0.48}{0}$ & $\mstd{4.62}{0.48}{0}$ & $\mstd{4.62}{0.48}{0}$ & $\mstd{4.73}{0.48}{0}$ \\
  &     &  Bound & $\bmstd{9.16}{0.76}{0}$ & $\mstd{9.17}{0.76}{0}$ & $\mstd{1.08}{0.07}{1}$ & $\mstd{1.78}{0.15}{2}$ \\
  \cmidrule(lr){2-7}
  &  \multirow{2}{*}{Data-scarce}   &  Test risk & $\mstd{4.62}{0.48}{0}$ & $\mstd{4.62}{0.48}{0}$ & $\mstd{4.73}{0.48}{0}$ & $\mstd{4.73}{0.48}{0}$ \\
  &     &  Bound & $\bmstd{9.23}{0.76}{0}$ & $\bmstd{9.23}{0.76}{0}$ & $\mstd{3.34}{0.01}{2}$ & $\mstd{1.78}{0.15}{2}$ \\
\bottomrule
\multirow{4}{*}{\textbf{\makecell{Flow \\ Mixing}}}&  \multirow{2}{*}{Balanced}  &  Test risk & $\mstd{3.22}{0.23}{-3}$ & $\mstd{3.22}{0.23}{-3}$ & $\mstd{3.22}{0.23}{-3}$ & $\mstd{3.22}{0.23}{-3}$ \\
  &     &  Bound & $\bmstd{3.98}{0.22}{-2}$ & $\mstd{5.58}{0.33}{-2}$ & $\mstd{4.51}{0.01}{-1}$ & $\mstd{1.44}{0.09}{0}$ \\
  \cmidrule(lr){2-7}
  &  \multirow{2}{*}{Data-scarce}   &  Test risk & $\mstd{3.22}{0.23}{-3}$ & $\mstd{3.22}{0.23}{-3}$ & $\mstd{3.22}{0.23}{-3}$ & $\mstd{3.22}{0.23}{-3}$ \\
  &     &  Bound & $\bmstd{1.21}{0.16}{1}$ & $\bmstd{1.21}{0.16}{1}$ & $\mstd{6.51}{0.00}{1}$ & $\mstd{4.19}{0.53}{1}$ \\
\bottomrule
\end{tabular}
}
\end{table*}

\begin{table*}[h]
\centering
\small
\caption{Effect of posterior observational sample size on Heat-2D MS bounds.}
\label{tab:ablation_n_d_posterior}
\resizebox{0.8\linewidth}{!}{
\begin{tabular}{ccccc}
\toprule
$m_d^{\mathrm{posterior}}$ & \textbf{Ours-Joint} & \textbf{Ours-Union} & \textbf{Bounded-$U$} & \textbf{Sub-Gaussian} \\
\midrule
10 & $\bmstd{1.73}{0.29}{0}$ & $\bmstd{1.73}{0.29}{0}$ & $\mstd{4.25}{0.00}{2}$ & $\mstd{2.43}{0.41}{0}$ \\
100 & $\bmstd{2.43}{0.41}{-1}$ & $\mstd{2.45}{0.41}{-1}$ & $\mstd{6.51}{0.00}{1}$ & $\mstd{9.59}{1.63}{-1}$ \\
1000 & $\bmstd{3.45}{0.60}{-2}$ & $\mstd{3.64}{0.63}{-2}$ & $\mstd{8.26}{0.00}{0}$ & $\mstd{4.42}{0.84}{-1}$ \\
10000 & $\bmstd{9.20}{2.08}{-3}$ & $\mstd{1.12}{0.25}{-2}$ & $\mstd{1.35}{0.00}{0}$ & $\mstd{2.58}{0.60}{-1}$ \\
20000 & $\bmstd{6.26}{1.67}{-3}$ & $\mstd{9.60}{2.33}{-3}$ & $\mstd{4.46}{0.04}{-1}$ & $\mstd{1.76}{0.49}{-1}$ \\
\bottomrule
\end{tabular}
}
\end{table*}

\section{Conclusion}
We introduced a PAC-Bayesian framework for physics-informed machine learning that connects physical regularity to generalisation through empirically accessible input-gradient quantities. We instantiated this framework through population-level regularity conditions and a PAC-Bayes calibration procedure, yielding data-dependent input-gradient complexities while controlling rare large-gradient events via a bounded residual correction. Adopting a multi-task perspective further yields a refined joint PIML certificate with a single KL complexity. Beyond the theory, we proposed a bound-aware learning algorithm and showed across six PDE benchmarks that our certificates are consistently tighter than bounded-envelope and sub-Gaussian alternatives, retain a substantial advantage under observational data scarcity, and tighten systematically with sample size. Our results further demonstrate the benefit of joint multi-task control while identifying the gradient tail as a key source that inflates the bounds. A more detailed discussion of the limitations is provided in Appendix~\ref{sa:limitations}. Promising directions include sharper residual-gradient control, improved data-dependent calibration, and extensions to more complex physical systems.

\clearpage
     



\bibliographystyle{iclr2027_conference}
\bibliography{references}

\clearpage
\appendix

\hypersetup{linkcolor=black}
\section*{Appendix Table of Contents}
\startcontents[appendices]
\printcontents[appendices]{}{1}{}
\hypersetup{linkcolor=red}
\clearpage
\section{Complete Proofs of Theoretical Results}
\label{app:proof}

\subsection{Auxiliary lemmas}
\begin{lemma}[\textbf{General PAC-Bayesian Bound}]  \label{general_pb}
For any distribution $D$ on $\mathcal{X}$, for any hypothesis set $\Theta$, for any  
prior $\pi \in \mathcal{M}(\Theta)$, for any measurable function $\varphi: \Theta \times \mathcal{X}^m \rightarrow \mathbb{R}$, for any posterior $\rho \in \mathcal{M}(\Theta)$ we have:

\begin{equation}
    \Prob_{S\sim D^m} \left[ \E_{\boldsymbol{\theta} \sim \rho} \varphi (\boldsymbol{\theta}, S) \leq \mathrm{KL}(\rho \| \pi) + \ln \left(\frac{1}{\delta} \E_{S' \sim D^m} \E_{\boldsymbol{\theta}' \sim \pi} e^{\varphi(\boldsymbol{\theta}', S')} \right)   \right] \geq 1 - \delta.  
\end{equation}
\end{lemma}

\begin{proof}
By the Donsker-Varadhan variational formula \citep{DonskerVaradhan},
\begin{equation}
\label{dk_formula}
    \E_{\boldsymbol{\theta}\sim\rho}\varphi(\boldsymbol{\theta},S)
    \leq \KL(\rho\|\pi)+\ln\E_{\boldsymbol{\theta}\sim\pi}e^{\varphi(\boldsymbol{\theta},S)}.
\end{equation}
By Markov's inequality, with probability at least $1-\delta$ over $S\sim D^m$,
\begin{equation}
    \E_{\boldsymbol{\theta}\sim\pi}e^{\varphi(\boldsymbol{\theta},S)}
    \leq \frac{1}{\delta}\E_{S'\sim D^m}\E_{\boldsymbol{\theta}\sim\pi}e^{\varphi(\boldsymbol{\theta},S')}.
\end{equation}
Substituting this inequality into~\eqref{dk_formula} yields the result.
\end{proof}

\begin{lemma}[\textbf{Convex comparison with a binomial variable} \citep{BerendTassa}]
\label{lem:berend-tassa}
Let $X_1,\ldots,X_t$ be independent random variables satisfying $0\le X_j\le 1$ almost surely, and let
\[
X:=\sum_{j=1}^{t}X_j,
\qquad
\mu:=\mathbb{E}[X].
\]
Let $Y\sim\mathrm{Bin}\!\left(t,\mu/t\right)$. Then, for any convex function $f$,
\[
\mathbb{E}[f(X)]\le \mathbb{E}[f(Y)].
\]
\end{lemma}

\begin{lemma}[\textbf{Exponential moment of the binary-kl divergence} \citep{Maurer2004ANO}]
\label{lem:maurer-kl}
Let $Y\sim\mathrm{Bin}\!\left(t,\mu/t\right)$. Then
\[
\mathbb{E}\!\left[
\exp\!\left(
t\,\kl\!\left(\frac{Y}{t}\,\middle\|\,\frac{\mu}{t}\right)
\right)
\right]
\le 2\sqrt{t},
\]
where
\[
\kl(q\|p)
:=q\log\frac{q}{p}
+(1-q)\log\frac{1-q}{1-p}
\]
denotes the binary Kullback--Leibler divergence.
\end{lemma}

\begin{lemma}[\textbf{PAC-Bayes binary-kl bound for independent bounded losses}]
\label{lem:pac-bayes-kl}
Let $\pi$ be a prior distribution over $\Theta$ independent of the sample $S$. For each $\boldsymbol{\theta}\in\Theta$, let $X_1(\boldsymbol{\theta}),\ldots,X_m(\boldsymbol{\theta})$ be independent random variables satisfying $0\le X_j(\boldsymbol{\theta})\le1$ almost surely, and define
\[
\widehat{\mathcal R}_S(\boldsymbol{\theta})
:=\frac{1}{m}\sum_{j=1}^{m}X_j(\boldsymbol{\theta}),
\qquad
\mathcal R(\boldsymbol{\theta})
:=\frac{1}{m}\sum_{j=1}^{m}\mathbb{E}[X_j(\boldsymbol{\theta})].
\]
Then, with probability at least $1-\delta$ over $S$, simultaneously for every posterior $\rho$ on $\Theta$,
\[
\kl\!\left(
\mathbb{E}_{\boldsymbol{\theta}\sim\rho}
[\widehat{\mathcal R}_S(\boldsymbol{\theta})]
\,\middle\|\,
\mathbb{E}_{\boldsymbol{\theta}\sim\rho}
[\mathcal R(\boldsymbol{\theta})]
\right)
\le
\frac{
\KL(\rho\|\pi)+\log\!\frac{2\sqrt{m}}{\delta}
}{m}.
\]
\end{lemma}

\begin{proof}
Fix $\boldsymbol{\theta}\in\Theta$ and define
\[
X(\boldsymbol{\theta})
:=\sum_{j=1}^{m}X_j(\boldsymbol{\theta}),
\qquad
\mu(\boldsymbol{\theta})
:=\mathbb{E}[X(\boldsymbol{\theta})]
=m\mathcal R(\boldsymbol{\theta}).
\]
Consider
\[
f_{\boldsymbol{\theta}}(x)
:=
\exp\!\left(
m\,\kl\!\left(
\frac{x}{m}
\,\middle\|\,
\mathcal R(\boldsymbol{\theta})
\right)
\right).
\]
Since $q\mapsto\kl(q\|p)$ is convex for fixed $p$, and the exponential is convex and non-decreasing, $f_{\boldsymbol{\theta}}$ is convex. Lemma~\ref{lem:berend-tassa} therefore yields, for $Y\sim\mathrm{Bin}(m,\mathcal R(\boldsymbol{\theta}))$,
\[
\mathbb{E}_{S}
\exp\!\left(
m\,\kl\!\left(
\widehat{\mathcal R}_S(\boldsymbol{\theta})
\,\middle\|\,
\mathcal R(\boldsymbol{\theta})
\right)
\right)
\le
\mathbb{E}_{Y}
\exp\!\left(
m\,\kl\!\left(
\frac{Y}{m}
\,\middle\|\,
\mathcal R(\boldsymbol{\theta})
\right)
\right).
\]
Applying Lemma~\ref{lem:maurer-kl},
\[
\mathbb{E}_{S}
\exp\!\left(
m\,\kl\!\left(
\widehat{\mathcal R}_S(\boldsymbol{\theta})
\,\middle\|\,
\mathcal R(\boldsymbol{\theta})
\right)
\right)
\le 2\sqrt{m}.
\]
Taking expectation with respect to the prior and using Tonelli's theorem gives
\[
\mathbb{E}_{S}\mathbb{E}_{\boldsymbol{\theta}\sim\pi}
\exp\!\left(
m\,\kl\!\left(
\widehat{\mathcal R}_S(\boldsymbol{\theta})
\,\middle\|\,
\mathcal R(\boldsymbol{\theta})
\right)
\right)
\le 2\sqrt{m}.
\]
Hence, by Markov's inequality, with probability at least $1-\delta$,
\[
\mathbb{E}_{\boldsymbol{\theta}\sim\pi}
\exp\!\left(
m\,\kl\!\left(
\widehat{\mathcal R}_S(\boldsymbol{\theta})
\,\middle\|\,
\mathcal R(\boldsymbol{\theta})
\right)
\right)
\le \frac{2\sqrt{m}}{\delta}.
\]

On this event, the Donsker--Varadhan change-of-measure inequality implies that, for every posterior $\rho$,
\[
m\,
\mathbb{E}_{\boldsymbol{\theta}\sim\rho}
\left[
\kl\!\left(
\widehat{\mathcal R}_S(\boldsymbol{\theta})
\,\middle\|\,
\mathcal R(\boldsymbol{\theta})
\right)
\right]
\le
\KL(\rho\|\pi)
+\log\frac{2\sqrt{m}}{\delta}.
\]
Finally, since binary $\kl$ is jointly convex, Jensen's inequality gives
\[
\kl\!\left(
\mathbb{E}_{\boldsymbol{\theta}\sim\rho}
[\widehat{\mathcal R}_S(\boldsymbol{\theta})]
\,\middle\|\,
\mathbb{E}_{\boldsymbol{\theta}\sim\rho}
[\mathcal R(\boldsymbol{\theta})]
\right)
\le
\mathbb{E}_{\boldsymbol{\theta}\sim\rho}
\left[
\kl\!\left(
\widehat{\mathcal R}_S(\boldsymbol{\theta})
\,\middle\|\,
\mathcal R(\boldsymbol{\theta})
\right)
\right],
\]
which proves
\[
\kl\!\left(
\mathbb{E}_{\boldsymbol{\theta}\sim\rho}
[\widehat{\mathcal R}_S(\boldsymbol{\theta})]
\,\middle\|\,
\mathbb{E}_{\boldsymbol{\theta}\sim\rho}
[\mathcal R(\boldsymbol{\theta})]
\right)
\le
\frac{
\KL(\rho\|\pi)+\log\!\frac{2\sqrt{m}}{\delta}
}{m}.
\]
\end{proof}

\subsection{Proof of Theorem~\ref{thm:generalpbc_mtl}}
\label{app:proof_oracle_mtl_piml}

\begin{lemma}[\textbf{Finiteness and differentiability of the CGFs and Cramér transforms}]
\label{lem:cgf-cramer-regularity}
For each loss function $i$, let $\mathcal{R}_i(\boldsymbol{\theta}):=\mathbb{E}_{\vx\sim D_i}[\ell_i(\boldsymbol{\theta},\vx)]<\infty$ with $\ell_i(\boldsymbol{\theta},\vx)\ge 0$, and define the centered CGF
\[
\Lambda_{\boldsymbol{\theta}}^{(i)}(\lambda):=\log\mathbb{E}_{\vx\sim D_i}\exp\!\left(\lambda\bigl(\mathcal{R}_i(\boldsymbol{\theta})-\ell_i(\boldsymbol{\theta},\vx)\bigr)\right), \qquad \lambda\ge 0,
\]
with effective domain $D_{\Lambda_{\boldsymbol{\theta}}^{(i)}}:=\{\lambda\in\mathbb{R}_+:\Lambda_{\boldsymbol{\theta}}^{(i)}(\lambda)<\infty\}$, and Cramér transform
\[
\Lambda_{\boldsymbol{\theta}}^{(i)\star}(a):=\sup_{\lambda\ge 0}\left\{\lambda a-\Lambda_{\boldsymbol{\theta}}^{(i)}(\lambda)\right\}.
\]
Then $D_{\Lambda_{\boldsymbol{\theta}}^{(i)}}=[0,\infty)$ and, provided $\mathcal{R}_i(\boldsymbol{\theta})-\ell_i(\boldsymbol{\theta},\vx)$ is non-degenerate, $\Lambda_{\boldsymbol{\theta}}^{(i)}$ is strictly convex and $C^\infty$ on $(0,\infty)$. Moreover, letting
\[
B_{\boldsymbol{\theta}}^{(i)}:=\operatorname*{ess\,sup}_{\vx\sim D_i}\bigl(\mathcal{R}_i(\boldsymbol{\theta})-\ell_i(\boldsymbol{\theta},\vx)\bigr)
=\mathcal{R}_i(\boldsymbol{\theta})-\operatorname*{ess\,inf}_{\vx\sim D_i}\ell_i(\boldsymbol{\theta},\vx),
\]
we have $\Lambda_{\boldsymbol{\theta}}^{(i)\star}(a)<\infty$ for all $a\in[0,B_{\boldsymbol{\theta}}^{(i)})$, and $\Lambda_{\boldsymbol{\theta}}^{(i)\star}$ is differentiable on $(0,B_{\boldsymbol{\theta}}^{(i)})$. In particular, if $\operatorname*{ess\,inf}_{\vx\sim D_i}\ell_i(\boldsymbol{\theta},\vx)=0$, these properties hold on $[0,\mathcal{R}_i(\boldsymbol{\theta}))$ and $(0,\mathcal{R}_i(\boldsymbol{\theta}))$, respectively.
\end{lemma}

\begin{proof}
Fix $i$ and $\boldsymbol{\theta}$. Since $\ell_i(\boldsymbol{\theta},\vx)\ge 0$, for every $\lambda\ge 0$,
\[
\Lambda_{\boldsymbol{\theta}}^{(i)}(\lambda)
=\log\mathbb{E}_{\vx\sim D_i}e^{\lambda(\mathcal{R}_i(\boldsymbol{\theta})-\ell_i(\boldsymbol{\theta},\vx))}
\le \lambda\mathcal{R}_i(\boldsymbol{\theta})<\infty.
\]
Hence $D_{\Lambda_{\boldsymbol{\theta}}^{(i)}}=[0,\infty)$. Standard properties of CGFs imply that $\Lambda_{\boldsymbol{\theta}}^{(i)}\in C^\infty((0,\infty))$ and is strictly convex whenever $\mathcal{R}_i(\boldsymbol{\theta})-\ell_i(\boldsymbol{\theta},\vx)$ is non-degenerate.

Since $\Lambda_{\boldsymbol{\theta}}^{(i)}(0)=0$ and
\[
\Lambda_{\boldsymbol{\theta}}^{(i)\prime}(0)
=\mathbb{E}_{\vx\sim D_i}\!\left[\mathcal{R}_i(\boldsymbol{\theta})-\ell_i(\boldsymbol{\theta},\vx)\right]=0,
\]
strict convexity implies that $\Lambda_{\boldsymbol{\theta}}^{(i)\prime}$ is strictly increasing on $(0,\infty)$, with
\[
\inf_{\lambda>0}\Lambda_{\boldsymbol{\theta}}^{(i)\prime}(\lambda)=0,
\qquad
\lim_{\lambda\to\infty}\Lambda_{\boldsymbol{\theta}}^{(i)\prime}(\lambda)=B_{\boldsymbol{\theta}}^{(i)}.
\]
Thus $\Lambda_{\boldsymbol{\theta}}^{(i)\prime}$ maps $(0,\infty)$ bijectively onto $(0,B_{\boldsymbol{\theta}}^{(i)})$. Therefore, for every $a\in(0,B_{\boldsymbol{\theta}}^{(i)})$, there is a unique $\lambda_a>0$ such that $\Lambda_{\boldsymbol{\theta}}^{(i)\prime}(\lambda_a)=a$. Since $\lambda\mapsto\lambda a-\Lambda_{\boldsymbol{\theta}}^{(i)}(\lambda)$ is strictly concave, $\lambda_a$ is its unique maximiser and
\[
\Lambda_{\boldsymbol{\theta}}^{(i)\star}(a)
=\lambda_a a-\Lambda_{\boldsymbol{\theta}}^{(i)}(\lambda_a)<\infty.
\]
Moreover, $\Lambda_{\boldsymbol{\theta}}^{(i)\star}(0)=0$ since $\Lambda_{\boldsymbol{\theta}}^{(i)}(0)=0$ and Jensen's inequality gives $\Lambda_{\boldsymbol{\theta}}^{(i)}(\lambda)\ge0$. Hence $\Lambda_{\boldsymbol{\theta}}^{(i)\star}$ is finite on $[0,B_{\boldsymbol{\theta}}^{(i)})$.

Finally, uniqueness of the maximiser and standard convex duality imply that $\Lambda_{\boldsymbol{\theta}}^{(i)\star}$ is differentiable on $(0,B_{\boldsymbol{\theta}}^{(i)})$, with
\[
\Lambda_{\boldsymbol{\theta}}^{(i)\star\prime}(a)=\lambda_a
=\bigl(\Lambda_{\boldsymbol{\theta}}^{(i)\prime}\bigr)^{-1}(a).
\]
Since $\operatorname*{ess\,inf}_{\vx\sim D_i}\ell_i(\boldsymbol{\theta},\vx)=0$ (see Section~\ref{ss:piml}, we obtain $B_{\boldsymbol{\theta}}^{(i)}=\mathcal{R}_i(\boldsymbol{\theta})$, which gives the stated particular case.
\end{proof}

\begin{remark}
\label{rem:posterior-cgf-finiteness}
The proof also gives $0\le\Lambda_{\boldsymbol{\theta}}^{(i)}(\lambda)\le\lambda\mathcal{R}_i(\boldsymbol{\theta})$ for every admissible $\lambda\ge0$. Consequently, $\mathbb{E}_{\boldsymbol{\theta}\sim\rho}[\mathcal{R}_i(\boldsymbol{\theta})]<\infty$ is sufficient to guarantee
\[
\mathbb{E}_{\boldsymbol{\theta}\sim\rho}\!\left[\Lambda_{\boldsymbol{\theta}}^{(i)}(\lambda)\right]
\le\lambda\mathbb{E}_{\boldsymbol{\theta}\sim\rho}\!\left[\mathcal{R}_i(\boldsymbol{\theta})\right]<\infty,
\qquad \lambda\ge0.
\]
\end{remark}

\begin{proposition}[\textbf{Multi-task Cramér--Chernoff}]
\label{mtl_cramer_chernoff}
Let $M=\sum_{i=1}^{N_L}m_i$ and define the sample-weighted CGF and its Cramér transform as
\[
\Lambda_{\boldsymbol{\theta}}^{\mathcal S}(\lambda)
:=\frac{1}{M}\sum_{i=1}^{N_L}m_i\Lambda_{\boldsymbol{\theta}}^{(i)}(\lambda),
\qquad
\Lambda_{\boldsymbol{\theta}}^{\mathcal S\star}(a)
:=\sup_{\lambda\ge0}\left\{\lambda a-\Lambda_{\boldsymbol{\theta}}^{\mathcal S}(\lambda)\right\}.
\]
For any $\boldsymbol{\theta}\in\Theta$ and $a\in\mathbb{R}$,
\[
\Prob_{S}\!\left[\mathrm{gen}^{\mathcal S}(\boldsymbol{\theta},S)\ge a\right]
\le e^{-M\Lambda_{\boldsymbol{\theta}}^{\mathcal S\star}(a)}.
\]
Moreover, by Lemma~\ref{lem:cgf-cramer-regularity}, $\Lambda_{\boldsymbol{\theta}}^{\mathcal S}$ is finite and $C^\infty$ on $(0,\infty)$ and, provided at least one positively weighted loss is non-degenerate, it is strictly convex. Defining
\[
B_{\boldsymbol{\theta}}^{\mathcal S}
:=\frac{1}{M}\sum_{i=1}^{N_L}m_iB_{\boldsymbol{\theta}}^{(i)},
\]
we further have $\Lambda_{\boldsymbol{\theta}}^{\mathcal S\star}(a)<\infty$ for $a\in[0,B_{\boldsymbol{\theta}}^{\mathcal S})$, and $\Lambda_{\boldsymbol{\theta}}^{\mathcal S\star}$ is differentiable and strictly increasing on $(0,B_{\boldsymbol{\theta}}^{\mathcal S})$.
\end{proposition}

\begin{proof}
Applying Markov's inequality for any $\lambda>0$ and using the i.i.d.\ assumption (see Section~\ref{ss:piml}),
\begin{align*}
\Prob_{S}\!\left[\mathrm{gen}^{\mathcal S}(\boldsymbol{\theta},S)\ge a\right]
&\le e^{-M\lambda a}\E_{S}\!\left[e^{M\lambda\mathrm{gen}^{\mathcal S}(\boldsymbol{\theta},S)}\right] \\
&=e^{-M\lambda a}\E_{S}\!\left[\exp\!\left(\lambda\sum_{i=1}^{N_L}\sum_{j=1}^{m_i}
\bigl(\mathcal{R}_i(\boldsymbol{\theta})-\ell_i(\boldsymbol{\theta},\boldsymbol\vx_{i,j})\bigr)\right)\right] \\
&=e^{-M\lambda a}\prod_{i=1}^{N_L}\prod_{j=1}^{m_i}
\E_{\boldsymbol\vx_{i,j}\sim D_i}
\exp\!\left(\lambda\bigl(\mathcal{R}_i(\boldsymbol{\theta})-\ell_i(\boldsymbol{\theta},\boldsymbol\vx_{i,j})\bigr)\right) \\
&=e^{-M\lambda a}e^{\sum_{i=1}^{N_L}m_i\Lambda_{\boldsymbol{\theta}}^{(i)}(\lambda)}
=e^{-M(\lambda a-\Lambda_{\boldsymbol{\theta}}^{\mathcal S}(\lambda))}.
\end{align*}
Taking the infimum over $\lambda>0$ yields
\[
\Prob_{S}\!\left[\mathrm{gen}^{\mathcal S}(\boldsymbol{\theta},S)\ge a\right]
\le e^{-M\sup_{\lambda\ge0}\{\lambda a-\Lambda_{\boldsymbol{\theta}}^{\mathcal S}(\lambda)\}}
=e^{-M\Lambda_{\boldsymbol{\theta}}^{\mathcal S\star}(a)}.
\]

For the regularity statement, $\Lambda_{\boldsymbol{\theta}}^{\mathcal S}$ is a positive weighted sum of the single-loss CGFs in Lemma~\ref{lem:cgf-cramer-regularity}; hence it is finite and $C^\infty$ on $(0,\infty)$, and strictly convex whenever at least one component with $m_i>0$ is non-degenerate. Moreover,
\[
\Lambda_{\boldsymbol{\theta}}^{\mathcal S\prime}(0)=0,
\qquad
\lim_{\lambda\to\infty}\Lambda_{\boldsymbol{\theta}}^{\mathcal S\prime}(\lambda)
=\frac{1}{M}\sum_{i=1}^{N_L}m_iB_{\boldsymbol{\theta}}^{(i)}
=B_{\boldsymbol{\theta}}^{\mathcal S}.
\]
The same convex-duality argument as in Lemma~\ref{lem:cgf-cramer-regularity} therefore gives the stated finiteness, differentiability, and monotonicity properties of $\Lambda_{\boldsymbol{\theta}}^{\mathcal S\star}$.
\end{proof}

\begin{lemma}[\textbf{Exponential domination of the sample-weighted Cramér transform}]
\label{lem:sample-cramer-exp}
For every $\boldsymbol{\theta}\in\Theta$ and $c\ge0$,
\[
\Prob_{S}\!\left[
M\Lambda_{\boldsymbol{\theta}}^{\mathcal S\star}
\bigl(\mathrm{gen}^{\mathcal S}(\boldsymbol{\theta},S)\bigr)\ge c
\right]
\le e^{-c}
=\Prob_{X\sim\mathrm{Exp}(1)}[X\ge c].
\]
\end{lemma}

\begin{proof}
Fix $\boldsymbol{\theta}$ and $c\ge0$. By Proposition~\ref{mtl_cramer_chernoff}, $\Lambda_{\boldsymbol{\theta}}^{\mathcal S\star}$ is continuous and strictly increasing on $(0,B_{\boldsymbol{\theta}}^{\mathcal S})$. Moreover,
\[
\mathrm{gen}^{\mathcal S}(\boldsymbol{\theta},S)
\le B_{\boldsymbol{\theta}}^{\mathcal S}
\qquad\text{a.s.}
\]
We distinguish three cases.

\emph{Case 1:} $c/M<\sup_{a\in(0,B_{\boldsymbol{\theta}}^{\mathcal S})}\Lambda_{\boldsymbol{\theta}}^{\mathcal S\star}(a)$. Using the generalised inverse of $\Lambda_{\boldsymbol{\theta}}^{\mathcal S\star}$ and Proposition~\ref{mtl_cramer_chernoff},
\[
\begin{aligned}
\Prob_S\!\left[
M\Lambda_{\boldsymbol{\theta}}^{\mathcal S\star}
\bigl(\mathrm{gen}^{\mathcal S}(\boldsymbol{\theta},S)\bigr)\ge c
\right]
&\le
\Prob_S\!\left[
\mathrm{gen}^{\mathcal S}(\boldsymbol{\theta},S)
\ge\bigl(\Lambda_{\boldsymbol{\theta}}^{\mathcal S\star}\bigr)^{-1}(c/M)
\right] \\
&\le
\exp\!\left[
-M\Lambda_{\boldsymbol{\theta}}^{\mathcal S\star}
\left(\bigl(\Lambda_{\boldsymbol{\theta}}^{\mathcal S\star}\bigr)^{-1}(c/M)\right)
\right]
=e^{-c}.
\end{aligned}
\]

\emph{Case 2:} $c/M=\sup_{a\in(0,B_{\boldsymbol{\theta}}^{\mathcal S})}\Lambda_{\boldsymbol{\theta}}^{\mathcal S\star}(a)$. Since $\Lambda_{\boldsymbol{\theta}}^{\mathcal S\star}(a)<c/M$ for every $a<B_{\boldsymbol{\theta}}^{\mathcal S}$, the event can only occur at $\mathrm{gen}^{\mathcal S}(\boldsymbol{\theta},S)=B_{\boldsymbol{\theta}}^{\mathcal S}$. Hence, by Proposition~\ref{mtl_cramer_chernoff} and monotonicity,
\[
\Prob_S\!\left[
M\Lambda_{\boldsymbol{\theta}}^{\mathcal S\star}
\bigl(\mathrm{gen}^{\mathcal S}(\boldsymbol{\theta},S)\bigr)\ge c
\right]
\le
e^{-M\Lambda_{\boldsymbol{\theta}}^{\mathcal S\star}(B_{\boldsymbol{\theta}}^{\mathcal S})}
\le e^{-c},
\]
with the first term interpreted as zero when $\Lambda_{\boldsymbol{\theta}}^{\mathcal S\star}(B_{\boldsymbol{\theta}}^{\mathcal S})=\infty$.

\emph{Case 3:} $c/M>\sup_{a\in(0,B_{\boldsymbol{\theta}}^{\mathcal S})}\Lambda_{\boldsymbol{\theta}}^{\mathcal S\star}(a)$. Again, the event can only occur at the boundary $B_{\boldsymbol{\theta}}^{\mathcal S}$. If $\Lambda_{\boldsymbol{\theta}}^{\mathcal S\star}(B_{\boldsymbol{\theta}}^{\mathcal S})<c/M$, the event is empty; otherwise Proposition~\ref{mtl_cramer_chernoff} gives
\[
\Prob_S\!\left[
M\Lambda_{\boldsymbol{\theta}}^{\mathcal S\star}
\bigl(\mathrm{gen}^{\mathcal S}(\boldsymbol{\theta},S)\bigr)\ge c
\right]
\le
e^{-M\Lambda_{\boldsymbol{\theta}}^{\mathcal S\star}(B_{\boldsymbol{\theta}}^{\mathcal S})}
\le e^{-c}.
\]
Combining the three cases proves the result.
\end{proof}

\GeneralPBChernoeffMTL*
\begin{proof}
\label{proof:pb_sobolev}
Apply Lemma~\ref{general_pb} with $\varphi(\boldsymbol{\theta},S)=M_1\Lambda_{\boldsymbol{\theta}}^{\mathcal S\star}(\mathrm{gen}^{\mathcal S}(\boldsymbol{\theta},S))$, where $0<M_1<M$. It holds with probability at least $1-\delta$ that
\begin{equation}
\E_{\boldsymbol{\theta}\sim\rho}M_1\Lambda_{\boldsymbol{\theta}}^{\mathcal S\star}(\mathrm{gen}^{\mathcal S}(\boldsymbol{\theta},S))
\leq \KL(\rho\|\pi)+\ln\left(\frac{1}{\delta}\E_S\E_{\boldsymbol{\theta}\sim\pi}
e^{M_1\Lambda_{\boldsymbol{\theta}}^{\mathcal S\star}(\mathrm{gen}^{\mathcal S}(\boldsymbol{\theta},S))}\right).
\label{eq:general_pb_cramer}
\end{equation}
Since the prior $\pi$ is independent of the training samples $S$, Tonelli's theorem allows us to swap the expectations and then control $\E_S \left[e^{M_1\Lambda_{\boldsymbol{\theta}}^{\mathcal S\star}(\mathrm{gen}^{\mathcal S}(\boldsymbol{\theta},S))}\right]$ for every fixed $\boldsymbol{\theta}\in\Theta$. In particular, Lemma~\ref{lem:sample-cramer-exp} gives
\begin{equation}
\Prob_S\!\left[M\Lambda_{\boldsymbol{\theta}}^{\mathcal S\star}(\mathrm{gen}^{\mathcal S}(\boldsymbol{\theta},S))\ge a\right]
\leq\Prob_{X\sim\mathrm{Exp}(1)}[X\ge a].
\end{equation}
Since $X\sim\mathrm{Exp}(1)$ implies $(M_1/M)X\sim\mathrm{Exp}(M/M_1)$, we obtain
\begin{equation}
\Prob_S\!\left[M_1\Lambda_{\boldsymbol{\theta}}^{\mathcal S\star}(\mathrm{gen}^{\mathcal S}(\boldsymbol{\theta},S))\ge a\right]
\leq\Prob_{X\sim\mathrm{Exp}(M/M_1)}[X\ge a].
\end{equation}
Since $X\sim\mathrm{Exp}(M/M_1)$ implies $e^X\sim\mathrm{Pareto}(M/M_1,1)$, for every $t>1$ we have
\begin{equation}
\label{eq:pareto}
\Prob_S\!\left[e^{M_1\Lambda_{\boldsymbol{\theta}}^{\mathcal S\star}(\mathrm{gen}^{\mathcal S}(\boldsymbol{\theta},S))}\ge t\right]
\leq\Prob_{X\sim\mathrm{Pareto}(M/M_1,1)}[X\ge t].
\end{equation}
Let $Z:=e^{M_1\Lambda_{\boldsymbol{\theta}}^{\mathcal S\star}(\mathrm{gen}^{\mathcal S}(\boldsymbol{\theta},S))}$. Since $Z\ge1$, its expectation can be written as $\E[Z]=1+\int_1^\infty\Prob(Z\ge t)\,dt$. Hence, using~\eqref{eq:pareto},
\begin{align}
\E_S \left[e^{M_1\Lambda_{\boldsymbol{\theta}}^{\mathcal S\star}(\mathrm{gen}^{\mathcal S}(\boldsymbol{\theta},S))} \right]
&=1+\int_1^\infty\Prob_S\!\left[e^{M_1\Lambda_{\boldsymbol{\theta}}^{\mathcal S\star}(\mathrm{gen}^{\mathcal S}(\boldsymbol{\theta},S))}\ge t\right]dt \nonumber \\
&\leq1+\int_1^\infty\Prob_{X\sim\mathrm{Pareto}(M/M_1,1)}[X\ge t]dt \nonumber \\
&=\E_{X\sim\mathrm{Pareto}(M/M_1,1)}[X]
=\frac{\frac{M}{M_1}}{\frac{M}{M_1}-1}
=\frac{M}{M-M_1}.
\label{eq:expectation_cramer_chernoff}
\end{align}
Moreover,
\begin{align}
\E_{\boldsymbol{\theta}\sim\rho}M_1\Lambda_{\boldsymbol{\theta}}^{\mathcal S\star}(\mathrm{gen}^{\mathcal S}(\boldsymbol{\theta},S))
&=\E_{\boldsymbol{\theta}\sim\rho}M_1\sup_{\lambda\ge0}\left\{\lambda\mathrm{gen}^{\mathcal S}(\boldsymbol{\theta},S)-\Lambda_{\boldsymbol{\theta}}^{\mathcal S}(\lambda)\right\} \nonumber \\
&\geq M_1\sup_{\lambda\ge0}\E_{\boldsymbol{\theta}\sim\rho}\left[\lambda\mathrm{gen}^{\mathcal S}(\boldsymbol{\theta},S)-\Lambda_{\boldsymbol{\theta}}^{\mathcal S}(\lambda)\right] \nonumber \\
&=M_1\sup_{\lambda\ge0}\left\{\lambda\mathrm{gen}^{\mathcal S}(\rho,S)-\E_{\boldsymbol{\theta}\sim\rho}\!\left[\Lambda_{\boldsymbol{\theta}}^{\mathcal S}(\lambda)\right]\right\}
=:M_1\Lambda_{\rho}^{\mathcal S\star}(\mathrm{gen}^{\mathcal S}(\rho,S)),
\label{eq:bounded_mtl_chernoff}
\end{align}
where
\[
\Lambda_{\rho}^{\mathcal S}(\lambda):=\E_{\boldsymbol{\theta}\sim\rho}\!\left[\Lambda_{\boldsymbol{\theta}}^{\mathcal S}(\lambda)\right],
\qquad
\Lambda_{\rho}^{\mathcal S\star}(a):=\sup_{\lambda\ge0}\left\{\lambda a-\Lambda_{\rho}^{\mathcal S}(\lambda)\right\}.
\]
Plugging~\eqref{eq:expectation_cramer_chernoff} and~\eqref{eq:bounded_mtl_chernoff} into~\eqref{eq:general_pb_cramer}, we have
\begin{align*}
\Prob_S\!\left[
M_1\Lambda_{\rho}^{\mathcal S\star}(\mathrm{gen}^{\mathcal S}(\rho,S))
\leq\KL(\rho\|\pi)+\ln\frac{1}{\delta}+\ln\frac{M}{M-M_1}
\right]\geq1-\delta.
\end{align*}
Choosing $M_1=M-1$ and using the definition of $\Lambda_{\rho}^{\mathcal S\star}$, we obtain, with probability at least $1-\delta$,
\begin{align*}
(M-1)\sup_{\lambda\ge0}\left\{\lambda\mathrm{gen}^{\mathcal S}(\rho,S)-\E_{\boldsymbol{\theta}\sim\rho}\!\left[\Lambda_{\boldsymbol{\theta}}^{\mathcal S}(\lambda)\right]\right\}
\leq\KL(\rho\|\pi)+\ln\frac{M}{\delta}.
\end{align*}
Hence, for every $\lambda>0$,
\[
\mathrm{gen}^{\mathcal S}(\rho,S)
\leq\frac{\KL(\rho\|\pi)+\ln\frac{M}{\delta}}{\lambda(M-1)}
+\frac{\E_{\boldsymbol{\theta}\sim\rho}\!\left[\Lambda_{\boldsymbol{\theta}}^{\mathcal S}(\lambda)\right]}{\lambda},
\]
and therefore
\begin{align*}
\Prob_S\!\left[
\mathrm{gen}^{\mathcal S}(\rho,S)
\leq
\inf_{\lambda>0}\left\{
\frac{\KL(\rho\|\pi)+\ln\frac{M}{\delta}}{\lambda(M-1)}
+\frac{\E_{\boldsymbol{\theta}\sim\rho}\!\left[\Lambda_{\boldsymbol{\theta}}^{\mathcal S}(\lambda)\right]}{\lambda}
\right\}
\right]\geq1-\delta.
\end{align*}
\end{proof}

\subsection{Proof of Theorem~\ref{thm:variance_gradient_calibration}}
\label{app:proof_variance_gradient_calibration}

\PBCal*
\begin{proof}
The proof separately controls the numerator and denominator of the bounded calibration coefficient
$
C_{K,G}(\rho)
=
\frac{\E_{\boldsymbol{\theta}\sim\rho}\Var_{\vx\sim D}[\ell_K(\vx,\boldsymbol{\theta})]}
{\E_{\boldsymbol{\theta}\sim\rho}\E_{\vx\sim D}[g_{K,G}(\vx,\boldsymbol{\theta})]}.
$
We obtain a PAC-Bayes upper confidence bound on the posterior-averaged variance in the numerator and a PAC-Bayes lower confidence bound on the posterior-averaged clipped-gradient energy in the denominator. Their ratio gives $\overline C_{K,G}(\rho,\delta_C)$, while Assumption~\ref{ass:calibration_dominance} transfers this control to the population variance-gradient coefficient $C(\rho)$ of the original, unclipped loss.

\paragraph{Step 1: PAC-Bayes--$\kl$ upper bound on the calibration variance.}
For a fixed $\boldsymbol{\theta}$, define the symmetric kernel
$
h_{K,\boldsymbol{\theta}}(\vx,\vx')
:=
\frac{(\ell_K(\vx,\boldsymbol{\theta})-\ell_K(\vx',\boldsymbol{\theta}))^2}{K^2}.
$
Since $0\leq\ell_K\leq K$, we have $0\leq h_{K,\boldsymbol{\theta}}\leq1$. If $\vx,\vx'\overset{\mathrm{i.i.d.}}{\sim}D$, then
\begin{align}
H_K(\boldsymbol{\theta})
&:=
\E_{\vx,\vx'\sim D}
[h_{K,\boldsymbol{\theta}}(\vx,\vx')]
\nonumber\\
&=
\frac{1}{K^2}
\E_{\vx,\vx'\sim D}
\left[
\left(
\ell_K(\vx,\boldsymbol{\theta})
-
\ell_K(\vx',\boldsymbol{\theta})
\right)^2
\right]
\nonumber\\
&=
\frac{2}{K^2}
\Var_{\vx\sim D}
[\ell_K(\vx,\boldsymbol{\theta})],
\label{eq:proof_calib_population_kernel}
\end{align}
where the last equality follows from the standard identity $\E[(X-X')^2]=2\Var(X)$ for independent copies $X,X'$.

The empirical counterpart associated with the calibration sample $S^{\mathrm{cal}}=(\vx_1,\ldots,\vx_{M_c})$ is the second-order U-statistic
\begin{align}
H_{K,S}(\boldsymbol{\theta})
&:=
\frac{2}{M_c(M_c-1)}
\sum_{1\leq j<k\leq M_c}
h_{K,\boldsymbol{\theta}}(\vx_j,\vx_k)
\nonumber\\
&=
\frac{2}{K^2}
\widehat V_{K,S}(\boldsymbol{\theta}),
\label{eq:proof_calib_empirical_kernel}
\end{align}
where $\widehat V_{K,S}$ is the U-statistic defined in Section~\ref{sec:variance_gradient_calibration}. In particular, $H_{K,S}(\boldsymbol{\theta})\in[0,1]$ and $\E_{S^{\mathrm{cal}}}[H_{K,S}(\boldsymbol{\theta})]=H_K(\boldsymbol{\theta})$.

The terms entering $H_{K,S}$ are not mutually independent because the same calibration observation appears in several pairs. The permutation argument provides a convenient way to reduce this dependent average over all pairs to averages over disjoint, and hence independent pairs. Let $m_c:=\lfloor M_c/2\rfloor$. For any permutation $\tau$ of $\{1,\ldots,M_c\}$, reorder the calibration observations according to $\tau$ and group the first $2m_c$ observations into the consecutive pairs $(\vx_{\tau(1)},\vx_{\tau(2)}),(\vx_{\tau(3)},\vx_{\tau(4)}),\ldots,(\vx_{\tau(2m_c-1)},\vx_{\tau(2m_c)})$. If $M_c$ is odd, the single remaining observation is simply unused for that pairing. Define the corresponding paired average as
$
H_{K,S,\tau}(\boldsymbol{\theta})
:=
\frac{1}{m_c}
\sum_{j=1}^{m_c}
h_{K,\boldsymbol{\theta}}
\left(
\vx_{\tau(2j-1)},
\vx_{\tau(2j)}
\right).
$
For each fixed permutation $\tau$, no calibration observation occurs in more than one of these $m_c$ pairs. Since the original observations $\vx_1,\ldots,\vx_{M_c}$ are i.i.d., these disjoint pairs are mutually independent under $S^{\mathrm{cal}}\sim D^{M_c}$. Thus, $H_{K,S,\tau}(\boldsymbol{\theta})$ is an ordinary empirical average of $m_c$ independent $[0,1]$-valued random variables, each having expectation $H_K(\boldsymbol{\theta})$.

The key point is that this pairing is only an analytical device: we do not replace the complete U-statistic by one particular pairing. Instead, we average over all permutations. Because the kernel is symmetric, every unordered pair $(\vx_a,\vx_b)$, $a<b$, has the same probability of occupying one of the $m_c$ paired positions under a uniformly random permutation. Consequently, averaging $H_{K,S,\tau}(\boldsymbol{\theta})$ uniformly over $\tau$ assigns the same weight to every unordered pair and exactly recovers the complete U-statistic,
\begin{equation}
H_{K,S}(\boldsymbol{\theta})
=
\E_{\tau}
[H_{K,S,\tau}(\boldsymbol{\theta})],
\label{eq:proof_calib_permutation_representation}
\end{equation}
where the expectation is taken uniformly over permutations. Intuitively, each fixed permutation selects a collection of independent pairs from the calibration sample, while averaging over all such selections restores the original average over all $\binom{M_c}{2}$ pairs. This is why the complete U-statistic can retain all pairwise information while its concentration is controlled using the effective sample size $m_c=\lfloor M_c/2\rfloor$ associated with the maximum number of disjoint pairs.

For every fixed $\tau$, Lemma~\ref{lem:pac-bayes-kl}, applied to the $m_c$ independent $[0,1]$-valued paired observations, gives the PAC-Bayes binary-$\kl$ control corresponding to $H_{K,S,\tau}$. The same result extends to the complete U-statistic through~\eqref{eq:proof_calib_permutation_representation}: since binary $\kl$ is jointly convex,
\begin{align}
&\kl
\left(
\E_{\boldsymbol{\theta}\sim\rho}
[H_{K,S}(\boldsymbol{\theta})]
\middle\|
\E_{\boldsymbol{\theta}\sim\rho}
[H_K(\boldsymbol{\theta})]
\right)
\nonumber\\
&=
\kl
\left(
\E_{\tau}
\E_{\boldsymbol{\theta}\sim\rho}
[H_{K,S,\tau}(\boldsymbol{\theta})]
\middle\|
\E_{\tau}
\E_{\boldsymbol{\theta}\sim\rho}
[H_K(\boldsymbol{\theta})]
\right)
\nonumber\\
&\leq
\E_{\tau}
\kl
\left(
\E_{\boldsymbol{\theta}\sim\rho}
[H_{K,S,\tau}(\boldsymbol{\theta})]
\middle\|
\E_{\boldsymbol{\theta}\sim\rho}
[H_K(\boldsymbol{\theta})]
\right).
\label{eq:proof_calib_kl_convexity}
\end{align}
Equivalently, the usual permutation reduction for U-statistics allows Lemma~\ref{lem:pac-bayes-kl} to be applied with effective sample size $m_c$. Hence, with probability at least $1-\delta_V$ over $S^{\mathrm{cal}}$, simultaneously for every $\rho\ll\pi$,
\begin{equation}
\kl
\left(
\E_{\boldsymbol{\theta}\sim\rho}
[H_{K,S}(\boldsymbol{\theta})]
\middle\|
\E_{\boldsymbol{\theta}\sim\rho}
[H_K(\boldsymbol{\theta})]
\right)
\leq
\frac{
\KL(\rho\|\pi)
+
\log\frac{2\sqrt{m_c}}{\delta_V}
}{
m_c
}.
\label{eq:proof_calib_variance_binary_kl}
\end{equation}

We allocate half of the calibration failure probability to this event, i.e., $\delta_V=\delta_C/2$. By definition,
$
\varepsilon_V(\rho,\delta_C)
=
\frac{\KL(\rho\|\pi)+\log\frac{4\sqrt{m_c}}{\delta_C}}{m_c},
$
so~\eqref{eq:proof_calib_variance_binary_kl} becomes
\[
\kl
\left(
\E_{\rho}[H_{K,S}]
\middle\|
\E_{\rho}[H_K]
\right)
\leq
\varepsilon_V(\rho,\delta_C).
\]
The definition of $\kl_+^{-1}$ therefore implies
\[
\E_{\boldsymbol{\theta}\sim\rho}
[H_K(\boldsymbol{\theta})]
\leq
\kl_+^{-1}
\left(
\E_{\boldsymbol{\theta}\sim\rho}
[H_{K,S}(\boldsymbol{\theta})],
\varepsilon_V(\rho,\delta_C)
\right).
\]
Substituting~\eqref{eq:proof_calib_population_kernel} and~\eqref{eq:proof_calib_empirical_kernel} yields
\begin{equation}
\E_{\boldsymbol{\theta}\sim\rho}
\Var_{\vx\sim D}
[\ell_K(\vx,\boldsymbol{\theta})]
\leq
\frac{K^2}{2}
\kl_+^{-1}
\left(
\frac{
2\E_{\boldsymbol{\theta}\sim\rho}
[\widehat V_{K,S}(\boldsymbol{\theta})]
}{K^2},
\varepsilon_V(\rho,\delta_C)
\right).
\label{eq:clipped_variance_ucb_proof}
\end{equation}
This provides the required upper confidence bound on the numerator of $C_{K,G}(\rho)$.

\paragraph{Step 2: PAC-Bayes--$\kl$ lower bound on the calibration gradient energy.}
We now control the denominator. Since $g_{K,G}(\vx,\boldsymbol{\theta})=\min\{\|\nabla_{\vx}\ell_K(\vx,\boldsymbol{\theta})\|_2^2,G\}$, we have
$
0\leq g_{K,G}(\vx,\boldsymbol{\theta})/G\leq1.
$
For fixed $\boldsymbol{\theta}$, define the empirical clipped-gradient energy
$
\widehat J_{K,G,S}(\boldsymbol{\theta})
:=
\frac{1}{M_c}
\sum_{j=1}^{M_c}
g_{K,G}(\vx_j,\boldsymbol{\theta})
$
and its population counterpart
$
J_{K,G}(\boldsymbol{\theta})
:=
\E_{\vx\sim D}
[g_{K,G}(\vx,\boldsymbol{\theta})].
$
Consequently,
\[
\frac{\widehat J_{K,G,S}(\boldsymbol{\theta})}{G}
=
\frac{1}{M_c}
\sum_{j=1}^{M_c}
\frac{g_{K,G}(\vx_j,\boldsymbol{\theta})}{G},
\qquad
\frac{J_{K,G}(\boldsymbol{\theta})}{G}
=
\E_{\vx\sim D}
\left[
\frac{g_{K,G}(\vx,\boldsymbol{\theta})}{G}
\right],
\]
and both quantities belong to $[0,1]$.

Unlike the U-statistic in Step~1, the empirical gradient energy is an ordinary average of the $M_c$ i.i.d.\ calibration observations. We may therefore apply Lemma~\ref{lem:pac-bayes-kl} directly with sample size $M_c$. With probability at least $1-\delta_J$ over $S^{\mathrm{cal}}$, simultaneously for every $\rho\ll\pi$,
\begin{equation}
\kl
\left(
\frac{
\E_{\boldsymbol{\theta}\sim\rho}
[\widehat J_{K,G,S}(\boldsymbol{\theta})]
}{G}
\middle\|
\frac{
\E_{\boldsymbol{\theta}\sim\rho}
[J_{K,G}(\boldsymbol{\theta})]
}{G}
\right)
\leq
\frac{
\KL(\rho\|\pi)
+
\log\frac{2\sqrt{M_c}}{\delta_J}
}{
M_c
}.
\label{eq:proof_calib_gradient_binary_kl}
\end{equation}

We allocate the remaining half of the calibration failure probability to this event, $\delta_J=\delta_C/2$. Since
$
\varepsilon_J(\rho,\delta_C)
=
\frac{\KL(\rho\|\pi)+\log\frac{4\sqrt{M_c}}{\delta_C}}{M_c},
$
~\eqref{eq:proof_calib_gradient_binary_kl} gives
\[
\kl
\left(
\frac{
\E_{\rho}
[\widehat J_{K,G,S}]
}{G}
\middle\|
\frac{
\E_{\rho}
[J_{K,G}]
}{G}
\right)
\leq
\varepsilon_J(\rho,\delta_C).
\]
By definition of the lower inverse $\kl_-^{-1}$,
\[
\frac{
\E_{\boldsymbol{\theta}\sim\rho}
[J_{K,G}(\boldsymbol{\theta})]
}{G}
\geq
\kl_-^{-1}
\left(
\frac{
\E_{\boldsymbol{\theta}\sim\rho}
[\widehat J_{K,G,S}(\boldsymbol{\theta})]
}{G},
\varepsilon_J(\rho,\delta_C)
\right).
\]
Therefore,
\begin{equation}
\E_{\boldsymbol{\theta}\sim\rho}
\E_{\vx\sim D}
[g_{K,G}(\vx,\boldsymbol{\theta})]
\geq
G
\kl_-^{-1}
\left(
\frac{
\E_{\boldsymbol{\theta}\sim\rho}
[\widehat J_{K,G,S}(\boldsymbol{\theta})]
}{G},
\varepsilon_J(\rho,\delta_C)
\right).
\label{eq:clipped_gradient_lcb_proof}
\end{equation}
This is the required lower confidence bound on the denominator of $C_{K,G}(\rho)$.

\paragraph{Step 3: upper bound on the bounded calibration coefficient.}
Let $\mathcal E_V$ and $\mathcal E_J$ denote the events in~\eqref{eq:clipped_variance_ucb_proof} and~\eqref{eq:clipped_gradient_lcb_proof}. Since $\Prob(\mathcal E_V)\geq1-\delta_C/2$ and $\Prob(\mathcal E_J)\geq1-\delta_C/2$, a union bound gives
$
\Prob(\mathcal E_V\cap\mathcal E_J)\geq1-\delta_C.
$
Both events are uniform over $\rho\ll\pi$, and therefore their intersection is also uniform over all admissible posteriors.

On $\mathcal E_V\cap\mathcal E_J$, whenever the lower confidence bound in~\eqref{eq:clipped_gradient_lcb_proof} is strictly positive, we may divide the upper bound on the numerator by the lower bound on the denominator. Recalling
$
C_{K,G}(\rho)
:=
\frac{
\E_{\boldsymbol{\theta}\sim\rho}
\Var_{\vx\sim D}
[\ell_K(\vx,\boldsymbol{\theta})]
}{
\E_{\boldsymbol{\theta}\sim\rho}
\E_{\vx\sim D}
[g_{K,G}(\vx,\boldsymbol{\theta})]
},
$
we obtain
\begin{align}
C_{K,G}(\rho)
&\leq
\frac{
\displaystyle
\frac{K^2}{2}
\kl_+^{-1}
\left(
\frac{
2\E_{\boldsymbol{\theta}\sim\rho}
[\widehat V_{K,S}(\boldsymbol{\theta})]
}{K^2},
\varepsilon_V(\rho,\delta_C)
\right)
}{
\displaystyle
G
\kl_-^{-1}
\left(
\frac{
\E_{\boldsymbol{\theta}\sim\rho}
[\widehat J_{K,G,S}(\boldsymbol{\theta})]
}{G},
\varepsilon_J(\rho,\delta_C)
\right)
}
\nonumber\\
&=
\overline C_{K,G}(\rho,\delta_C).
\label{eq:proof_calib_bounded_coefficient}
\end{align}

\paragraph{Step 4: transfer to the population variance-gradient coefficient of the original loss.}
The preceding argument controls $C_{K,G}(\rho)$, which is defined using the bounded calibration quantities $\ell_K$ and $g_{K,G}$. The quantity required by the subsequent generalisation analysis is instead
$
C(\rho)
=
\frac{
\E_{\boldsymbol{\theta}\sim\rho}
\Var_{\vx\sim D}
[\ell(\vx,\boldsymbol{\theta})]
}{
\E_{\boldsymbol{\theta}\sim\rho}
\E_{\vx\sim D}
[g(\vx,\boldsymbol{\theta})]
},
$
which is associated with the original, unclipped loss. Assumption~\ref{ass:calibration_dominance} provides precisely the required link, namely
$
C(\rho)\leq C_{K,G}(\rho).
$
Therefore, on the same event of probability at least $1-\delta_C$, simultaneously for every $\rho\ll\pi$,
\[
C(\rho)
\leq
C_{K,G}(\rho)
\leq
\overline C_{K,G}(\rho,\delta_C).
\]
This proves Theorem~\ref{thm:variance_gradient_calibration}.
\end{proof}

\subsection{Proof of Theorem~\ref{thm:component_bound}}
\label{app:proof_component_bound}

\ComponentBound*
\begin{proof}
Fix a loss component $i\in\{1,\ldots,N_L\}$. We prove the result through four ingredients: calibration of the population variance-gradient coefficient, concentration of the bounded part of the input-gradient energy, binary-$\kl$ control of the residual gradient contribution, and a PAC-Bayes--Chernoff bound for the original loss. The corresponding failure probabilities are denoted by $\delta_C,\delta_J,\delta_T,\delta_G\in(0,1)$ and are chosen at the end of the proof such that $\delta_C+\delta_J+\delta_T+\delta_G=\delta$.

\paragraph{Step 1: calibration of the CGF--gradient relation.}
By Corollary~\ref{cor:cgf_grad}, there exists an event $\mathcal E_C$ with $\Prob(\mathcal E_C)\geq1-\delta_C$ such that, simultaneously for every component $i$ and every $\rho\ll\pi$, for every admissible $\lambda>0$, we have 
\begin{equation}
\E_{\boldsymbol{\theta}\sim\rho}
[\Lambda_{\boldsymbol{\theta}}^{(i)}(\lambda)]
\leq
\frac{\overline C_i(\rho,\delta_C)\lambda^2}{2}
\E_{\boldsymbol{\theta}\sim\rho}
\E_{\vx\sim D_i}
[g_i(\vx,\boldsymbol{\theta})],
\label{eq:proof_component_cgf}
\end{equation}
where $\overline C_i(\rho,\delta_C) := \kappa_{p,i} \overline C_{i,K_i,G_i}\left(\rho,\frac{\delta_C}{N_L}\right)$.

Thus, conditional on the calibration event, it remains to upper-bound the population input-gradient energy of the original loss.

\paragraph{Step 2: exact decomposition of the population input-gradient energy.}
For the prescribed threshold $G_i\in(0,U_i)$, recall that $g_i=g_{i,G_i}+r_{i,G_i}$ with $g_{i,G_i}:=\min\{g_i,G_i\}$ and $r_{i,G_i}:=(g_i-G_i)_+$. Assumption~\ref{ass:gradient_envelope} implies $0\leq g_i\leq U_i$, hence $0\leq g_{i,G_i}\leq G_i$ and $0\leq r_{i,G_i}\leq H_i:=U_i-G_i$. Consequently,
\begin{equation}
\E_{\rho}\E_{D_i}[g_i]
=
\E_{\rho}\E_{D_i}[g_{i,G_i}]
+
\E_{\rho}\E_{D_i}[r_{i,G_i}].
\label{eq:proof_component_gradient_decomposition}
\end{equation}
We control the two population terms separately.

\paragraph{Step 2a: bounded-gradient concentration.}
Since $g_{i,G_i}(\vx,\boldsymbol{\theta})\in[0,G_i]$, Hoeffding's lemma gives, for every fixed $\boldsymbol{\theta}$ and every $\eta>0$,
\[
\E_{\vx\sim D_i}
\exp\left(
\eta
\left[
\E_{\vx\sim D_i}g_{i,G_i}(\vx,\boldsymbol{\theta})
-
g_{i,G_i}(\vx,\boldsymbol{\theta})
\right]
\right)
\leq
\exp\left(\frac{\eta^2G_i^2}{8}\right).
\]
Applying the single-loss PAC-Bayes--Chernoff argument to the bounded random variable $g_{i,G_i}$ with confidence $\delta_J/N_L$ therefore gives, on an event $\mathcal E_{J,i}$ of probability at least $1-\delta_J/N_L$, simultaneously for every $\rho\ll\pi$,
\begin{align}
\E_{\rho}\E_{D_i}[g_{i,G_i}]
-
\frac{1}{m_i}
\E_{\rho}
\sum_{j=1}^{m_i}
g_{i,G_i}(\vx_{i,j},\boldsymbol{\theta})
\leq
\frac{
\KL(\rho\|\pi)
+
\ln\frac{N_Lm_i}{\delta_J}
}{
\eta(m_i-1)
}
+
\frac{\eta G_i^2}{8}.
\label{eq:proof_component_gradient_eta}
\end{align}
The right-hand side has the form $a/\eta+b\eta$, with $a=\frac{\KL(\rho\|\pi)+\ln\frac{N_Lm_i}{\delta_J}}{m_i-1}$ and $b=G_i^2/8$. Its minimum over $\eta>0$ is $2\sqrt{ab}$, reached at $\eta=\sqrt{a/b}$. Hence,
\begin{equation}
\E_{\rho}\E_{D_i}[g_{i,G_i}]
\leq
\frac{1}{m_i}
\E_{\rho}
\sum_{j=1}^{m_i}
g_{i,G_i}(\vx_{i,j},\boldsymbol{\theta})
+
G_i
\sqrt{
\frac{
\KL(\rho\|\pi)
+
\ln\frac{N_Lm_i}{\delta_J}
}{
2(m_i-1)
}
}.
\label{eq:proof_component_gradient_clipped}
\end{equation}

At this point, the concentration argument necessarily involves $g_{i,G_i}$ because its bounded range $[0,G_i]$ is what permits the use of Hoeffding's lemma. However, the empirical term in~\eqref{eq:proof_component_gradient_clipped} can be upper-bounded by the empirical input-gradient energy of the original loss. Indeed, $g_{i,G_i}\leq g_i$ pointwise, and therefore
\begin{equation}
\frac{1}{m_i}
\E_{\rho}
\sum_{j=1}^{m_i}
g_{i,G_i}(\vx_{i,j},\boldsymbol{\theta})
\leq
\frac{1}{m_i}
\E_{\rho}
\sum_{j=1}^{m_i}
g_i(\vx_{i,j},\boldsymbol{\theta})
=:
\widehat J_i(\rho).
\label{eq:proof_component_unclipped_empirical}
\end{equation}
Combining~\eqref{eq:proof_component_gradient_clipped} and~\eqref{eq:proof_component_unclipped_empirical}, multiplying by $m_i\overline C_i(\rho,\delta_C)$, and defining
$
\widehat{\mathcal G}_i(\rho,\delta_C)
:=
\overline C_i(\rho,\delta_C)
\E_{\rho}
\sum_{j=1}^{m_i}
g_i(\vx_{i,j},\boldsymbol{\theta}),
$
and
$
\Delta_{i,G_i}(\rho,\delta_C,\delta_J)
:=
m_i\overline C_i(\rho,\delta_C)G_i
\sqrt{
\frac{
\KL(\rho\|\pi)
+
\ln\frac{N_Lm_i}{\delta_J}
}{
2(m_i-1)
}
},
$
gives
\begin{equation}
m_i\overline C_i(\rho,\delta_C)
\E_{\rho}\E_{D_i}[g_{i,G_i}]
\leq
\widehat{\mathcal G}_i(\rho,\delta_C)
+
\Delta_{i,G_i}(\rho,\delta_C,\delta_J).
\label{eq:proof_component_gradient_empirical}
\end{equation}
Taking a union bound over the $N_L$ events $\{\mathcal E_{J,i}\}_{i=1}^{N_L}$ gives an event $\mathcal E_J$ with $\Prob(\mathcal E_J)\geq1-\delta_J$ on which~\eqref{eq:proof_component_gradient_empirical} holds simultaneously for every component and every posterior.

\paragraph{Step 2b: binary-$\kl$ control of the residual gradient.}
The residual term cannot be discarded because the original population gradient energy satisfies~\eqref{eq:proof_component_gradient_decomposition}. Since $0\leq r_{i,G_i}\leq H_i$, define the normalised residual
$
Z_{i,G_i}(\vx,\boldsymbol{\theta})
:=
r_{i,G_i}(\vx,\boldsymbol{\theta})/H_i
\in[0,1].
$
Its empirical and population posterior averages are
$
\widehat q_{i,G_i}(\rho)
:=
\frac{1}{m_i}
\sum_{j=1}^{m_i}
\E_{\boldsymbol{\theta}\sim\rho}
[Z_{i,G_i}(\vx_{i,j},\boldsymbol{\theta})]
$
and
$
q_{i,G_i}(\rho)
:=
\E_{\boldsymbol{\theta}\sim\rho}
\E_{\vx\sim D_i}
[Z_{i,G_i}(\vx,\boldsymbol{\theta})].
$
By construction,
\begin{equation}
\E_{\rho}\E_{D_i}[r_{i,G_i}]
=
H_iq_{i,G_i}(\rho).
\label{eq:proof_component_tail_relation}
\end{equation}

Applying the PAC-Bayes binary-$\kl$ inequality to $Z_{i,G_i}\in[0,1]$ with confidence $\delta_T/N_L$ gives an event $\mathcal E_{T,i}$ of probability at least $1-\delta_T/N_L$ such that, simultaneously for every $\rho\ll\pi$,
\begin{equation}
\kl
\left(
\widehat q_{i,G_i}(\rho)
\middle\|
q_{i,G_i}(\rho)
\right)
\leq
\frac{
\KL(\rho\|\pi)
+
\log\frac{2N_L\sqrt{m_i}}{\delta_T}
}{
m_i
}.
\label{eq:proof_component_tail_kl}
\end{equation}
By the definition of the upper inverse $\kl_+^{-1}$,
\[
q_{i,G_i}(\rho)
\leq
\kl_+^{-1}
\left(
\widehat q_{i,G_i}(\rho),
\frac{
\KL(\rho\|\pi)
+
\log\frac{2N_L\sqrt{m_i}}{\delta_T}
}{
m_i
}
\right).
\]
Therefore,
\begin{equation}
\E_{\rho}\E_{D_i}[r_{i,G_i}]
\leq
H_i
\kl_+^{-1}
\left(
\widehat q_{i,G_i}(\rho),
\frac{
\KL(\rho\|\pi)
+
\log\frac{2N_L\sqrt{m_i}}{\delta_T}
}{
m_i
}
\right)
=
\mathcal T_{i,G_i}
\left(
\rho,\frac{\delta_T}{N_L}
\right).
\label{eq:proof_component_tail}
\end{equation}
Multiplying by $m_i\overline C_i(\rho,\delta_C)$ gives
\begin{equation}
m_i\overline C_i(\rho,\delta_C)
\E_{\rho}\E_{D_i}[r_{i,G_i}]
\leq
m_i\overline C_i(\rho,\delta_C)
\mathcal T_{i,G_i}
\left(
\rho,\frac{\delta_T}{N_L}
\right).
\label{eq:proof_component_tail_weighted}
\end{equation}
Again, a union bound over the $N_L$ residual events yields an event $\mathcal E_T$ satisfying $\Prob(\mathcal E_T)\geq1-\delta_T$.

\paragraph{Step 2c: computable upper bound on the original population gradient energy.}
Combining~\eqref{eq:proof_component_gradient_decomposition} with~\eqref{eq:proof_component_gradient_empirical} and~\eqref{eq:proof_component_tail_weighted}, on $\mathcal E_C\cap\mathcal E_J\cap\mathcal E_T$ we have, simultaneously for every component $i$ and every $\rho\ll\pi$,
\begin{align}
m_i\overline C_i(\rho,\delta_C)
\E_{\rho}\E_{D_i}[g_i]
&=
m_i\overline C_i(\rho,\delta_C)
\E_{\rho}\E_{D_i}[g_{i,G_i}]
+
m_i\overline C_i(\rho,\delta_C)
\E_{\rho}\E_{D_i}[r_{i,G_i}]
\nonumber\\
&\leq
\widehat{\mathcal G}_{i}(\rho,\delta_C)
+
\Delta_{i,G_i}(\rho,\delta_C,\delta_J)
+
m_i\overline C_i(\rho,\delta_C)
\mathcal T_{i,G_i}
\left(
\rho,\frac{\delta_T}{N_L}
\right)
\nonumber\\
&=
\mathcal G_{i,G_i}(\rho,\delta_C,\delta_J,\delta_T).
\label{eq:proof_component_gradient}
\end{align}

Substituting~\eqref{eq:proof_component_gradient} into~\eqref{eq:proof_component_cgf} gives
\begin{equation}
\E_{\rho}
[\Lambda_{\boldsymbol{\theta}}^{(i)}(\lambda)]
\leq
\frac{\lambda^2}{2m_i}
\mathcal G_{i,G_i}(\rho,\delta_C,\delta_J,\delta_T).
\label{eq:proof_component_cgf_empirical}
\end{equation}

\paragraph{Step 3: PAC-Bayes--Chernoff control of the original loss.}
We now apply the single-loss PAC-Bayes--Chernoff inequality directly to the original, unclipped loss $\ell_i$. With confidence $\delta_G/N_L$, there exists an event $\mathcal E_{G,i}$ of probability at least $1-\delta_G/N_L$ such that, simultaneously for every $\rho\ll\pi$,
\begin{equation}
\mathcal R_i(\rho)
-
\widehat{\mathcal R}_i(\rho)
\leq
\inf_{\lambda>0}
\left\{
\frac{
\KL(\rho\|\pi)
+
\ln\frac{N_Lm_i}{\delta_G}
}{
\lambda(m_i-1)
}
+
\frac{
\E_{\rho}
[\Lambda_{\boldsymbol{\theta}}^{(i)}(\lambda)]
}{
\lambda
}
\right\}.
\label{eq:proof_component_pb}
\end{equation}
Taking a union bound over the $N_L$ losses gives an event $\mathcal E_G$ with $\Prob(\mathcal E_G)\geq1-\delta_G$.

On the intersection $\mathcal E_C\cap\mathcal E_J\cap\mathcal E_T\cap\mathcal E_G$, substituting~\eqref{eq:proof_component_cgf_empirical} into~\eqref{eq:proof_component_pb} yields
\[
\mathcal R_i(\rho)
-
\widehat{\mathcal R}_i(\rho)
\leq
\inf_{\lambda>0}
\left\{
\frac{
\KL(\rho\|\pi)
+
\ln\frac{N_Lm_i}{\delta_G}
}{
\lambda(m_i-1)
}
+
\frac{\lambda}{2m_i}
\mathcal G_{i,G_i}(\rho,\delta_C,\delta_J,\delta_T)
\right\}.
\]
Using $\inf_{\lambda>0}\{a/\lambda+b\lambda\}=2\sqrt{ab}$ gives
\begin{align}
\mathcal R_i(\rho)
-
\widehat{\mathcal R}_i(\rho)
&\leq
\sqrt{
\frac{
2
\mathcal G_{i,G_i}(\rho,\delta_C,\delta_J,\delta_T)
\left(
\KL(\rho\|\pi)
+
\ln\frac{N_Lm_i}{\delta_G}
\right)
}{
m_i(m_i-1)
}
}
\nonumber\\
&=
\frac{1}{m_i}
\sqrt{
2
\mathcal G_{i,G_i}(\rho,\delta_C,\delta_J,\delta_T)
K_i^{\mathrm{PB}}
\left(
\rho,\pi,\frac{\delta_G}{N_L}
\right)
},
\label{eq:proof_component_optimised}
\end{align}
where $K_i^{\mathrm{PB}}(\rho,\pi,\delta_i):=\frac{m_i\left(\KL(\rho\|\pi)+\ln\frac{m_i}{\delta_i}\right)}{m_i-1}$.

\paragraph{Step 4: confidence allocation.}
The events $\mathcal E_C,\mathcal E_J,\mathcal E_T,\mathcal E_G$ have respective failure probabilities at most $\delta_C,\delta_J,\delta_T,\delta_G$. Hence, by a union bound, their intersection has probability at least $1-(\delta_C+\delta_J+\delta_T+\delta_G)$. Choosing $\delta_C=\delta_G=\delta/3$ and $\delta_J=\delta_T=\delta/6$ gives total failure probability $\delta$. Therefore, with probability at least $1-\delta$, simultaneously for every component $i$ and every $\rho\ll\pi$,
\[
\mathcal R_i(\rho)
\leq
\widehat{\mathcal R}_i(\rho)
+
\frac{1}{m_i}
\sqrt{
2\,
\mathcal G_{i,G_i}
\left(
\rho,\frac{\delta}{3},\frac{\delta}{6},\frac{\delta}{6}
\right)
K_i^{\mathrm{PB}}
\left(
\rho,\pi,\frac{\delta}{3N_L}
\right)
},
\]
which proves the theorem.
\end{proof}

\subsection{Proof of Theorem~\ref{thm:main}}
\label{proof:main}

\FinalBound*
\begin{proof}
The proof follows the same calibrated input-gradient control established in Theorem~\ref{thm:component_bound}, but differs in the final PAC-Bayes step: instead of applying the Chernoff argument independently to every loss component, we first aggregate their CGFs and obtain a single sample-weighted PAC-Bayesian complexity. We use the confidence allocation $\delta_C=\delta_G=\delta/3$ and $\delta_J=\delta_T=\delta/6$.

\paragraph{Step 1: simultaneous calibration of the loss components.}
As in the proof of Theorem~\ref{thm:component_bound}, applying Corollary~\ref{cor:cgf_grad} gives an event $\mathcal E_C$ satisfying $\Prob(\mathcal E_C)\geq1-\delta_C$ such that, simultaneously for every $i$ and every $\rho\ll\pi$,
\begin{equation}
\E_{\boldsymbol{\theta}\sim\rho}
[\Lambda_{\boldsymbol{\theta}}^{(i)}(\lambda)]
\leq
\frac{
\overline C_i(\rho,\delta_C)\lambda^2
}{2}
\E_{\boldsymbol{\theta}\sim\rho}
\E_{\vx\sim D_i}
[g_i(\vx,\boldsymbol{\theta})].
\label{eq:calibrated_cgf_task_specific}
\end{equation}

Define the population sample-weighted input-gradient complexity
\begin{equation}
\mathcal G^{\mathcal S}(\rho,\delta_C)
:=
\sum_{i=1}^{N_L}
m_i
\overline C_i(\rho,\delta_C)
\E_{\boldsymbol{\theta}\sim\rho}
\E_{\vx\sim D_i}
[g_i(\vx,\boldsymbol{\theta})].
\label{eq:population_gradient_complexity}
\end{equation}
Since $\Lambda_{\boldsymbol{\theta}}^{\mathcal S}(\lambda)=\frac{1}{M}\sum_{i=1}^{N_L}m_i\Lambda_{\boldsymbol{\theta}}^{(i)}(\lambda)$,~\eqref{eq:calibrated_cgf_task_specific} gives, on $\mathcal E_C$,
\begin{align}
\E_{\boldsymbol{\theta}\sim\rho}
[\Lambda_{\boldsymbol{\theta}}^{\mathcal S}(\lambda)]
&=
\frac{1}{M}
\sum_{i=1}^{N_L}
m_i
\E_{\boldsymbol{\theta}\sim\rho}
[\Lambda_{\boldsymbol{\theta}}^{(i)}(\lambda)]
\nonumber\\
&\leq
\frac{\lambda^2}{2M}
\sum_{i=1}^{N_L}
m_i
\overline C_i(\rho,\delta_C)
\E_{\boldsymbol{\theta}\sim\rho}
\E_{\vx\sim D_i}
[g_i(\vx,\boldsymbol{\theta})]
\nonumber\\
&=
\frac{\lambda^2}{2M}
\mathcal G^{\mathcal S}(\rho,\delta_C).
\label{eq:sample_cgf_gradient_bound}
\end{align}

\paragraph{Step 2: joint PAC-Bayes--Chernoff control of the original PIML losses.}
Apply Theorem~\ref{thm:generalpbc_mtl} to the original losses $\{\ell_i\}_{i=1}^{N_L}$ with failure probability $\delta_G$. There exists an event $\mathcal E_G$ satisfying $\Prob(\mathcal E_G)\geq1-\delta_G$ such that, simultaneously for every $\rho\ll\pi$,
\begin{equation}
\mathrm{gen}^{\mathcal S}(\rho,S)
\leq
\inf_{\lambda>0}
\left\{
\frac{
\KL(\rho\|\pi)
+
\log\frac{M}{\delta_G}
}{
\lambda(M-1)
}
+
\frac{
\E_{\boldsymbol{\theta}\sim\rho}
[\Lambda_{\boldsymbol{\theta}}^{\mathcal S}(\lambda)]
}{
\lambda
}
\right\}.
\label{eq:general_pb_main}
\end{equation}
Combining~\eqref{eq:sample_cgf_gradient_bound} and~\eqref{eq:general_pb_main}, on $\mathcal E_C\cap\mathcal E_G$,
\[
\mathrm{gen}^{\mathcal S}(\rho,S)
\leq
\inf_{\lambda>0}
\left\{
\frac{
\KL(\rho\|\pi)
+
\log\frac{M}{\delta_G}
}{
\lambda(M-1)
}
+
\frac{\lambda}{2M}
\mathcal G^{\mathcal S}(\rho,\delta_C)
\right\}.
\]
Optimising over $\lambda>0$ gives
\begin{equation}
\mathcal R_{\mathrm{PIML}}^{\mathcal S}(\rho)
-
\widehat{\mathcal R}_{\mathrm{PIML}}^{\mathcal S}(\rho)
\leq
\frac{1}{M}
\sqrt{
2
\mathcal G^{\mathcal S}(\rho,\delta_C)
K(\rho,\pi,\delta_G)
},
\label{eq:population_gradient_bound}
\end{equation}
where $K(\rho,\pi,\delta_G):=\frac{M\left(\KL(\rho\|\pi)+\log\frac{M}{\delta_G}\right)}{M-1}$. It therefore remains only to obtain a computable upper bound on $\mathcal G^{\mathcal S}(\rho,\delta_C)$.

\paragraph{Step 3: decomposition of the population input-gradient complexity.}
For every component $i$, $g_i=g_{i,G_i}+r_{i,G_i}$, hence
\begin{align}
\mathcal G^{\mathcal S}(\rho,\delta_C)
&=
\sum_{i=1}^{N_L}
m_i
\overline C_i(\rho,\delta_C)
\E_{\rho}\E_{D_i}[g_{i,G_i}]
\nonumber\\
&\quad+
\sum_{i=1}^{N_L}
m_i
\overline C_i(\rho,\delta_C)
\E_{\rho}\E_{D_i}[r_{i,G_i}].
\label{eq:main_gradient_decomposition}
\end{align}
We control the two sums on the right-hand side using the component-wise concentration arguments already established above.

\paragraph{Step 3a: empirical control of the bounded-gradient contribution.}
For every $i$, the argument leading to~\eqref{eq:proof_component_gradient_clipped} gives, with confidence $\delta_J/N_L$,
\begin{equation}
\E_{\rho}\E_{D_i}[g_{i,G_i}]
\leq
\frac{1}{m_i}
\E_{\rho}
\sum_{j=1}^{m_i}
g_{i,G_i}(\vx_{i,j},\boldsymbol{\theta})
+
G_i
\sqrt{
\frac{
\KL(\rho\|\pi)
+
\log\frac{N_Lm_i}{\delta_J}
}{
2(m_i-1)
}
}.
\label{eq:main_gradient_clipped_control}
\end{equation}
Taking a union bound over all components gives an event $\mathcal E_J$ with probability at least $1-\delta_J$ on which~\eqref{eq:main_gradient_clipped_control} holds simultaneously for every $i$ and every $\rho\ll\pi$.

As before, $g_{i,G_i}\leq g_i$ pointwise. Therefore,
\[
\E_{\rho}
\sum_{j=1}^{m_i}
g_{i,G_i}(\vx_{i,j},\boldsymbol{\theta})
\leq
\E_{\rho}
\sum_{j=1}^{m_i}
g_i(\vx_{i,j},\boldsymbol{\theta}).
\]
Multiplying~\eqref{eq:main_gradient_clipped_control} by $m_i\overline C_i(\rho,\delta_C)$ and summing over $i$ gives
\begin{align}
\sum_{i=1}^{N_L}
m_i
\overline C_i(\rho,\delta_C)
\E_{\rho}\E_{D_i}[g_{i,G_i}]
&\leq
\sum_{i=1}^{N_L}
\overline C_i(\rho,\delta_C)
\E_{\rho}
\sum_{j=1}^{m_i}
g_i(\vx_{i,j},\boldsymbol{\theta})
\nonumber\\
&\quad+
\sum_{i=1}^{N_L}
m_i
\overline C_i(\rho,\delta_C)
G_i
\sqrt{
\frac{
\KL(\rho\|\pi)
+
\log\frac{N_Lm_i}{\delta_J}
}{
2(m_i-1)
}
}
\nonumber\\
&=
\widehat{\mathcal G}^{\mathcal S}(\rho,\delta_C)
+
\Delta_{\mathbf G}^{\mathcal S}(\rho,\delta_C,\delta_J).
\label{eq:gradient_empirical_control}
\end{align}
Hence the empirical part of the final joint complexity is expressed directly through the original input-gradient energies, although bounded concentration is still obtained through $g_{i,G_i}$.

\paragraph{Step 3b: control of the residual-gradient contribution.}
For each $i$, define $Z_{i,G_i}:=r_{i,G_i}/H_i\in[0,1]$. From the binary-$\kl$ argument in the proof of Theorem~\ref{thm:component_bound}, with confidence $\delta_T/N_L$,
\[
\E_{\rho}\E_{D_i}[r_{i,G_i}]
\leq
\mathcal T_{i,G_i}
\left(
\rho,\frac{\delta_T}{N_L}
\right).
\]
A union bound over the $N_L$ components therefore gives an event $\mathcal E_T$ satisfying $\Prob(\mathcal E_T)\geq1-\delta_T$ such that, simultaneously for every $i$ and every $\rho\ll\pi$,
\begin{equation}
\E_{\rho}\E_{D_i}[r_{i,G_i}]
\leq
\mathcal T_{i,G_i}
\left(
\rho,\frac{\delta_T}{N_L}
\right).
\label{eq:gradient_tail_task_bound}
\end{equation}
Multiplying by $m_i\overline C_i(\rho,\delta_C)$ and summing gives
\begin{equation}
\sum_{i=1}^{N_L}
m_i
\overline C_i(\rho,\delta_C)
\E_{\rho}\E_{D_i}[r_{i,G_i}]
\leq
\mathcal T_{\mathbf G}^{\mathcal S}
(\rho,\delta_C,\delta_T).
\label{eq:gradient_tail_total}
\end{equation}

\paragraph{Step 3c: computable joint gradient complexity.}
Combining~\eqref{eq:main_gradient_decomposition}, \eqref{eq:gradient_empirical_control}, and~\eqref{eq:gradient_tail_total}, on $\mathcal E_C\cap\mathcal E_J\cap\mathcal E_T$,
\begin{align}
\mathcal G^{\mathcal S}(\rho,\delta_C)
&\leq
\widehat{\mathcal G}^{\mathcal S}(\rho,\delta_C)
+
\Delta_{\mathbf G}^{\mathcal S}(\rho,\delta_C,\delta_J)
+
\mathcal T_{\mathbf G}^{\mathcal S}(\rho,\delta_C,\delta_T)
\nonumber\\
&=
\mathcal G_{\mathbf G}^{\mathcal S}(\rho,\delta_C,\delta_J,\delta_T).
\label{eq:computable_population_gradient}
\end{align}

\paragraph{Step 4: combination of the high-probability events.}
The four events $\mathcal E_C,\mathcal E_G,\mathcal E_J,\mathcal E_T$ have failure probabilities at most $\delta_C,\delta_G,\delta_J,\delta_T$, respectively. Since $\delta_C=\delta_G=\delta/3$ and $\delta_J=\delta_T=\delta/6$, their total failure probability is at most
$
\delta_C+\delta_G+\delta_J+\delta_T
=
\frac{\delta}{3}
+
\frac{\delta}{3}
+
\frac{\delta}{6}
+
\frac{\delta}{6}
=
\delta.
$
Hence, by a union bound, all preceding inequalities hold simultaneously with probability at least $1-\delta$.

Substituting~\eqref{eq:computable_population_gradient} into~\eqref{eq:population_gradient_bound} gives
\[
\mathcal R_{\mathrm{PIML}}^{\mathcal S}(\rho)
\leq
\widehat{\mathcal R}_{\mathrm{PIML}}^{\mathcal S}(\rho)
+
\frac{1}{M}
\sqrt{
2\,
\mathcal G_{\mathbf G}^{\mathcal S}
(\rho,\delta_C,\delta_J,\delta_T)
K(\rho,\pi,\delta_G)
}.
\]
Finally, inserting $\delta_C=\delta_G=\delta/3$ and $\delta_J=\delta_T=\delta/6$ yields
\[
\mathcal R_{\mathrm{PIML}}^{\mathcal S}(\rho)
\leq
\widehat{\mathcal R}_{\mathrm{PIML}}^{\mathcal S}(\rho)
+
\frac{1}{M}
\sqrt{
2\,
\mathcal G_{\mathbf G}^{\mathcal S}
\left(
\rho,\frac{\delta}{3},\frac{\delta}{6},\frac{\delta}{6}
\right)
K
\left(
\rho,\pi,\frac{\delta}{3}
\right)
},
\]
which is~\eqref{eq:main} and completes the proof.
\end{proof}

\clearpage

\section{Details of the Bound-Aware Algorithm}
\label{apdx:algo}

This section provides further details of the bound-aware learning procedure introduced in Section~\ref{sec:bound_aware_algorithm} and summarised in Algorithm~\ref{algo:self_bounding_aware}. The procedure consists of three stages: learning and localising a physics-informed prior, calibrating the quantities entering the bound and optimising the posterior, and certifying the learned posterior. We particularly detail the practical selection of the localisation radius, the calibration thresholds $(K_i,G_i)$, and the gradient envelope $U_i$.

\paragraph{Data splitting and prior learning.}
We use three mutually independent datasets $S^{\mathrm{prior}}$, $S^{\mathrm{cal}}$, and $S^{\mathrm{post}}$ for prior learning and prior-side diagnostics, final calibration, and posterior learning, respectively. Starting from $\boldsymbol{\theta}_0$, \textbf{the prior mean $\boldsymbol{\theta}_{\pi}$ is first trained on $S^{\mathrm{prior}}$ using physics-only information} (\ie by fixing $\lambda_d = 0$ for $\widehat{\mathcal{R}}_{\textrm{PIML},\lambda}(\boldsymbol{\theta})$), after which the prior $\pi$ is fixed. Importantly, $S^{\mathrm{cal}}$ remains untouched throughout prior construction and posterior learning and is accessed only after the learned posterior has been fixed. This disjoint construction preserves the independence required by the calibration theorem, while the dependence of the posterior on $S^{\mathrm{post}}$ is handled by the PAC-Bayes complexity.

We localise the hypothesis space around $\boldsymbol{\theta}_{\pi}$ using
\begin{equation}
R_{\mathrm{loc}}
=
r_{\mathrm{loc}}\sigma\sqrt{d_\theta},
\label{eq:appendix_localisation_radius}
\end{equation}
where $\sigma$ denotes the Gaussian standard deviation and $d_\theta$ the number of trainable parameters. We use $r_{\mathrm{loc}}=3$ by default, corresponding to a $3\sigma$ radial scale after accounting for the typical $\sqrt{d_\theta}$ dependence of the Euclidean norm of an isotropic Gaussian perturbation.

\paragraph{Reference loss and gradient scales.}
Since $\boldsymbol{\theta}_{\pi}$ is already well trained, evaluating the loss and its input-gradient energy only at the prior mean may underestimate the values encountered over the localised hypothesis region. We therefore sample models on the boundary $\|\boldsymbol{\theta}-\boldsymbol{\theta}_{\pi}\|_2 = R_{\mathrm{loc}}$ and evaluate each loss component over $S^{\mathrm{prior}}$, which contains a disjoint subset of $100$ observational samples for evaluating the data-fidelity loss. This yields component-wise reference scales $K_{\mathrm{ref},i}$ and $G_{\mathrm{ref},i}$.

The calibration thresholds are then parameterised as
\begin{equation}
K_i
=
K_{\mathrm{factor}}K_{\mathrm{ref},i},
\qquad
G_i
=
G_{\mathrm{factor}}G_{\mathrm{ref},i}.
\label{eq:appendix_KG_scaling}
\end{equation}
We fix $G_{\mathrm{factor}}=1$ and vary only $K_{\mathrm{factor}}$ when assessing the calibration-dominance condition in Assumption~\ref{ass:calibration_dominance}. This choice is motivated by the role of $G_i$ in the denominator of the bounded calibration coefficient: fixing it at the observed reference gradient scale provides a positive and numerically stable denominator, while adjusting $K_i$ is sufficient to empirically assess the required dominance relation. Before posterior learning, this procedure is performed under the fixed prior $\pi$ on $S^{\mathrm{prior}}$ to select a training value $K_{\mathrm{factor}}^{\mathrm{train}}$. After posterior learning, it is repeated on the same $S^{\mathrm{prior}}$ under the fixed learned posterior $\rho$ to select the posterior-specific certification value $K_{\mathrm{factor}}^{\mathrm{cert}}(\rho)$. The reference scales $K_{\mathrm{ref},i}$ and $G_{\mathrm{ref},i}$ remain fixed throughout.

Importantly, both $K_i$ and $G_i$ are used only in the calibration construction of Theorem~\ref{thm:variance_gradient_calibration}. In particular, $K_i$ clips the loss entering the empirical calibration variance, while $G_i$ clips the corresponding calibration input-gradient energy. Neither threshold clips the original loss or the empirical input-gradient complexity used for posterior learning and final certification.

\paragraph{Gradient envelope.}
Assumption~\ref{ass:gradient_envelope} additionally requires a population envelope $U_i$ for the original input-gradient energy. We parameterise it relative to the same reference gradient scale as
\begin{equation}
U_i
=
U_{G,\mathrm{factor}}G_{\mathrm{ref},i},
\label{eq:appendix_UG_scaling}
\end{equation}
with $U_{G,\mathrm{factor}}\gg1$. In our experiments, $U_{G,\mathrm{factor}}$ is chosen conservatively enough, typically at least $10^4$, and the resulting envelopes are stress-tested using inputs and models sampled throughout the localised parameter region.

This way, the roles of $G_i$ and $U_i$ should be distinguished. The former is a comparatively small clipping threshold used only to construct the bounded calibration surrogate, whereas the latter is a conservative envelope required by the population input-gradient assumption and the finite-sample concentration argument. In particular, the empirical contribution appearing in the proposed bound is evaluated using the original input-gradient energy $g_i$, rather than its clipped counterpart $g_{i,G_i}$. Hence, the final complexity retains the observed gradient magnitude directly, while $U_i$ enters only through the corresponding concentration correction.

Overall, the three scales used in the analysis are
\begin{equation}
K_i
=
K_{\mathrm{factor}}K_{\mathrm{ref},i},
\qquad
G_i
=
G_{\mathrm{ref},i},
\qquad
U_i
=
U_{G,\mathrm{factor}}G_{\mathrm{ref},i},
\label{eq:appendix_calibration_scales}
\end{equation}
where $K_i$ and $G_i$ determine the bounded calibration coefficient and $U_i$ controls the concentration of the original input-gradient energy. The localisation factor $r_{\mathrm{loc}}$ determines the parameter-space region over which the corresponding reference scales are assessed. Here $K_{\mathrm{factor}}=K_{\mathrm{factor}}^{\mathrm{train}}$ during posterior learning and $K_{\mathrm{factor}}=K_{\mathrm{factor}}^{\mathrm{cert}}(\rho)$ during final posterior-specific calibration.

\paragraph{Calibration of the regularity quantities.}
Once $\pi$, $(K_{\mathrm{ref},i},G_{\mathrm{ref},i},G_i,U_i)_{i=1}^{N_L}$, and the localisation region are fixed, the CGF-to-variance coefficients $\{\kappa_{p,i}\}_{i=1}^{N_L}$, the input-gradient envelope, and the calibration-dominance ratios are empirically assessed on $S^{\mathrm{prior}}$. The coefficients $\{\kappa_{p,i}\}_{i=1}^{N_L}$ are fixed before posterior learning and are subsequently kept unchanged for final certification. As discussed in Section~\ref{sec:bound_aware_algorithm}, these diagnostics provide empirical evidence for the corresponding population assumptions, while the final theoretical guarantee remains conditional on their validity.

For posterior learning, $S^{\mathrm{post}}$ is used to estimate surrogate coefficients $\{\overline C_i^{\mathrm{train}}\}_{i=1}^{N_L}$ using $K_{\mathrm{factor}}^{\mathrm{train}}$. These coefficients guide optimisation only and do not enter the final certificate. After learning, $\rho$ is fixed, $S^{\mathrm{prior}}$ is reused to select $K_{\mathrm{factor}}^{\mathrm{cert}}(\rho)$, and only then is the untouched $S^{\mathrm{cal}}$ used to obtain the final posterior-specific coefficients $\{\overline C_i^{\mathrm{cert}}(\rho)\}_{i=1}^{N_L}$ through Theorem~\ref{thm:variance_gradient_calibration}.

\paragraph{Bound-aware posterior learning.}
The posterior is initialised at the learned prior mean, $\boldsymbol{\theta}_{\rho}\gets\boldsymbol{\theta}_{\pi}$, and subsequently optimised on $S^{\mathrm{post}}$ using the set of surrogate coefficients $\{\overline C_i^{\mathrm{train}}\}_{i=1}^{N_L}$ estimated from the same posterior-training split under the fixed prior $\pi$. At each iteration, mini-batches $B=\{B_i\}_{i=1}^{N_L}$ are sampled and a single model $\boldsymbol{\theta}\sim\rho$ is drawn through reparameterisation. The same realisation is used to evaluate the empirical PIML risk $\widehat R_{\mathrm{PIML}}^B(\boldsymbol{\theta})$ and the empirical input-gradient complexity. In accordance with Theorem~\ref{thm:main}, the latter is constructed from the original, unclipped input-gradient energies $g_i(\vx,\boldsymbol{\theta})$, together with their finite-sample corrections based on the prescribed envelopes $U_i$.

In contrast, the PAC-Bayes complexity is a distribution-level quantity and is evaluated directly from the current posterior parameters through $\KL(\rho\|\pi)$. The resulting stochastic objective therefore combines a single-sample Monte Carlo realisation of the empirical risk and input-gradient complexity with the exact KL complexity of the current posterior. This construction directly reflects the final certificate: posterior learning favors low empirical PIML risk, small observed input-gradient energy, and limited displacement from the physics-informed prior. Importantly, neither the original PIML losses nor their input-gradient energies are clipped during posterior optimisation. The clipping thresholds $(K_i,G_i)$ are confined to the calibration procedure used to estimate the population variance-gradient coefficients.

\paragraph{Posterior calibration and certification.}
After optimisation, the learned posterior $\rho$ is fixed. We first reuse $S^{\mathrm{prior}}$ to select the posterior-specific $K_{\mathrm{factor}}^{\mathrm{cert}}(\rho)$ and empirically assess the calibration-dominance condition under $\rho$, while keeping the previously fixed reference scales $(K_{\mathrm{ref},i},G_{\mathrm{ref},i})$, gradient thresholds $G_i$, envelopes $U_i$, and coefficients $\kappa_{p,i}$ unchanged. The posterior-dependent coefficients $\{\overline C_i^{\mathrm{cert}}(\rho)\}_{i=1}^{N_L}$ are then evaluated on the untouched $S^{\mathrm{cal}}$ using Theorem~\ref{thm:variance_gradient_calibration}. The posterior-averaged empirical losses and original input-gradient energies are estimated on $S^{\mathrm{post}}$ using draws $\boldsymbol{\theta}\sim\rho$, and the final certificate is computed according to Theorem~\ref{thm:component_bound} and Corollary~\ref{cor:piml_union}, or Theorem~\ref{thm:main} and Corollary~\ref{cor:joint_balanced_unbalanced}. Thus, $S^{\mathrm{prior}}$ determines the physics-informed prior, reference scales, fixed prior-side quantities, and the training and certification values of $K_{\mathrm{factor}}$ through their respective diagnostics; $S^{\mathrm{post}}$ determines the surrogate calibration quantities and the learned posterior and supplies the empirical terms entering the PAC-Bayes bound; and the independent $S^{\mathrm{cal}}$ is reserved exclusively for the final posterior-specific calibration. The dependence of $\rho$ on $S^{\mathrm{post}}$ is handled by the PAC-Bayes KL complexity, while keeping $S^{\mathrm{cal}}$ untouched until the posterior is fixed preserves the independence required by the final calibration theorem.

\begin{algorithm}[t]
\caption{Bound-aware PAC-Bayes learning and certification}
\label{algo:self_bounding_aware}
\begin{algorithmic}[1]

\Require Prior data $S^{\mathrm{prior}}$; calibration data $S^{\mathrm{cal}}$; posterior data $S^{\mathrm{post}}$; initial parameters $\boldsymbol{\theta}_0$; standard deviation $\sigma$; confidence $\delta$; localisation factor $r_{\mathrm{loc}}$; gradient factor $U_{G,\mathrm{factor}}$; iterations $N_T,N_{T'}$.

\Statex \textbf{Step 1: prior learning and prior-side setup}

\State $\boldsymbol{\theta}_{\pi}\gets\boldsymbol{\theta}_0$
\For{$t=1$ to $N_T$}
    \State Draw mini-batches from $S^{\mathrm{prior}}$ and update $\boldsymbol{\theta}_{\pi}$ to minimise $\widehat{\mathcal{R}}_{\textrm{PIML},\lambda}(\boldsymbol{\theta})$ (\eqref{r_piml})
\EndFor
\State Fix $\pi$ and set $R_{\mathrm{loc}}=r_{\mathrm{loc}}\sigma\sqrt{d_\theta}$
\State Sample models on $\|\boldsymbol{\theta}-\boldsymbol{\theta}_{\pi}\|_2=R_{\mathrm{loc}}$ and estimate $(K_{\mathrm{ref},i},G_{\mathrm{ref},i})_{i=1}^{N_L}$ from $S^{\mathrm{prior}}$
\State Fix $G_{\mathrm{factor}}=1$, set $G_i=G_{\mathrm{ref},i}$ and $U_i=U_{G,\mathrm{factor}}G_{\mathrm{ref},i}$, and empirically assess the input-gradient envelopes
\State Set each $\kappa_{p,i}$ to a conservative upper envelope of the empirical CGF-to-variance ratios evaluated on $S^{\mathrm{prior}}$, and keep it fixed thereafter
\State Select $K_{\mathrm{factor}}^{\mathrm{train}}$ on $S^{\mathrm{prior}}$ under $\pi$ according to the empirical diagnostics for Assumption~\ref{ass:calibration_dominance}, and set $K_i^{\mathrm{train}}=K_{\mathrm{factor}}^{\mathrm{train}}K_{\mathrm{ref},i}$

\Statex \textbf{Step 2: bound-aware posterior learning}

\State Estimate the surrogate coefficients $\{\overline C_i^{\mathrm{train}}\}_{i=1}^{N_L}$ from $S^{\mathrm{post}}$ using $(K_i^{\mathrm{train}},G_i)_{i=1}^{N_L}$
\State $\boldsymbol{\theta}_{\rho}\gets\boldsymbol{\theta}_{\pi}$
\For{$t=1$ to $N_{T'}$}
    \State Draw mini-batches $B=\{B_i\}_{i=1}^{N_L}$ from $S^{\mathrm{post}}$
    \State Draw a reparameterised model $\boldsymbol{\theta}\sim\rho$
    \State Evaluate $\widehat R_{\mathrm{PIML}}^B(\boldsymbol{\theta})$ and the empirical original input-gradient complexity using the same $\boldsymbol{\theta}$
    \State Evaluate the corresponding finite-sample gradient correction using $(U_i)_{i=1}^{N_L}$
    \State Evaluate $K_{\mathrm{tr}}(\rho,\pi)$, including $\KL(\rho\|\pi)$, from the current distribution parameters
    \State Update $\boldsymbol{\theta}_{\rho}$ by minimising $\mathcal L_{\mathrm{BA}}^B(\boldsymbol{\theta},\rho)$ in~\eqref{eq:bound_aware_objective}
\EndFor

\Statex \textbf{Step 3: posterior-specific calibration and certification}

\State Fix the learned posterior $\rho$
\State Select $K_{\mathrm{factor}}^{\mathrm{cert}}(\rho)$ on $S^{\mathrm{prior}}$ under the fixed posterior $\rho$ according to the empirical diagnostics for Assumption~\ref{ass:calibration_dominance}, and set $K_i^{\mathrm{cert}}=K_{\mathrm{factor}}^{\mathrm{cert}}(\rho)K_{\mathrm{ref},i}$
\State Estimate $\{\overline C_i^{\mathrm{cert}}(\rho)\}_{i=1}^{N_L}$ from the untouched $S^{\mathrm{cal}}$ using $(K_i^{\mathrm{cert}},G_i)_{i=1}^{N_L}$
\State Estimate the posterior-averaged empirical risks and original input-gradient energies on $S^{\mathrm{post}}$ using draws $\boldsymbol{\theta}\sim\rho$
\State Evaluate the final certificate according to Theorem~\ref{thm:component_bound} and Corollary~\ref{cor:piml_union}, or Theorem~\ref{thm:main} and Corollary~\ref{cor:joint_balanced_unbalanced}
\State \Return learned posterior $\rho$ and its PAC-Bayes certificate

\end{algorithmic}
\end{algorithm}

\section{Details of Experimental Setting and Results} \label{apdx:details}
\subsection{Benchmark} \label{apdx:problem_formulation}

\paragraph{1D-Wave.} This problem considers a one-dimensional hyperbolic PDE that is commonly used as a benchmark for wave propagation phenomena. The governing equation is \footnote{For clarity, \(\pi\) denotes the numeric constant \(\pi\approx 3.14\) in this section, to avoid confusion with the prior distribution in the bounds.}
\begin{equation}
\begin{aligned}
\frac{\partial^2 u}{\partial t^2} - 4 \frac{\partial^2 u}{\partial x^2} &= 0, 
&& x \in (0,1),\; t \in (0,1), \\
u(0, x) &= \sin(\pi x) + \frac{1}{2}\sin(\beta \pi x), 
&& x \in [0,1], \\
\frac{\partial u(0, x)}{\partial t} &= 0, 
&& x \in [0,1], \\
u(t, 0) &= u(t, 1) = 0, 
&& t \in [0,1].
\end{aligned}
\label{eq:wave}
\end{equation}

The corresponding analytic solution is
\begin{equation}
u(t, x) = \sin(\pi x)\cos(2\pi t) + \frac{1}{2}\sin(\beta \pi x)\cos(2\beta \pi t).
\end{equation}

In our experiments, we set $\beta = 4$. Despite its relatively simple form, this problem poses challenges for physics-informed neural networks (PINNs) due to the presence of multiple frequency components in the solution.  For simplicity, in the experimental parts, we denote the three physical loss terms as follows: 
\begin{equation}
\begin{aligned}
\ell_p(\boldsymbol{\theta}, \boldsymbol{\mathrm{x}}) &= \left(\frac{\partial^2 u_{\boldsymbol{\theta}}}{\partial t^2} - 4 \frac{\partial^2 u_{\boldsymbol{\theta}}}{\partial x^2}\right)^2, \\
\ell_{ic}(\boldsymbol{\theta}, \boldsymbol{\mathrm{x}}) &= \left(u_{\boldsymbol{\theta}}(0, x) - (\sin(\pi x) + \frac{1}{2}\sin(\beta \pi x)) \right)^2, \\
\ell_{bc}(\boldsymbol{\theta}, \boldsymbol{\mathrm{x}}) &= \frac{1}{2}\left[\left(u_{\boldsymbol{\theta}}(t, 0)\right)^2 + \left(u_{\boldsymbol{\theta}}(t, 1)\right)^2  \right].
\end{aligned}
\label{eq:wave2}
\end{equation}

\paragraph{Convection.} This problem considers a hyperbolic PDE that serves as a standard benchmark for transport-dominated dynamics, also known as advection. The governing equation is
\begin{equation}
\begin{aligned}
\frac{\partial u}{\partial t} + \beta \frac{\partial u}{\partial x} &= 0, 
&& x \in (0, 2\pi),\; t \in (0, 1), \\
u(0, x) &= \sin(x), 
&& x \in [0, 2\pi], \\
u(t, 0) &= u(t, 2\pi), 
&& t \in [0,1].
\end{aligned}
\label{eq:advection}
\end{equation}

The corresponding analytic solution is
\begin{equation}
u(t, x) = \sin(x - \beta t).
\end{equation}

In our experiments, we set $\beta = 50$. Despite the simple closed-form expression, this problem remains challenging for physics-informed neural networks (PINNs) due to the high-frequency and rapidly propagating solution features. We consider the following physical losses in the experiments: 
\begin{equation}
\begin{aligned}
\ell_p(\boldsymbol{\theta}, \boldsymbol{\mathrm{x}}) &= \left(\frac{\partial u_{\boldsymbol{\theta}}}{\partial t} + \beta \frac{\partial u_{\boldsymbol{\theta}}}{\partial x}\right)^2, \\
\ell_{ic}(\boldsymbol{\theta}, \boldsymbol{\mathrm{x}}) &= \left(u_{\boldsymbol{\theta}}(0, x) - \sin(x) \right)^2, \\
\ell_{bc}(\boldsymbol{\theta}, \boldsymbol{\mathrm{x}}) &= \left(u_{\boldsymbol{\theta}}(t, 0) - u_{\boldsymbol{\theta}}(t, 2\pi)\right)^2. \\
\end{aligned}
\label{eq:convection_losses}
\end{equation}

\paragraph{Burgers \citep[PINNacle,][]{PINNacle}.} The Burgers equation is a fundamental nonlinear PDE that serves as a simplified model for various physical phenomena, including fluid dynamics, gas dynamics, traffic flow, and acoustics. It combines nonlinear convection with diffusion, making it an important benchmark for studying the interaction between nonlinear wave propagation and dissipative effects. The one-dimensional viscous Burgers equation is commonly expressed as
\begin{equation}
\begin{aligned}
\frac{\partial u}{\partial t} + u \frac{\partial u}{\partial x} &= \nu \frac{\partial^2 u}{\partial x^2}, \\
u(0, x) &= -\sin(\pi x), \\
u(t, -1) &= u(t, 1) = 0, \\
\end{aligned}
\label{eq:burgers}
\end{equation}
where $\nu = \frac{0.01}{\pi}$ is the kinematic viscosity coefficient. We consider the following physical losses in the experiments: 
\begin{equation}
\begin{aligned}
\ell_p(\boldsymbol{\theta}, \boldsymbol{\mathrm{x}}) &= \left(\frac{\partial u_{\boldsymbol{\theta}}}{\partial t} + u_{\boldsymbol{\theta}} \frac{\partial u_{\boldsymbol{\theta}}}{\partial x} - \nu \frac{\partial^2 u_{\boldsymbol{\theta}}}{\partial x^2}\right)^2, \\
\ell_{ic}(\boldsymbol{\theta}, \boldsymbol{\mathrm{x}}) &= \left(u_{\boldsymbol{\theta}}(0, x) + \sin (\pi x)\right)^2, \\
\ell_{bc}(\boldsymbol{\theta}, \boldsymbol{\mathrm{x}}) &= \frac{1}{2}\left[\left(u_{\boldsymbol{\theta}}(t, -1)\right)^2 + \left(u_{\boldsymbol{\theta}}(t, 1)\right)^2 \right]. \\
\end{aligned}
\label{eq:burgers_losses}
\end{equation}

Figure~\ref{fig:1D_solution} shows the solution map of the Burgers' problem, with a remarkable discontinuity around $x = 0$. This highly contrast physical structure make it hard to precisely predict the outcome in this region, thus causing very large gradients that inflate the bounds significantly. 

\begin{figure}[t]
    \centering
    \begin{subfigure}{0.28\linewidth}
        \includegraphics[width=\linewidth]{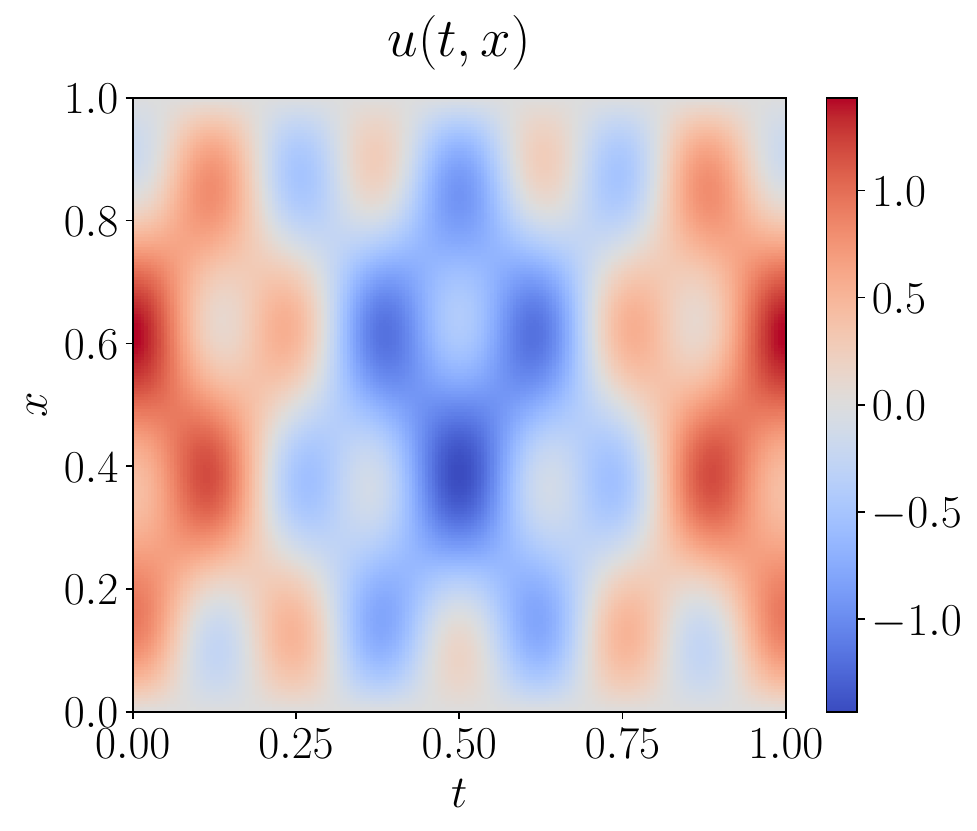}
        \caption{1D-Wave}
    \end{subfigure}
    \begin{subfigure}{0.28\linewidth}
        \includegraphics[width=\linewidth]{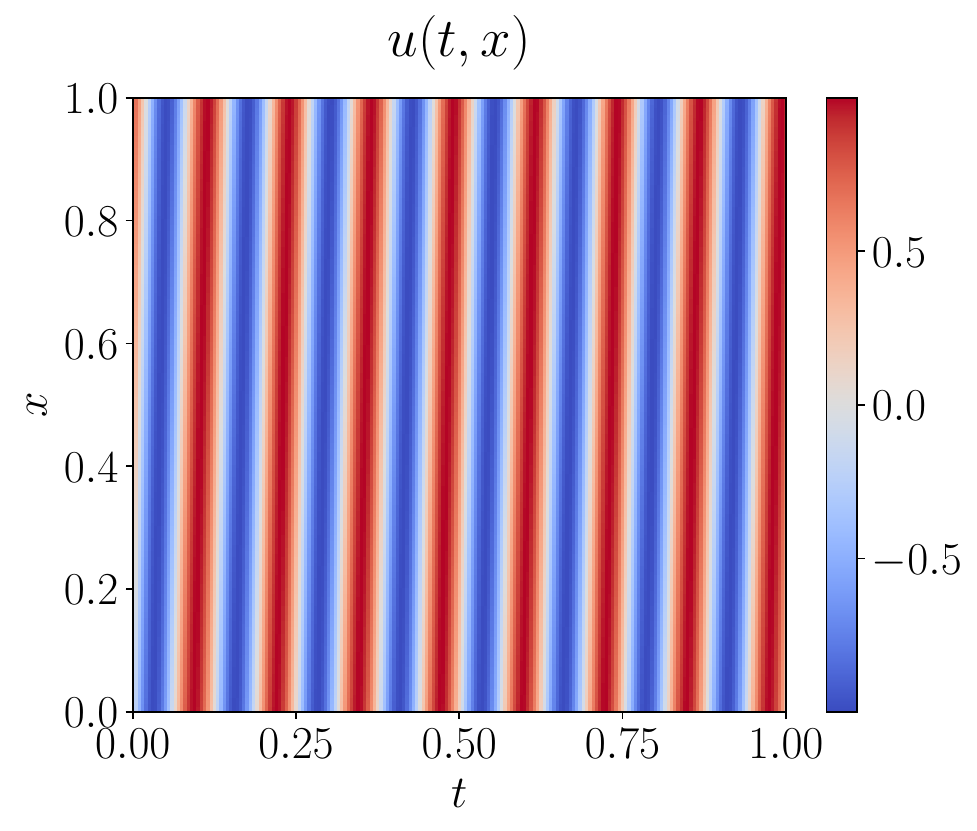}
        \caption{Convection}
    \end{subfigure}
    \begin{subfigure}{0.28\linewidth}
        \includegraphics[width=\linewidth]{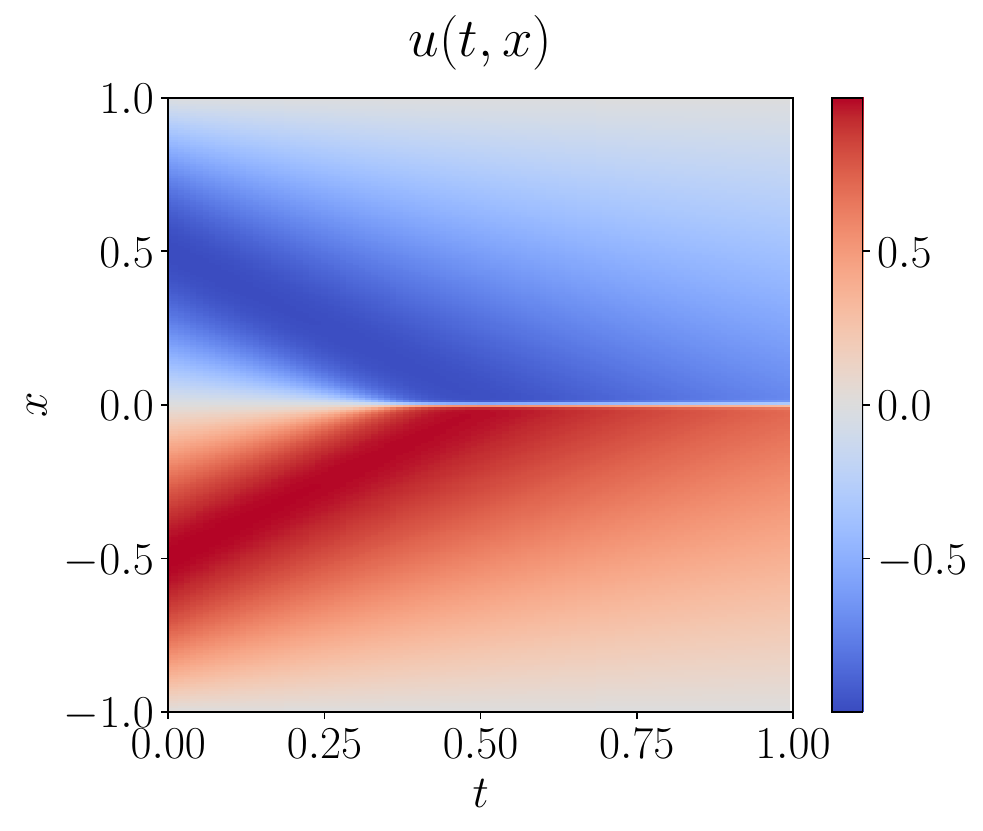}
        \caption{Burgers}
    \end{subfigure}
    \caption{Visualisation of the solution $u(t, x)$ for 1D-Wave, Convection and Burgers.}
    \label{fig:1D_solution}
\end{figure}

\paragraph{Heat-2D Multi-Scale (MS) (based on PINNacle \citep{PINNacle}).} The two-dimensional heat equation is a fundamental parabolic PDE that describes the diffusion of heat over time in a spatial domain. The Heat-2D MS problem considers anisotropic diffusion with substantially different diffusion coefficients along the two spatial dimensions. Moreover, the high-frequency oscillations along the $x$-direction introduce a multiscale spatial structure, making this problem a challenging benchmark for physics-informed learning. The Heat-2D MS equation is expressed as
\begin{equation}
\begin{aligned}
\frac{\partial u}{\partial t} &= a_x \frac{\partial^2 u}{\partial x^2} + a_y \frac{\partial^2 u}{\partial y^2}, \\
u(0, x, y) &= \sin(20\pi x)\sin(\pi y), \\
u(t, 0, y) &= u(t, 1, y) = 0, \\
u(t, x, 0) &= u(t, x, 1) = 0, 
\end{aligned}
\label{eq:heat2d_ms}
\end{equation}
where $a_x = \frac{1}{(500\pi)^2}$ and $a_y = \frac{1}{\pi^2}$ are the diffusion coefficients. The analytical solution is given by $u(t,x,y) = \sin(20\pi x)\sin(\pi y)\exp\left[-\left(a_x(20\pi)^2 + a_y\pi^2\right)t\right]$, where $(x,y)\in[0,1]^2$ and $t\in[0,5]$ (see Figure~\ref{fig:heat2d_ms_solution}). We consider the following physical losses in the experiments:
\begin{equation}
\begin{aligned}
\ell_p(\boldsymbol{\theta}, \boldsymbol{\mathrm{x}}) &= \left(\frac{\partial u_{\boldsymbol{\theta}}}{\partial t} - a_x \frac{\partial^2 u_{\boldsymbol{\theta}}}{\partial x^2} - a_y \frac{\partial^2 u_{\boldsymbol{\theta}}}{\partial y^2}\right)^2, \\
\ell_{ic}(\boldsymbol{\theta}, \boldsymbol{\mathrm{x}}) &= \left(u_{\boldsymbol{\theta}}(0,x,y) - \sin(20\pi x)\sin(\pi y)\right)^2, \\
\ell_{bc}(\boldsymbol{\theta}, \boldsymbol{\mathrm{x}}) &= \frac{1}{4}\left[\left(u_{\boldsymbol{\theta}}(t,0,y)\right)^2 + \left(u_{\boldsymbol{\theta}}(t,1,y)\right)^2 + \left(u_{\boldsymbol{\theta}}(t,x,0)\right)^2 + \left(u_{\boldsymbol{\theta}}(t,x,1)\right)^2 \right]. \\
\end{aligned}
\label{eq:heat2d_ms_losses}
\end{equation}

\begin{figure}[t]
    \centering
    \begin{subfigure}{0.28\linewidth}
        \includegraphics[width=\linewidth]{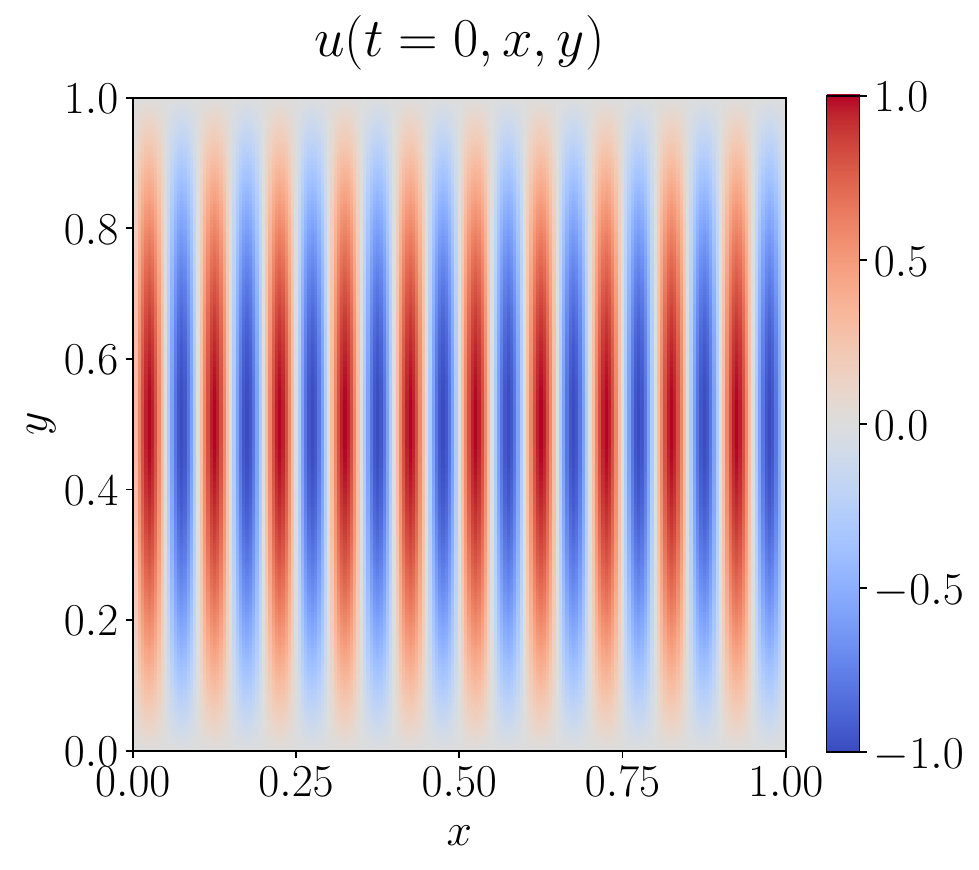}
    \end{subfigure}
    \begin{subfigure}{0.28\linewidth}
        \includegraphics[width=\linewidth]{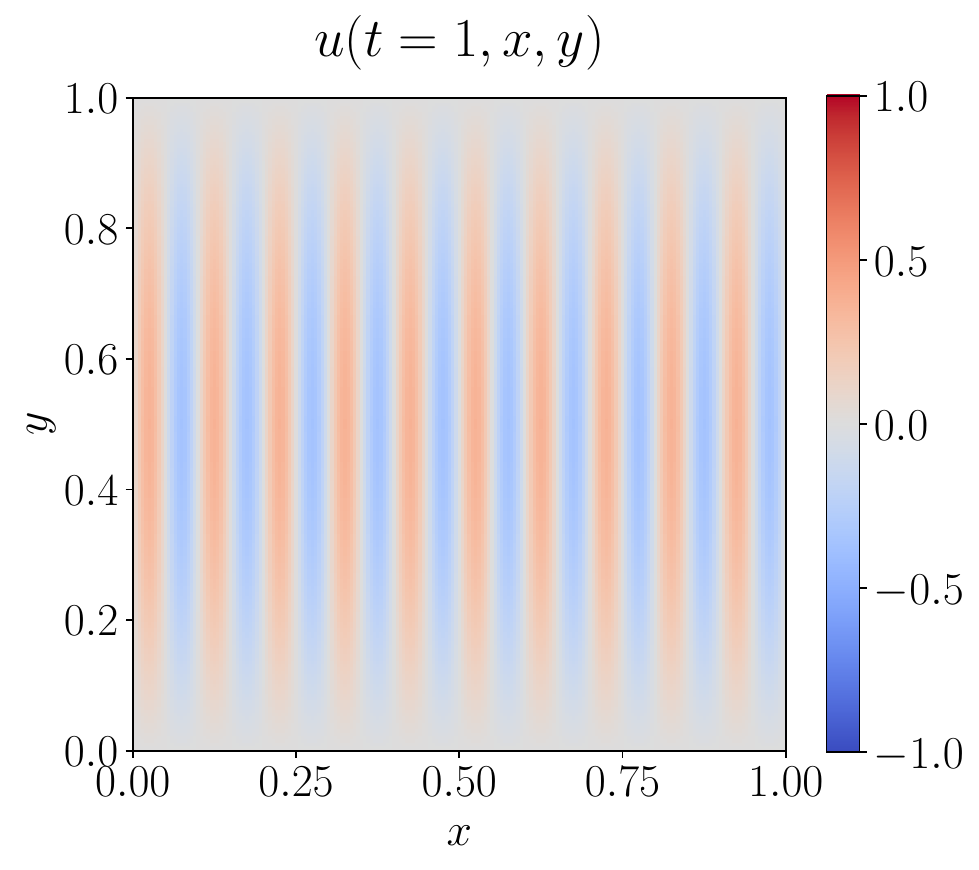}
    \end{subfigure}
    \begin{subfigure}{0.28\linewidth}
        \includegraphics[width=\linewidth]{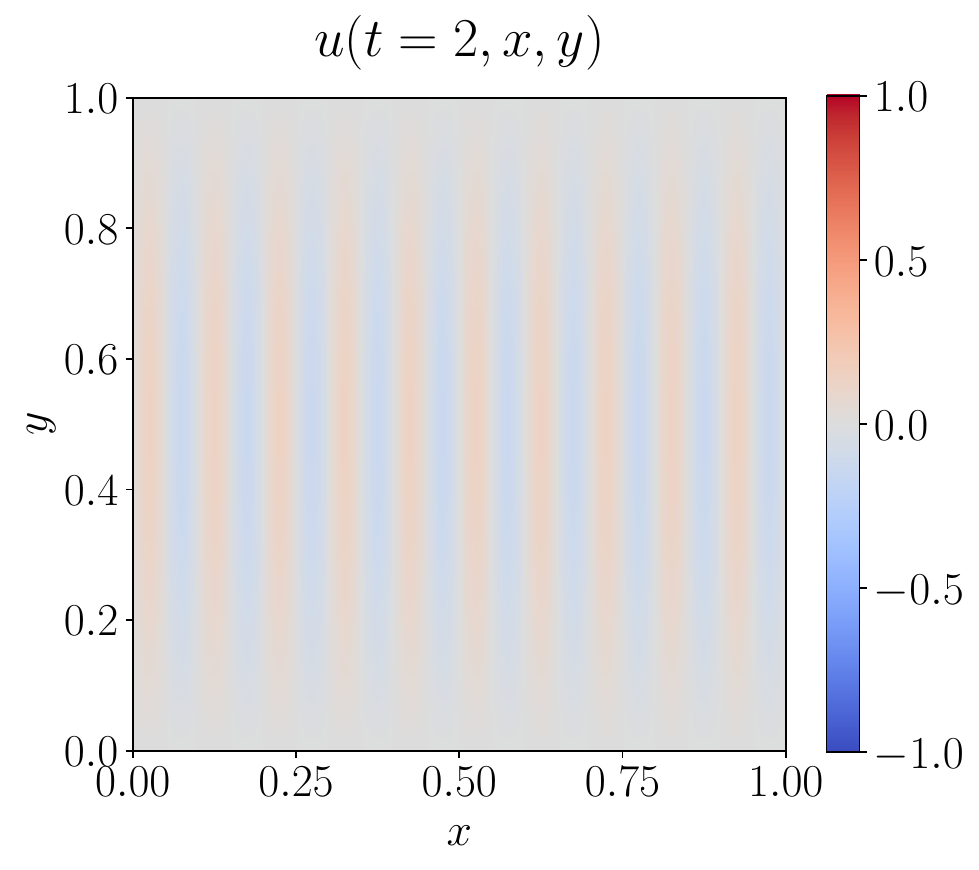}
    \end{subfigure}
    \begin{subfigure}{0.28\linewidth}
        \includegraphics[width=\linewidth]{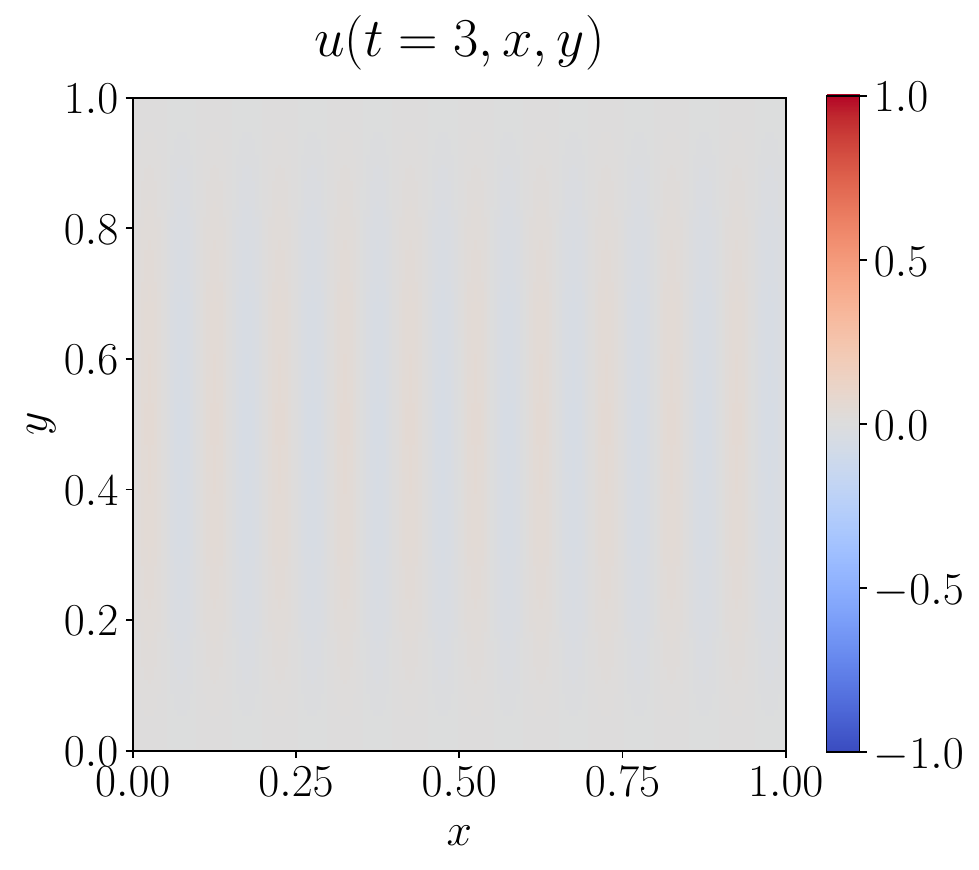}
    \end{subfigure}
    \begin{subfigure}{0.28\linewidth}
        \includegraphics[width=\linewidth]{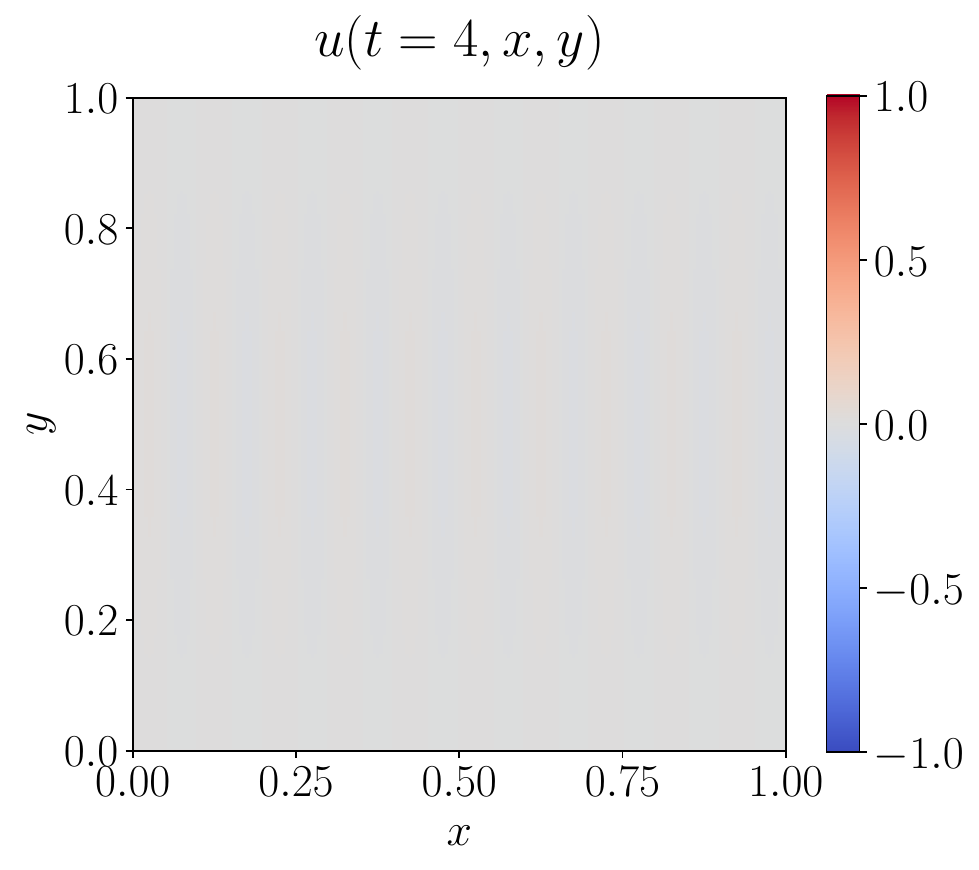}
    \end{subfigure}
    \begin{subfigure}{0.28\linewidth}
        \includegraphics[width=\linewidth]{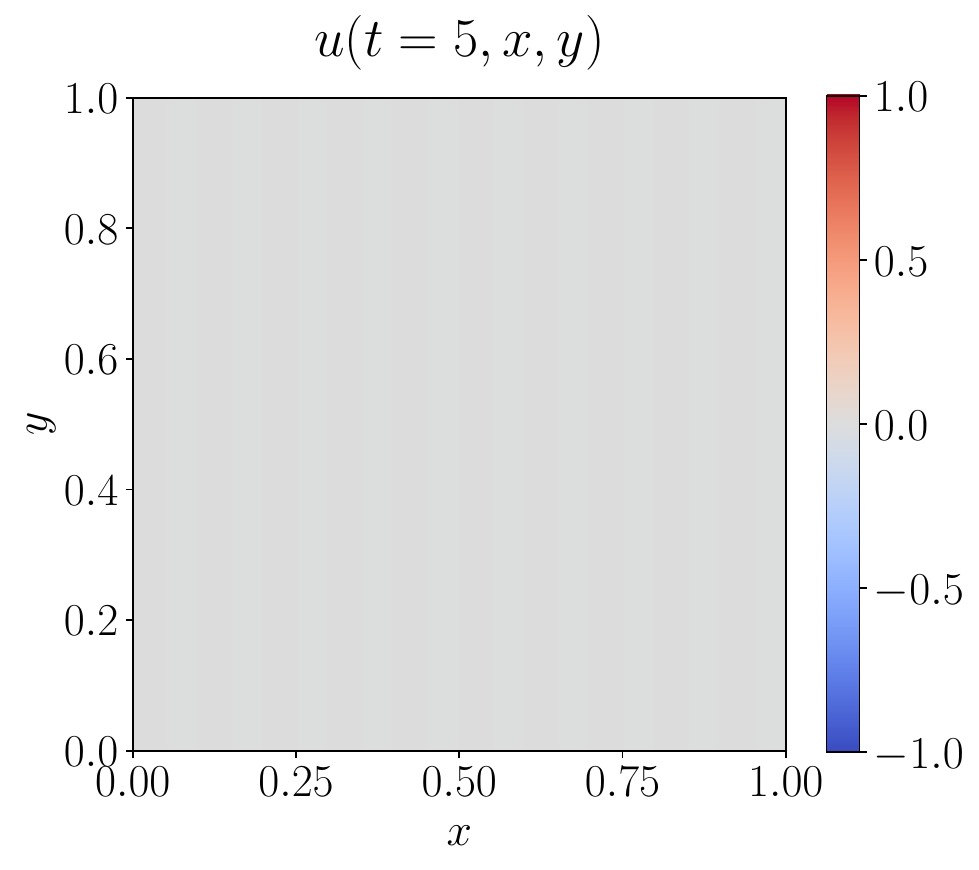}
    \end{subfigure}
    \caption{Visualisation of the solution $u(t, x, y)$ for the Heat-2D MS.}
    \label{fig:heat2d_ms_solution}
\end{figure}

\begin{figure}[h]
    \centering
    \begin{subfigure}{0.28\linewidth}
        \includegraphics[width=\linewidth]{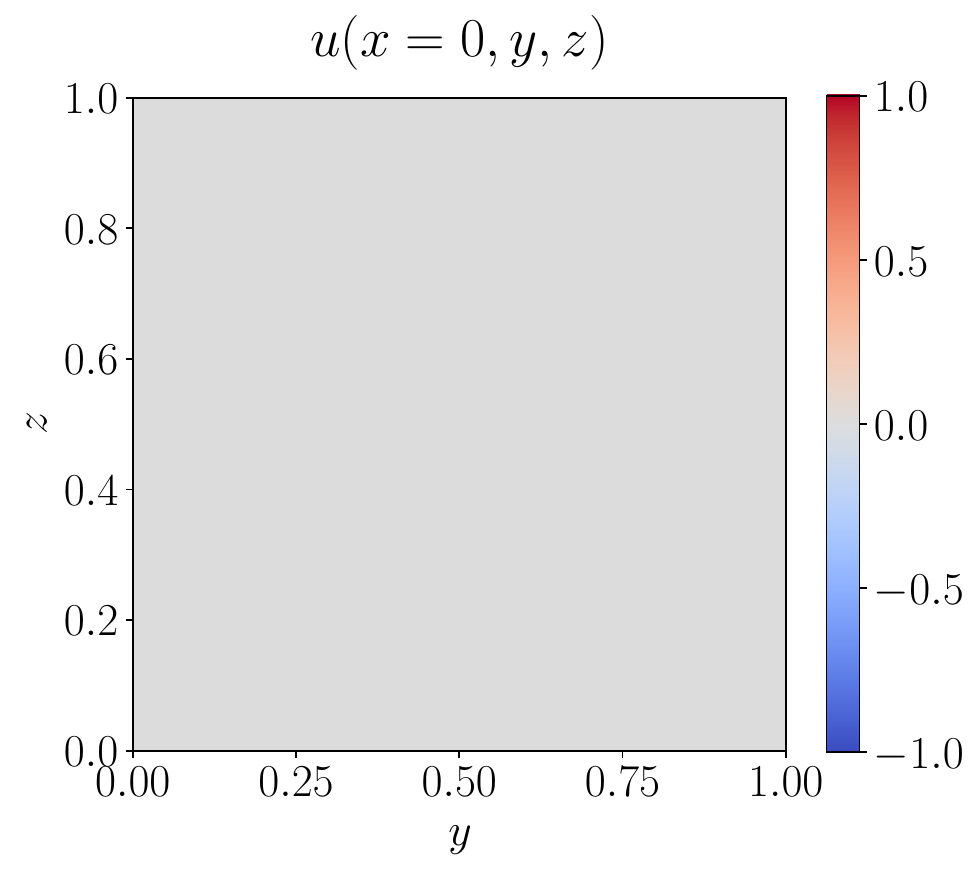}
    \end{subfigure}
    \begin{subfigure}{0.28\linewidth}
        \includegraphics[width=\linewidth]{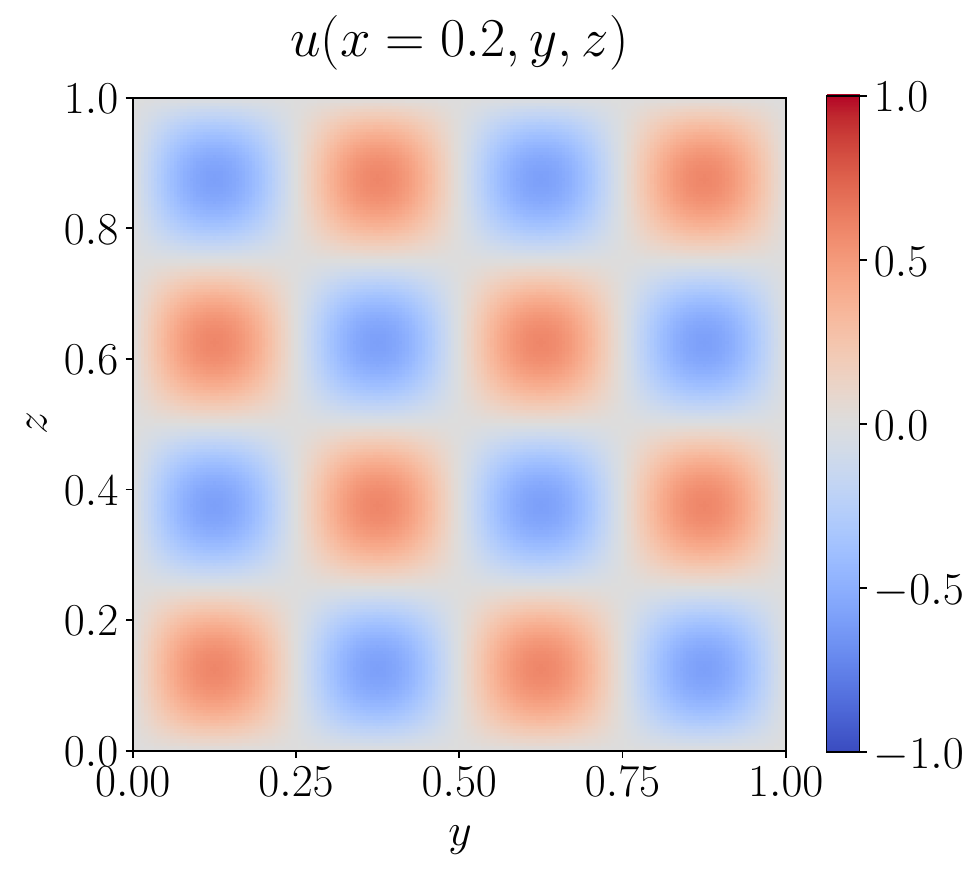}
    \end{subfigure}
    \begin{subfigure}{0.28\linewidth}
        \includegraphics[width=\linewidth]{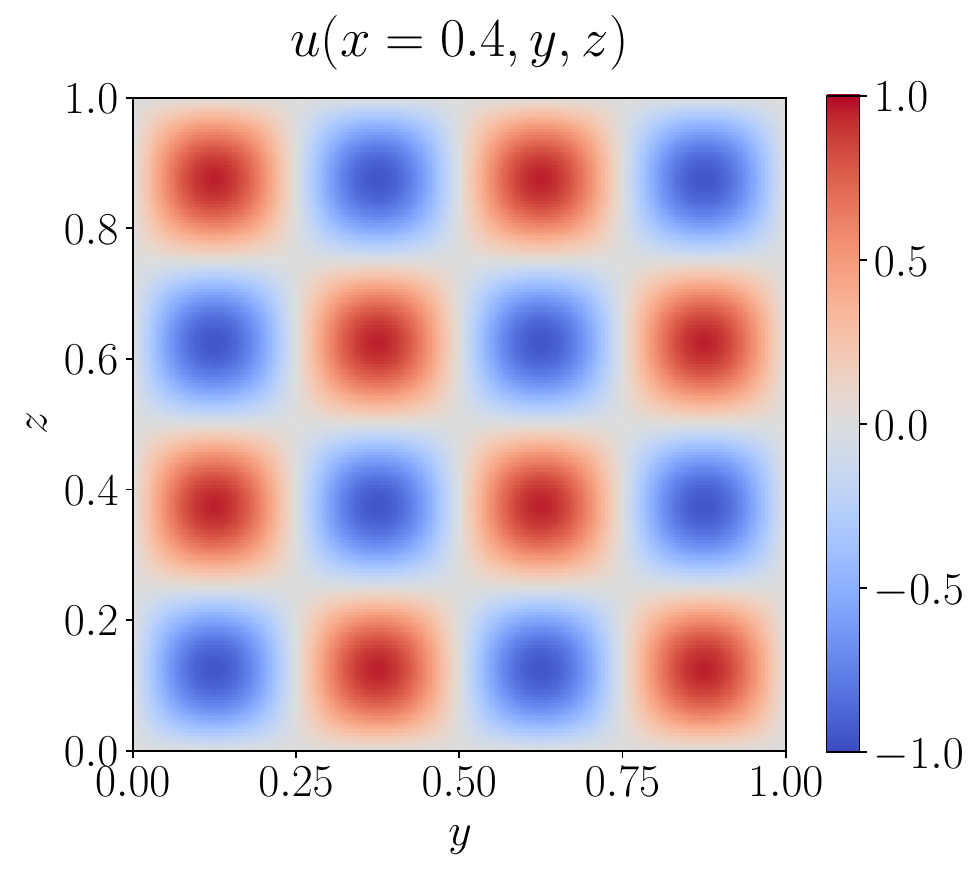}
    \end{subfigure}
    \begin{subfigure}{0.28\linewidth}
        \includegraphics[width=\linewidth]{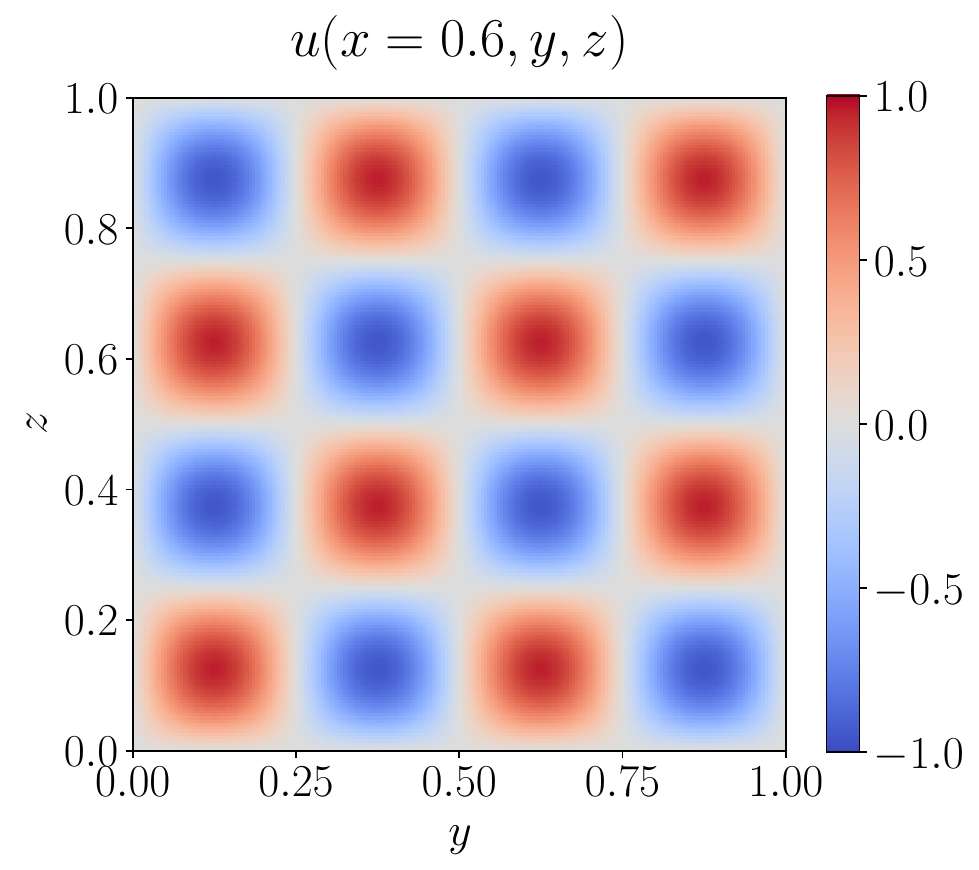}
    \end{subfigure}
    \begin{subfigure}{0.28\linewidth}
        \includegraphics[width=\linewidth]{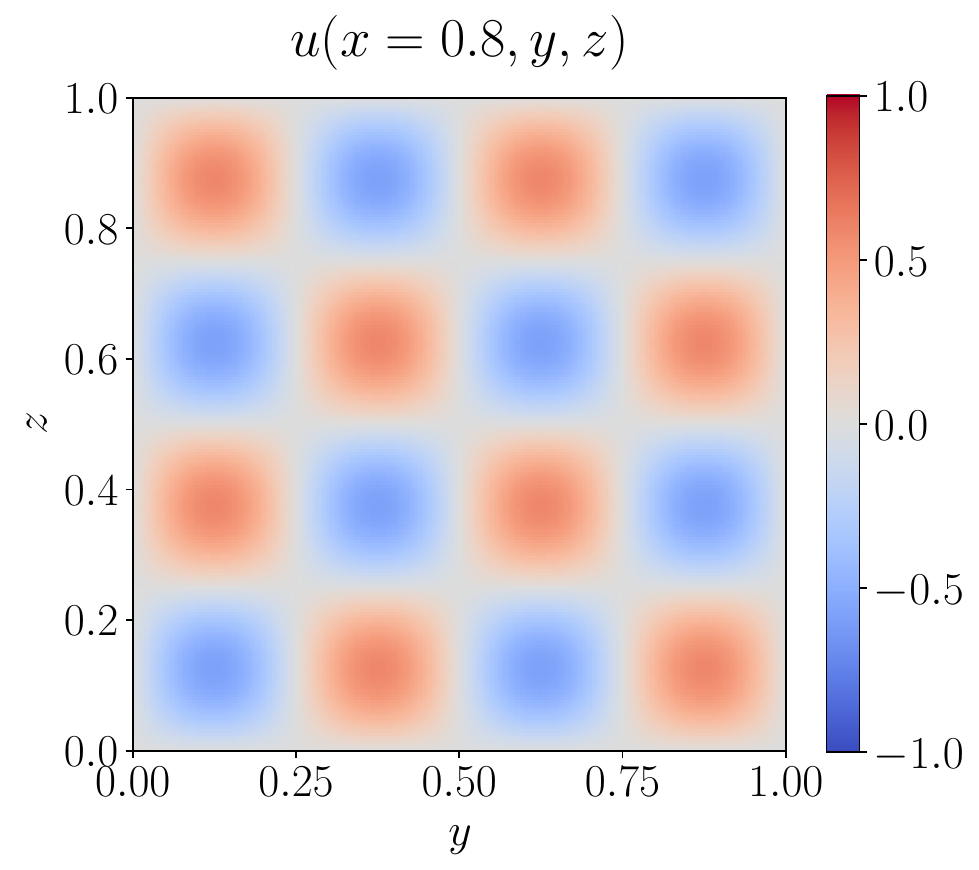}
    \end{subfigure}
    \begin{subfigure}{0.28\linewidth}
        \includegraphics[width=\linewidth]{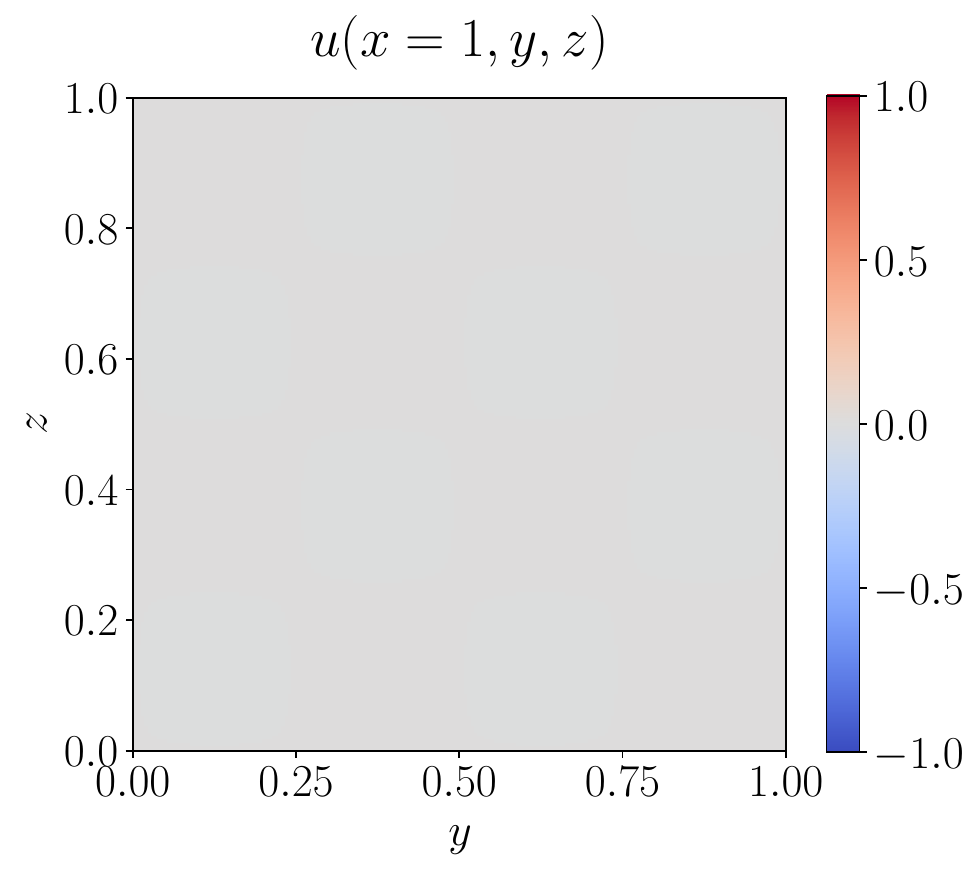}
    \end{subfigure}
    \caption{Visualisation of the solution $u(x, y, z)$ for the Helmholtz.}
    \label{fig:helmholtz3d_solution}
\end{figure}

\paragraph{Helmholtz \citep[SPINN,][]{cho2023separable}.}
To further evaluate the proposed bounds in a higher-dimensional setting, we consider the three-dimensional Helmholtz equation 
\begin{equation}
    \nabla^2 u + k^2u = f(\mathbf{x}), \qquad \mathbf{x}=(x,y,z)\in[0,1]^3,
\end{equation}
with homogeneous Dirichlet boundary conditions. We use the manufactured exact solution
$u(x,y,z)=\sin(A_1\pi x)\sin(A_2\pi y)\sin(A_3\pi z)$, illustrated by Fig.~\ref{fig:helmholtz3d_solution},
where the corresponding source term is
\begin{equation}
    f(x,y,z)=\left[k^2-\pi^2(A_1^2+A_2^2+A_3^2)\right]u(x,y,z).
\end{equation}
In our experiments, we set \(A_1=A_2=A_3=4\) and \(k=1\). We consider the following physical losses in the experiments:
\begin{equation}
\begin{aligned}
\ell_p(\boldsymbol{\theta}, \boldsymbol{\mathrm{x}}) &= \left(\frac{\partial^2 u_{\boldsymbol{\theta}}}{\partial x^2} + \frac{\partial^2 u_{\boldsymbol{\theta}}}{\partial y^2} + \frac{\partial^2 u_{\boldsymbol{\theta}}}{\partial z^2} + ku_{\boldsymbol{\theta}} - f(x,y,z) \right)^2 , \\
\ell_{ic}(\boldsymbol{\theta}, \boldsymbol{\mathrm{x}}) &= \left(u_{\boldsymbol{\theta}}(0,y,z) \right)^2, \\
\ell_{bc}(\boldsymbol{\theta}, \boldsymbol{\mathrm{x}}) &= \frac{1}{5}\bigg[\left(u_{\boldsymbol{\theta}}(1,y,z)\right)^2 + \left(u_{\boldsymbol{\theta}}(x,0,z)\right)^2 + \left(u_{\boldsymbol{\theta}}(x,1,z)\right)^2 \\
&\qquad + \left(u_{\boldsymbol{\theta}}(x,y,0)\right)^2 + \left(u_{\boldsymbol{\theta}}(x,y,1)\right)^2 \bigg]. \\
\end{aligned}
\label{eq:helmholtz3d_losses}
\end{equation}

\paragraph{Flow mixing-2D \citep[SPINN,][]{cho2023separable}.} This is a time-dependent transport equation that describes the mixing of two fluids at an interface under a spatially-varying rotational velocity field. The problem provides a challenging benchmark due to the nonlinear spatial structure of the velocity field and the progressive deformation of the fluid interface over time. Following the setting in SPINN, the governing equation is expressed as:
\begin{equation}
\begin{aligned}
\frac{\partial u}{\partial t} + a(x,y)\frac{\partial u}{\partial x} + b(x,y)\frac{\partial u}{\partial y} &= 0, \qquad (t,x,y)\in[0,4]\times[-4,4]^2, \\
u(0,x,y) &= -\tanh\left(\frac{y}{2}\right), \\
u(t,x,y) &= u_{\mathrm{sol}}(t,x,y), \qquad (x,y)\in\partial[-4,4]^2,
\end{aligned}
\label{eq:flow_mixing}
\end{equation}
where $a(x,y) = -\frac{v_t}{v_{t,\max}}\frac{y}{r}$;  $\; b(x,y) = \frac{v_t}{v_{t,\max}}\frac{x}{r}$;  $\; v_t = \operatorname{sech}^2(r)\tanh(r)$; with $r=\sqrt{x^2+y^2}$. We set $v_{t,\max}=0.385$. The initial and boundary conditions are obtained from the analytical solution  
\begin{equation}
u_{\mathrm{sol}}(t,x,y)=-\tanh\left(\frac{y}{2}\cos(\omega t)-\frac{x}{2}\sin(\omega t)\right), \qquad \omega=\frac{1}{r}\frac{v_t}{v_{t,\max}}.
\label{eq}
\end{equation}
Fig.~\ref{fig:flow_mixing_solution} illustrates the solution map at different time instant. The apparent singularity of $\omega$ at $r=0$ is removable, with $\omega(0,0)=1/v_{t,\max}$. We consider the following physical losses in the experiments:
\begin{equation}
\begin{aligned}
\ell_p(\boldsymbol{\theta},\boldsymbol{\mathrm{x}}) &= \left(\frac{\partial u_{\boldsymbol{\theta}}}{\partial t} + a(x,y)\frac{\partial u_{\boldsymbol{\theta}}}{\partial x} + b(x,y)\frac{\partial u_{\boldsymbol{\theta}}}{\partial y}\right)^2, \\
\ell_{ic}(\boldsymbol{\theta},\boldsymbol{\mathrm{x}}) &= \left(u_{\boldsymbol{\theta}}(0,x,y)+\tanh\left(\frac{y}{2}\right)\right)^2, \\
\ell_{bc}(\boldsymbol{\theta},\boldsymbol{\mathrm{x}}) &= \frac{1}{4}\bigg[ \big(u_{\boldsymbol{\theta}}(t,-4,y)-u_{\mathrm{sol}}(t,-4,y)\big)^2 +  \big(u_{\boldsymbol{\theta}}(t, 4,y)-u_{\mathrm{sol}}(t,4,y)\big)^2 \\ 
&\qquad\qquad + \big(u_{\boldsymbol{\theta}}(t,x,-4)-u_{\mathrm{sol}}(t,x,-4)\big)^2 + \big(u_{\boldsymbol{\theta}}(t,x,4)-u_{\mathrm{sol}}(t,x,4)\big)^2 \bigg]
\end{aligned}
\label{eq:flow_mixing_losses}
\end{equation}

\begin{figure}[t]
    \centering
    \begin{subfigure}{0.28\linewidth}
        \includegraphics[width=\linewidth]{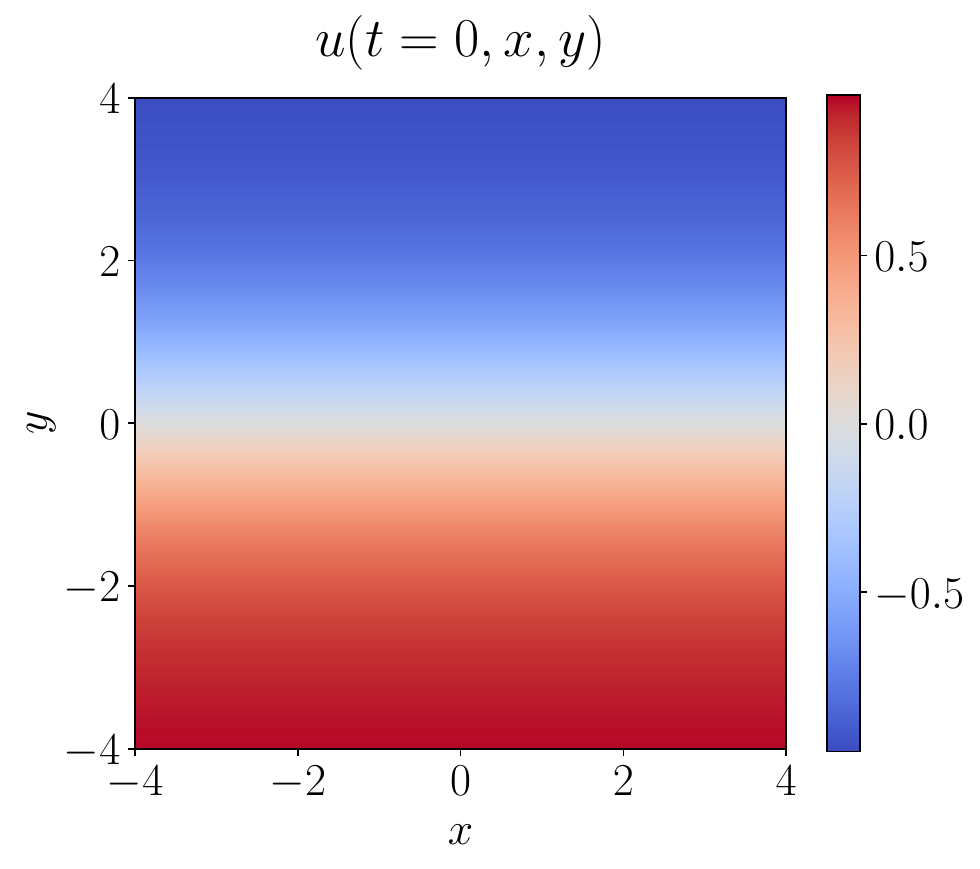}
    \end{subfigure}
    \begin{subfigure}{0.28\linewidth}
        \includegraphics[width=\linewidth]{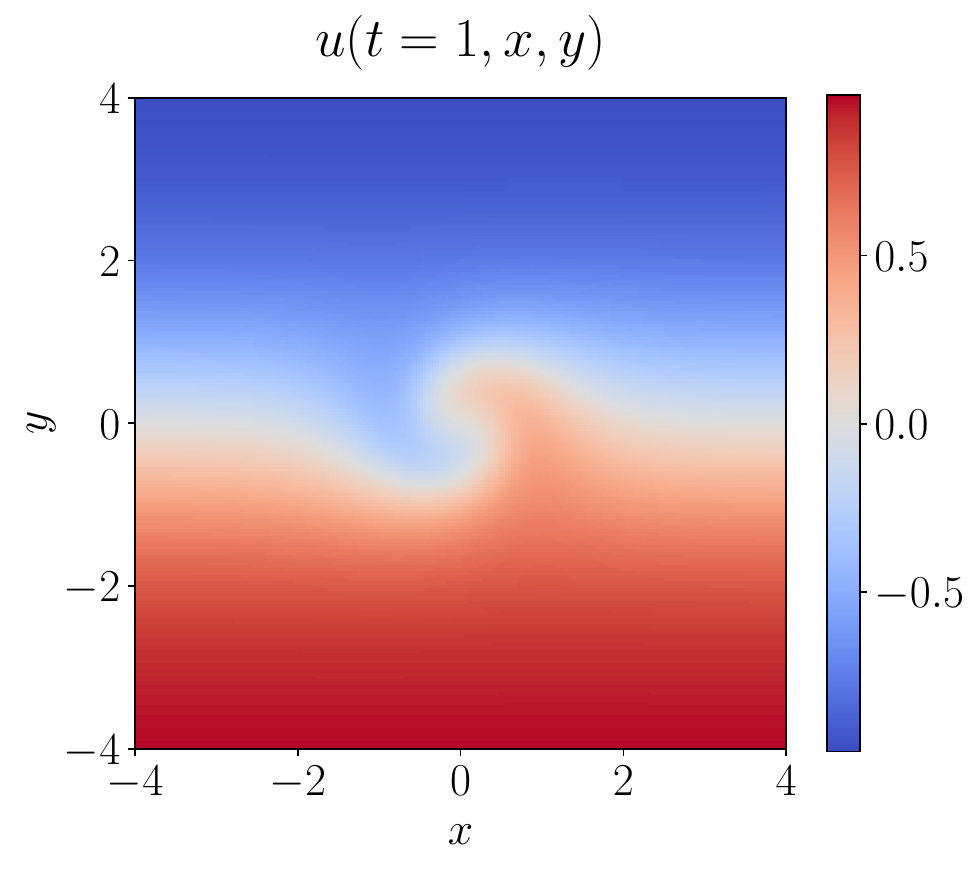}
    \end{subfigure}
    \begin{subfigure}{0.28\linewidth}
        \includegraphics[width=\linewidth]{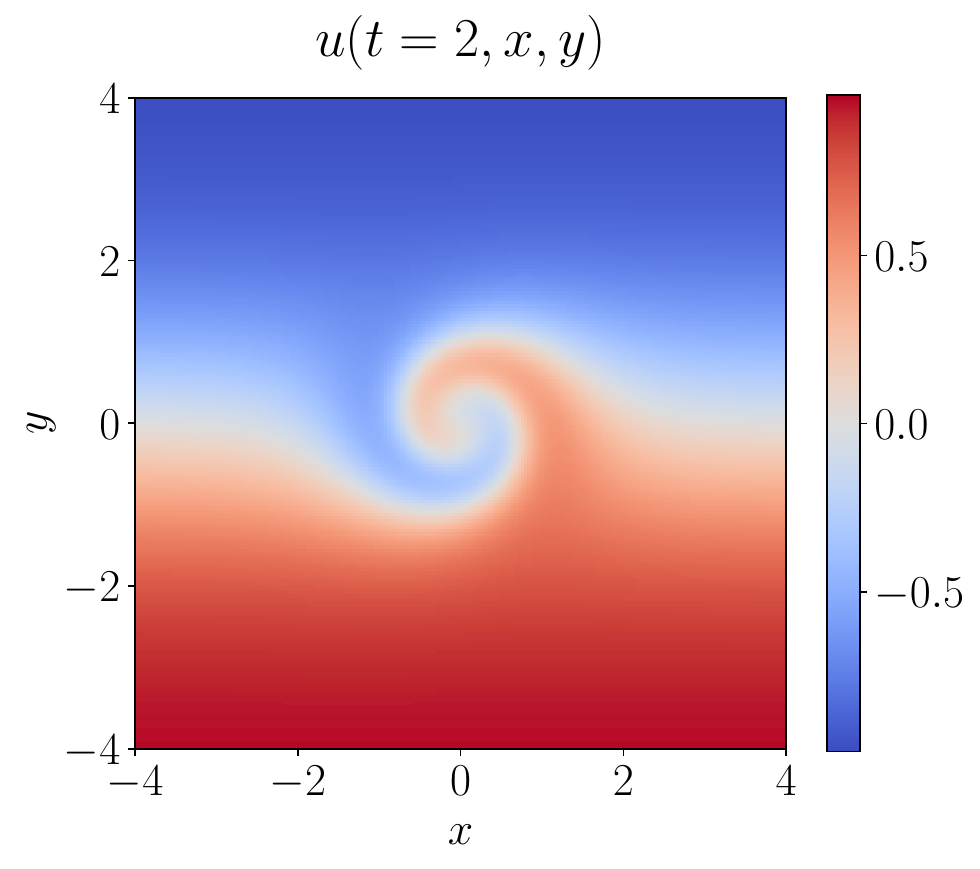}
    \end{subfigure}
    \begin{subfigure}{0.28\linewidth}
        \includegraphics[width=\linewidth]{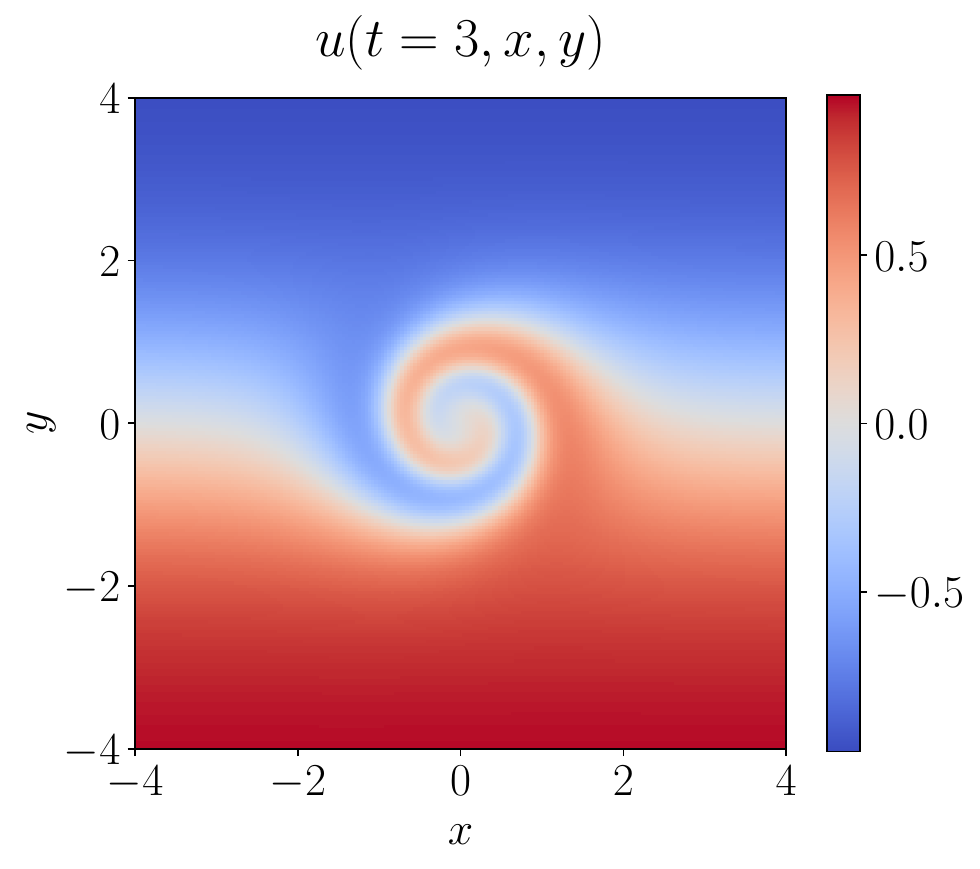}
    \end{subfigure}
    \begin{subfigure}{0.28\linewidth}
        \includegraphics[width=\linewidth]{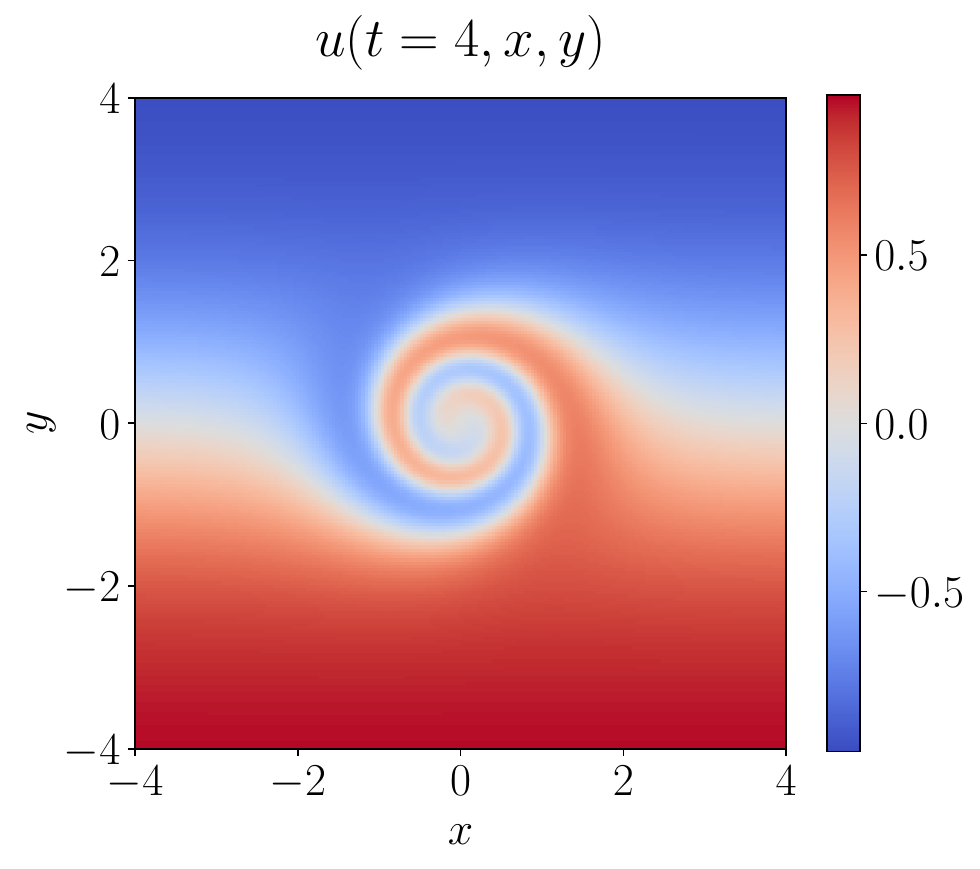}
    \end{subfigure}
    \caption{Visualisation of the solution $u(t, x, y)$ for the Flow Mixing 2D.}
    \label{fig:flow_mixing_solution}
\end{figure}

\subsection{Implementation Details} \label{apdx:implementation}
\paragraph{Data setting.} Throughout the experiments, our primary objective is to numerically validate the proposed PAC-Bayes bounds rather than to pursue state-of-the-art predictive performance. For problems admitting closed-form solutions (\ie Wave-1D, Convection, Heat-2D MS, Helmholtz and Flow Mixing), we uniformly sample 70k collocation points from the domains associated with each physical loss component, namely, $\Omega$ for $\ell_p$, $\Omega_0$ for $\ell_i$, and $\partial\Omega$ for $\ell_b$, in either the 2D $(t,x)$ or 3D ($(t,x,y)$ or $(x, y, z)$) domain. These points are partitioned into prior, posterior, test, and calibration sets of 10k, 20k, 30k, and 10k samples, respectively. For the data-fidelity loss $\ell_d$, we use the same posterior, test, and calibration sets, while reserving only 100 samples for the prior set to estimate the reference loss, gradient quantities and $K_{\mathrm{factor}}$ (see Section~\ref{apdx:algo}). For Burgers, the ground-truth solution is available only at 25,600 discrete points rather than in closed form. We therefore partition these observational data into prior, posterior, test, and calibration sets of 100, 10k, 10k, and 5.5k samples, respectively. For the physical collocation points of Burgers, we instead allocate 5k, 10k, 15k, and 5k samples to the prior, posterior, test, and calibration sets, respectively, such that the posterior set has the same size as the observational posterior set in the balanced setting. In the unbalanced setting, we vary the number of observational posterior samples by selecting a fixed subset consisting of the first $m_d^{\mathrm{posterior}}$ samples of the corresponding posterior set. Similarly, when varying the number of observational calibration samples, we select the first $m_d^{\mathrm{calibration}}$ samples of the calibration set. For each benchmark, all datasets and their partitions are generated once and reused across experimental configurations to ensure a consistent evaluation protocol.

\paragraph{Model and training setting.} We use a standard PINN with three hidden layers of width $256$ and a scalar output, with a $\tanh$ activation function applied after each hidden layer. The resulting network contains approximately $d_{\boldsymbol{\theta}} \approx 1.3 \times 10^5$ trainable parameters. We normalise the model inputs using the mean and standard deviation determined by the corresponding ranges of the domain $\Omega$. In Step 1, the prior model is initialised from a uniform distribution scaled by $1/\sqrt{\mathrm{fan}_{\mathrm{in}}}$, where $\mathrm{fan}_{\mathrm{in}}$ denotes the number of input connections to the neurons in the corresponding layer. The prior is then trained with the NTK scheme \citep{PINN_NTK} for $N_T=20$k iterations for Wave-1D, Convection and Flow Mixing, and $N_T=30$k iterations for the rest. For each random seed, we train a single prior model, which is subsequently used both to estimate the constants entering the bounds and to initialise the posterior optimisation for all bounds in Step 2. Consequently, all bounds are evaluated using the same estimated constants and posterior initialisation, ensuring a consistent comparison.

In Step 2, we optimise each bound with the confidence parameter fixed at $\delta=0.1$ and the variance fixed at $\sigma^2=10^{-6}$ for both the prior distribution $\pi$ and posterior distribution $\rho$. The parameters of the posterior distribution are optimised for $N_{T'}=5$k iterations. We use the Adam optimiser with a batch size of $300$ in both steps. The learning rate is initialised at $10^{-3}$ in Step 1 and $10^{-7}$ in Step 2, with an exponential decay factor of $0.95$ applied every $1{,}000$ iterations. The two partial derivatives of the $\kl_+^{-1}$ are defined based on \citep{Viallard2021AGF}. All experiments are implemented in PyTorch~\citep{pytorch} and NumPy. For both calibration and certification of the PAC-Bayes bounds, the expected empirical risks are estimated via Monte Carlo sampling using $1{,}000$ models drawn from the posterior distribution. Each experiment is repeated with five independent random seeds, with each seed fixed throughout the entire training, calibration, and certification pipeline. The reported variability therefore reflects the stochasticity of the complete procedure, rather than that of an individual stage in isolation. All experiments are conducted on a single NVIDIA RTX 5090 GPU.

\subsection{Empirical assessment of the population-level assumptions}
Assumptions~\ref{ass:kappa}--\ref{ass:calibration_dominance} impose population-level regularity conditions that cannot in general be verified uniformly from finite samples. Our certificates are therefore conditional on their validity. In practice, we assess their plausibility using data independent of posterior training and models sampled throughout the localised hypothesis region $\Theta_{\mathrm{loc}}$. Specifically, we examine the CGF-to-variance ratios underlying Assumption~\ref{ass:kappa}, the observed input-gradient energies relative to the prescribed envelopes $U_i$ in Assumption~\ref{ass:gradient_envelope}, and the ratio $C(\rho)/C_{K,G}(\rho)$ in Assumption~\ref{ass:calibration_dominance}. These diagnostics are used to select problem-dependent constants and to stress-test the assumptions over the empirically relevant region, but should not be interpreted as uniform certification of the corresponding population conditions. Moreover, all constants are selected using \(S^{\mathrm{prior}}\); \(S^{\mathrm{cal}}\) is used only for held-out diagnostic assessment and final PAC-Bayes calibration.

\subsubsection{Empirical Assessment of the Assumption~\ref{ass:kappa}} \label{apdx:ass_kappa}
We evaluate the empirical ratio $\kappa_i(\lambda)$ over the considered range of $\lambda$ using models sampled throughout $\Theta_{\mathrm{loc}}$. For each seed-specific trained prior and loss component, we set $\kappa_{p,i}$ to the maximum observed ratio over the considered models and $\lambda$-values. In almost all cases, the maximum is attained as $\lambda\to0^+$, where the ratio approaches $1$, in which case we conservatively set $\kappa_{p,i}=1$. Only a small number of seed--loss combinations require values slightly larger than one. Figure~\ref{fig:kappa_diagnostic} reports the mean empirical ratio across random seeds together with the corresponding minimum--maximum range. The data, IC, and BC components remain particularly close to one, whereas the PDE component generally decreases more noticeably with $\lambda$ and exhibits greater variability across seeds.
\begin{figure}[h]
    \centering
    \begin{subfigure}{0.42\linewidth}
        \includegraphics[width=\linewidth]{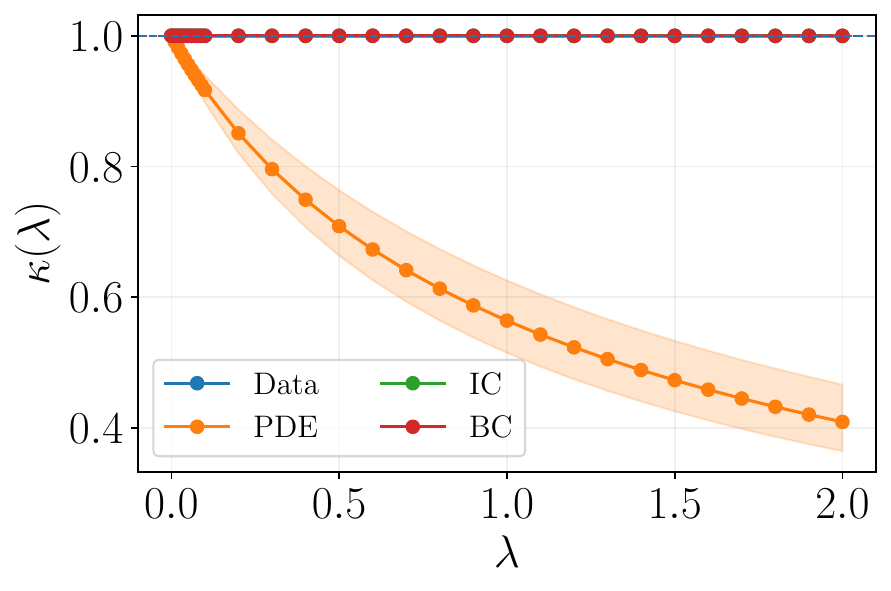}
        \caption{Wave-1D}
    \end{subfigure}
    \begin{subfigure}{0.42\linewidth}
        \includegraphics[width=\linewidth]{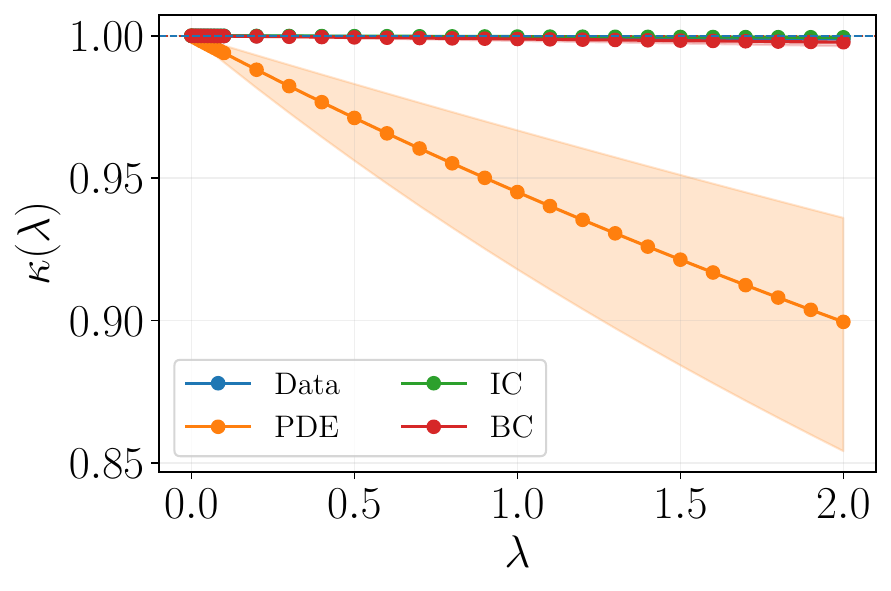}
        \caption{Convection}
    \end{subfigure}
    \begin{subfigure}{0.42\linewidth}
        \includegraphics[width=\linewidth]{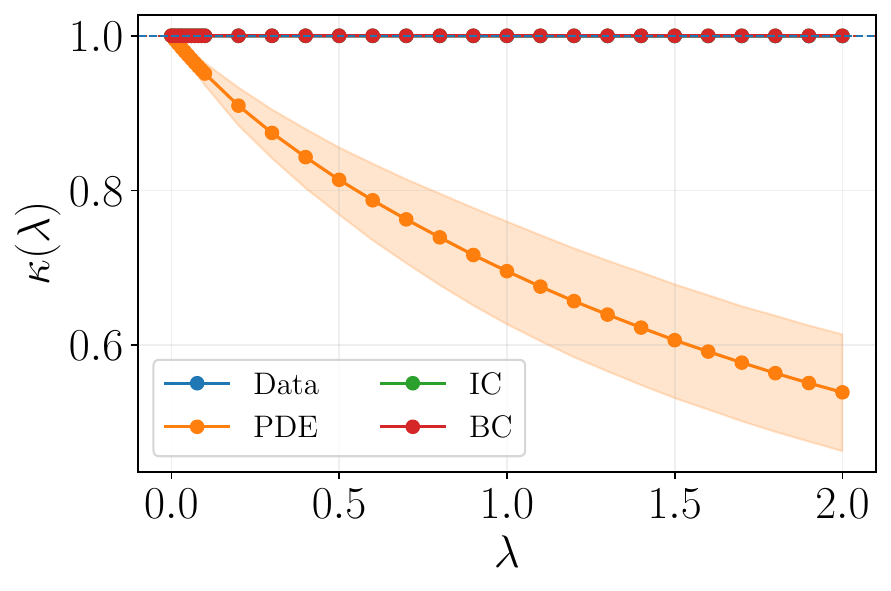}
        \caption{Burgers}
    \end{subfigure}
    \begin{subfigure}{0.42\linewidth}
        \includegraphics[width=\linewidth]{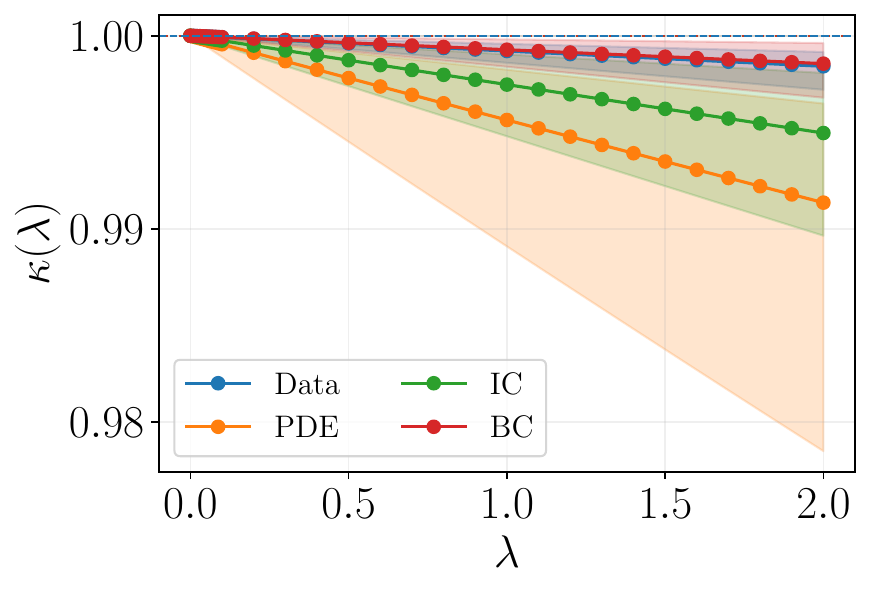}
        \caption{Heat-2D MS}
    \end{subfigure}
    \begin{subfigure}{0.42\linewidth}
        \includegraphics[width=\linewidth]{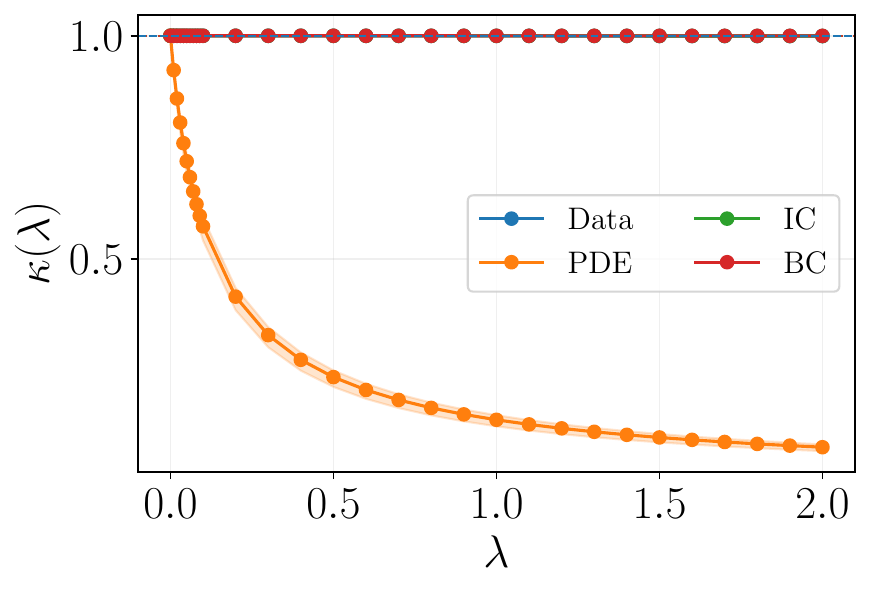}
        \caption{Helmholtz}
    \end{subfigure}
    \begin{subfigure}{0.42\linewidth}
        \includegraphics[width=\linewidth]{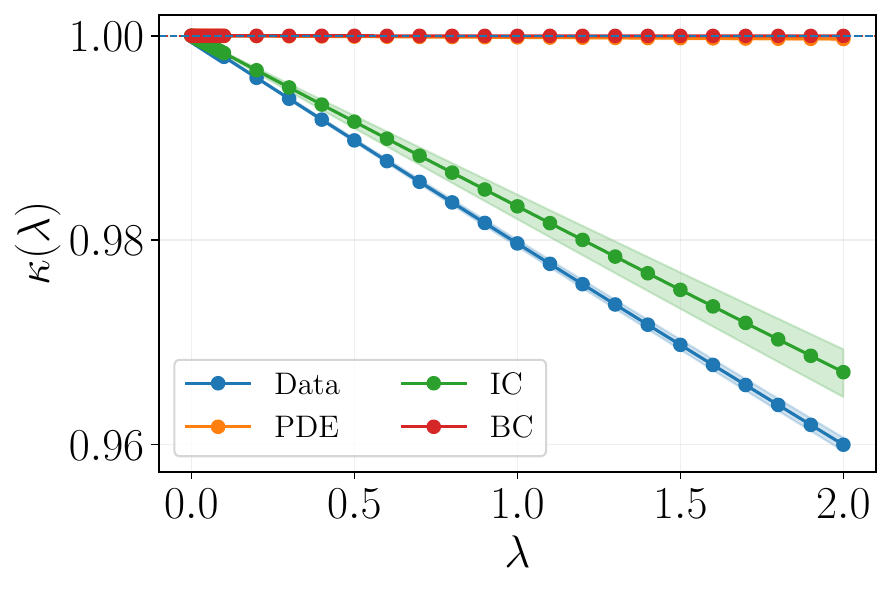}
        \caption{Flow Mixing-2D}
    \end{subfigure}
    \caption{Empirical CGF-to-variance ratio $\kappa_i(\lambda)$ as a function of $\lambda$ for each loss component.}
    \label{fig:kappa_diagnostic}
\end{figure}

\subsubsection{Empirical Assessment of the Assumption~\ref{ass:gradient_envelope}} \label{apdx:ass_grad_enveloppe}
We empirically stress-test the prescribed gradient envelopes $U_i$ using calibration inputs and models sampled throughout $\Theta_{\mathrm{loc}}$. Figure~\ref{fig:gradient_enveloppe_wave1D}--\ref{fig:gradient_enveloppe_flow_mixing} shows typical distributions of the observed input-gradient energies, together with their mean, reference threshold $G_i$, maximum observed value, and prescribed envelope $U_i$. In general, the data, IC, and BC components exhibit strongly concentrated gradient distributions, whereas the PDE component displays a substantially broader and heavier-tailed distribution spanning several orders of magnitude. Nevertheless, the maximum observed gradient remains below $U_i$ for every component and across the considered models. In particular, the PDE component provides the most stringent empirical stress test, with its observed maximum approaching the envelope more closely than the typical gradient scale. These results support the choice $U_i=U_{G,\mathrm{factor}}G_{\mathrm{ref},i}$ with $U_{G,\mathrm{factor}}\gg1$, while illustrating why $U_i$ is used only as a concentration envelope rather than as a representative scale of the empirical input-gradient complexity.

\begin{figure}[t]
    \centering
    \begin{subfigure}{0.4\linewidth}
        \includegraphics[width=\linewidth]{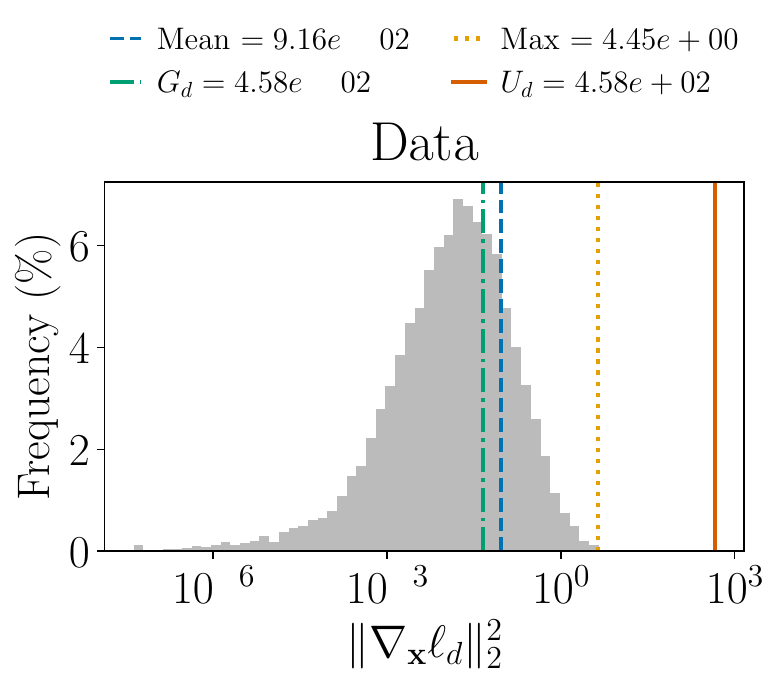}
    \end{subfigure}
    \begin{subfigure}{0.4\linewidth}
        \includegraphics[width=\linewidth]{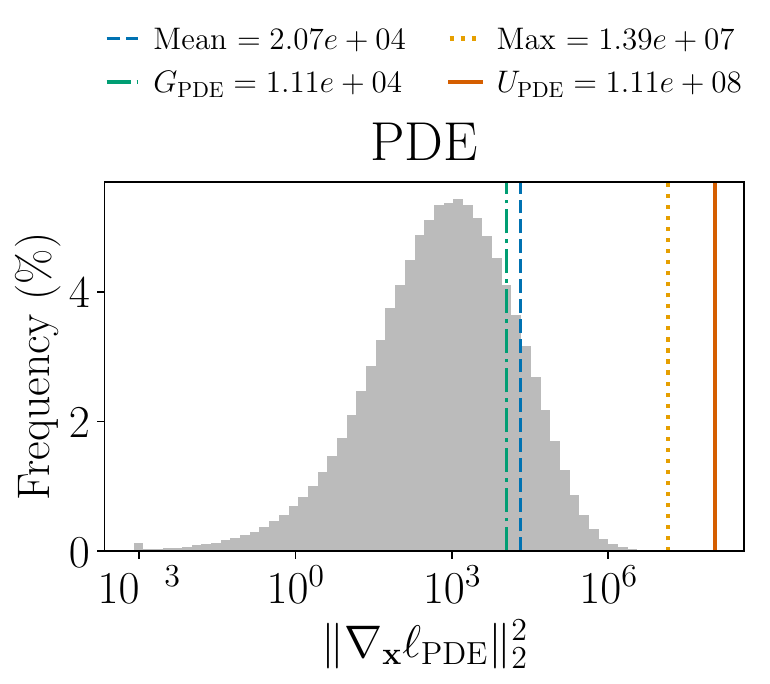}
    \end{subfigure}
    \begin{subfigure}{0.4\linewidth}
        \includegraphics[width=\linewidth]{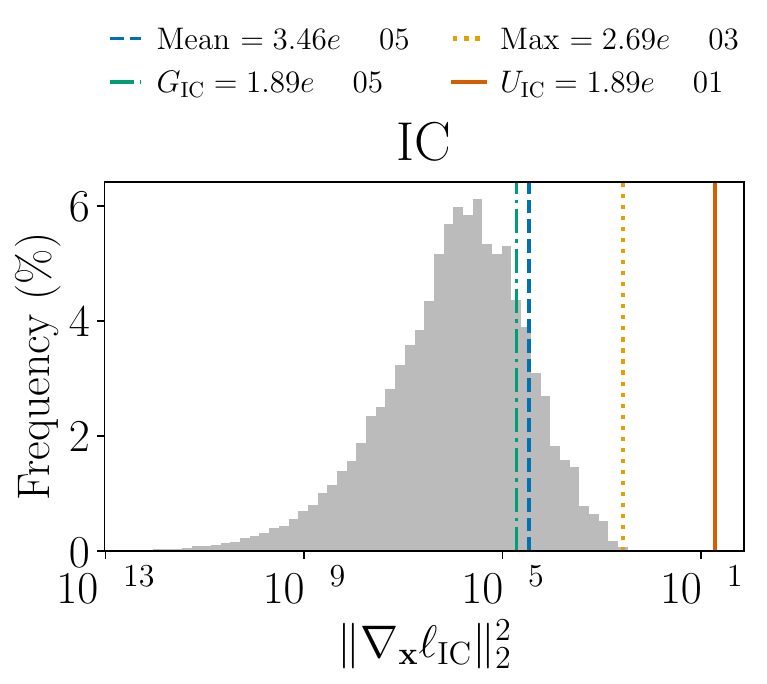}
    \end{subfigure}
    \begin{subfigure}{0.4\linewidth}
        \includegraphics[width=\linewidth]{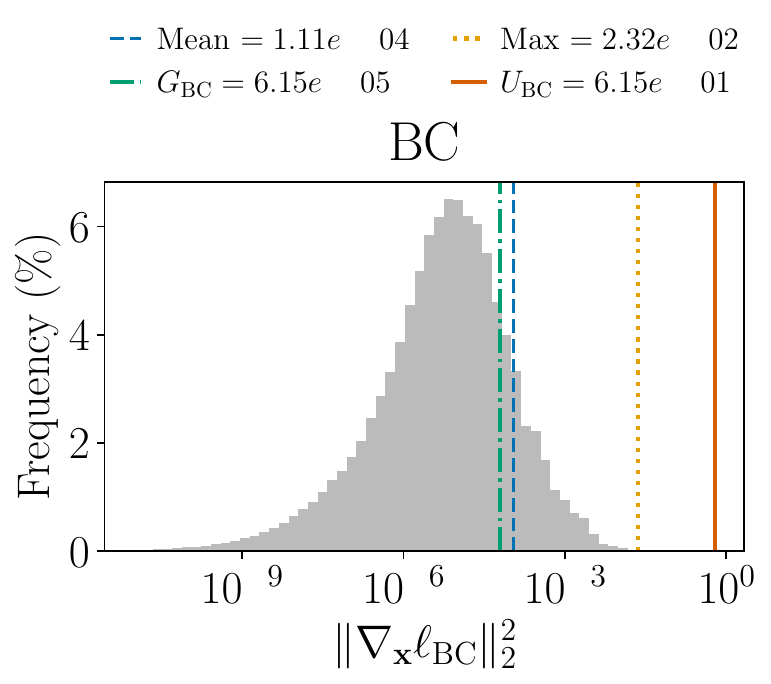}
    \end{subfigure}
    \caption{Empirical input-gradient energy distributions for Wave-1D, with a common scaling factor of $U_{G,\mathrm{factor}} = 10^4$ applied to all loss components.}
    \label{fig:gradient_enveloppe_wave1D}
\end{figure}

\begin{figure}[h]
    \centering
    \begin{subfigure}{0.4\linewidth}
        \includegraphics[width=\linewidth]{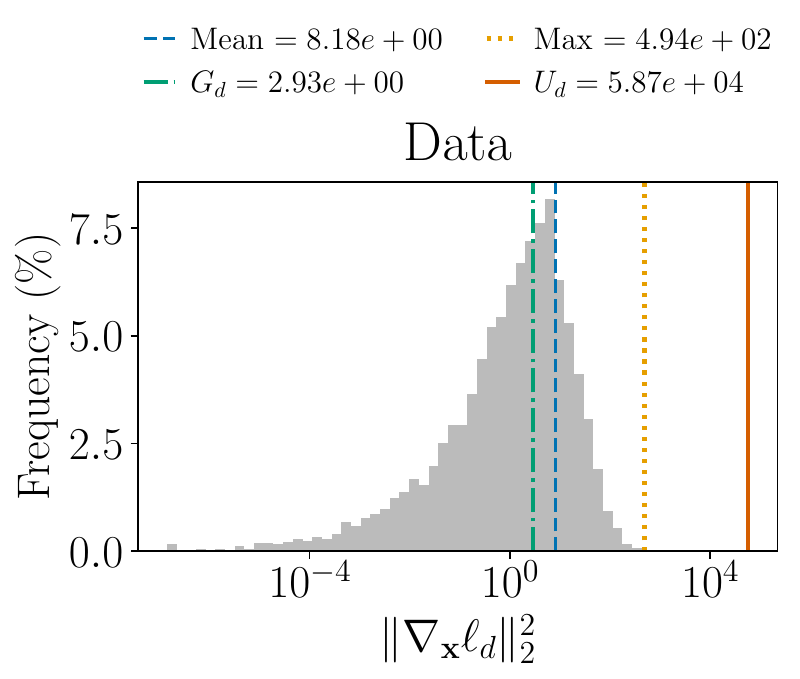}
    \end{subfigure}
    \begin{subfigure}{0.4\linewidth}
        \includegraphics[width=\linewidth]{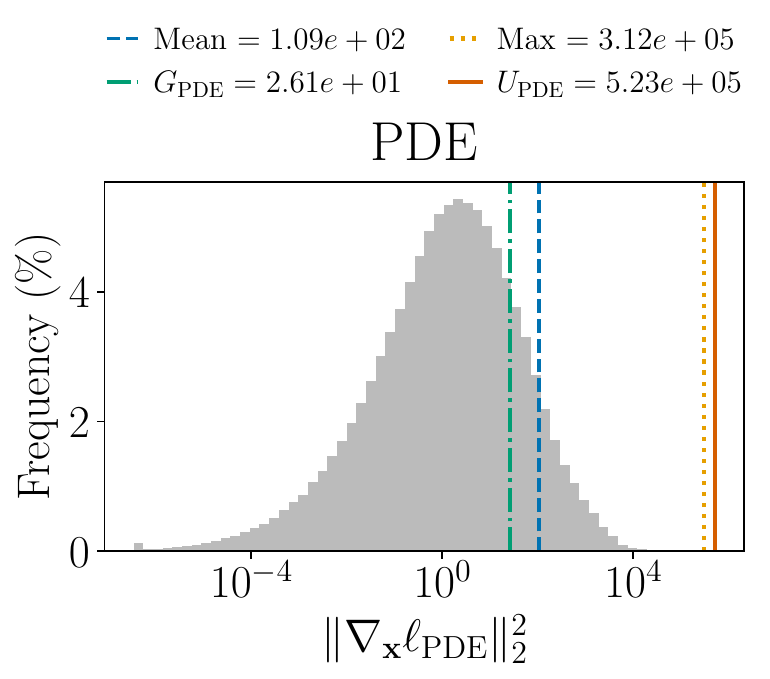}
    \end{subfigure}
    \begin{subfigure}{0.4\linewidth}
        \includegraphics[width=\linewidth]{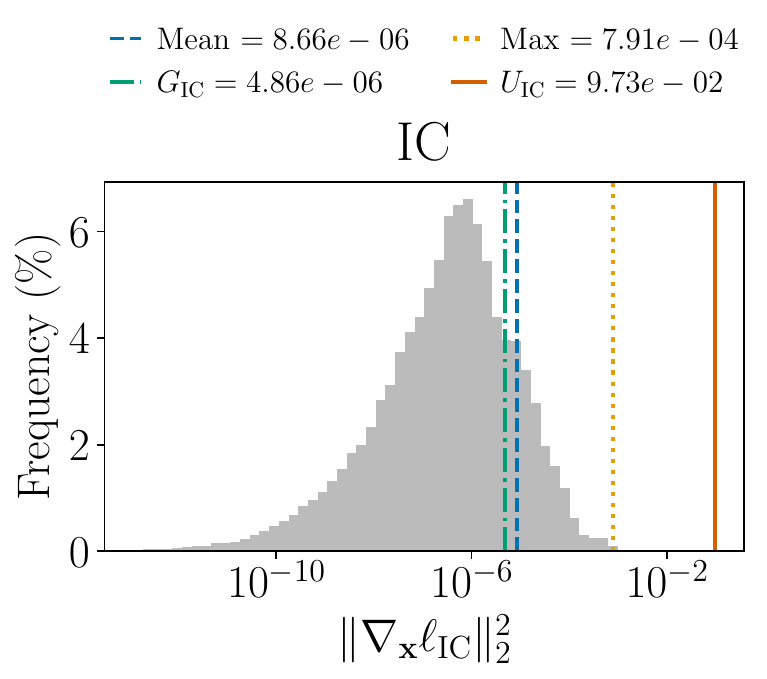}
    \end{subfigure}
    \begin{subfigure}{0.4\linewidth}
        \includegraphics[width=\linewidth]{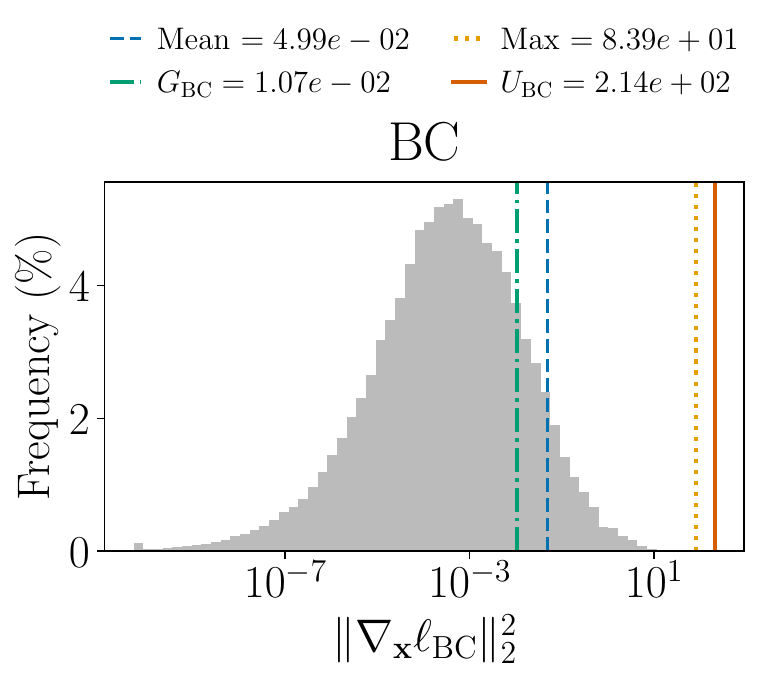}
    \end{subfigure}
    \caption{Empirical input-gradient energy distributions for Convection, with a common scaling factor of $U_{G,\mathrm{factor}} = 2\times 10^4$ applied to all loss components.}
    \label{fig:gradient_enveloppe_convection}
\end{figure}

\clearpage

\begin{figure}[t]
    \centering
    \begin{subfigure}{0.4\linewidth}
        \includegraphics[width=\linewidth]{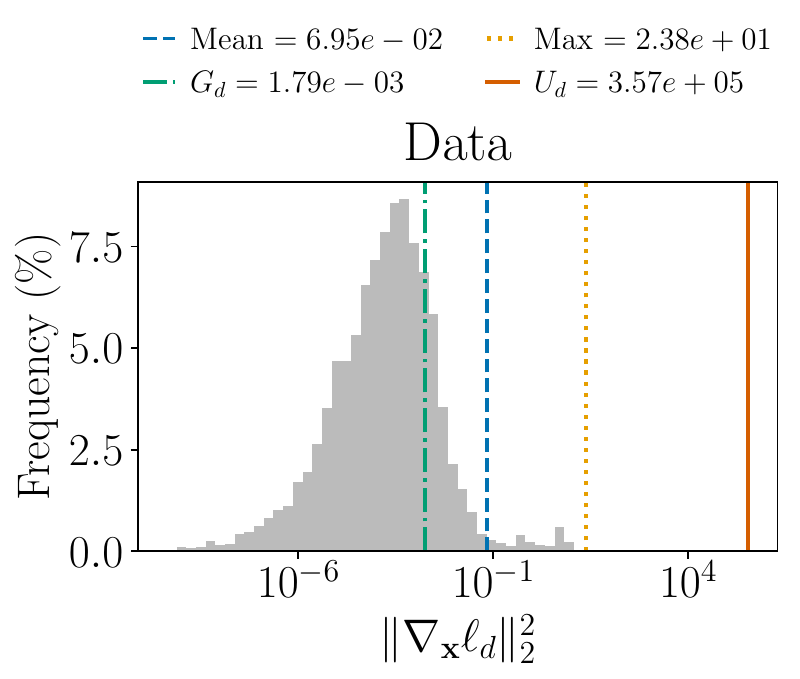}
    \end{subfigure}
    \begin{subfigure}{0.4\linewidth}
        \includegraphics[width=\linewidth]{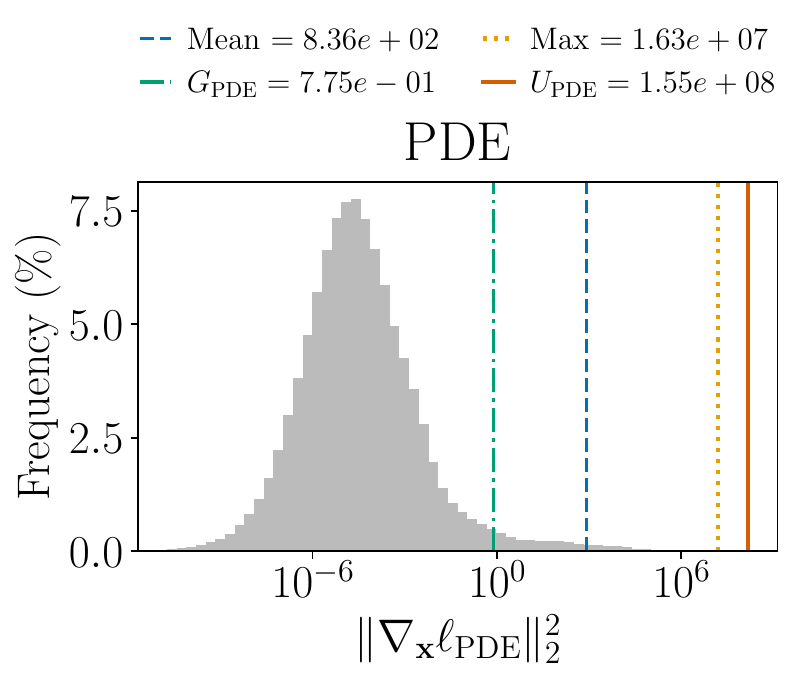}
    \end{subfigure}
    \begin{subfigure}{0.4\linewidth}
        \includegraphics[width=\linewidth]{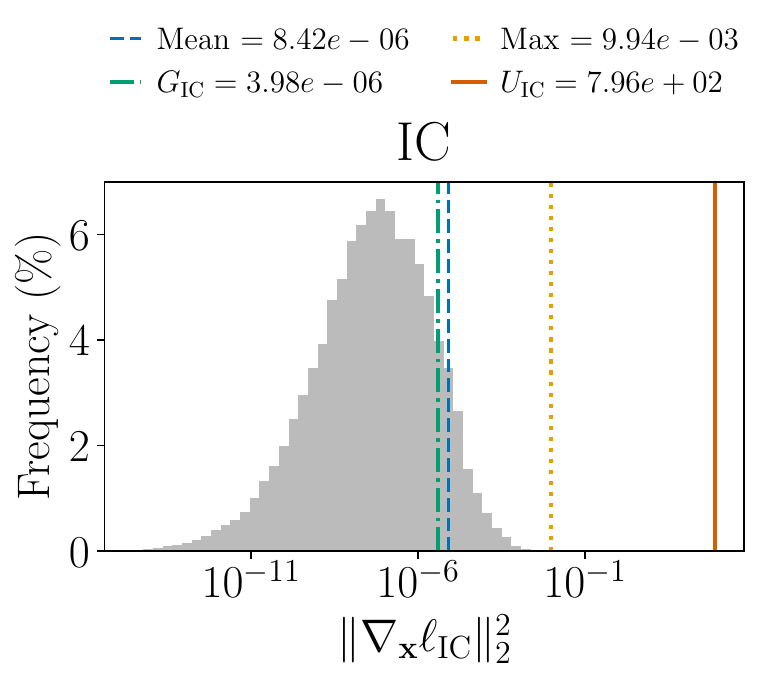}
    \end{subfigure}
    \begin{subfigure}{0.4\linewidth}
        \includegraphics[width=\linewidth]{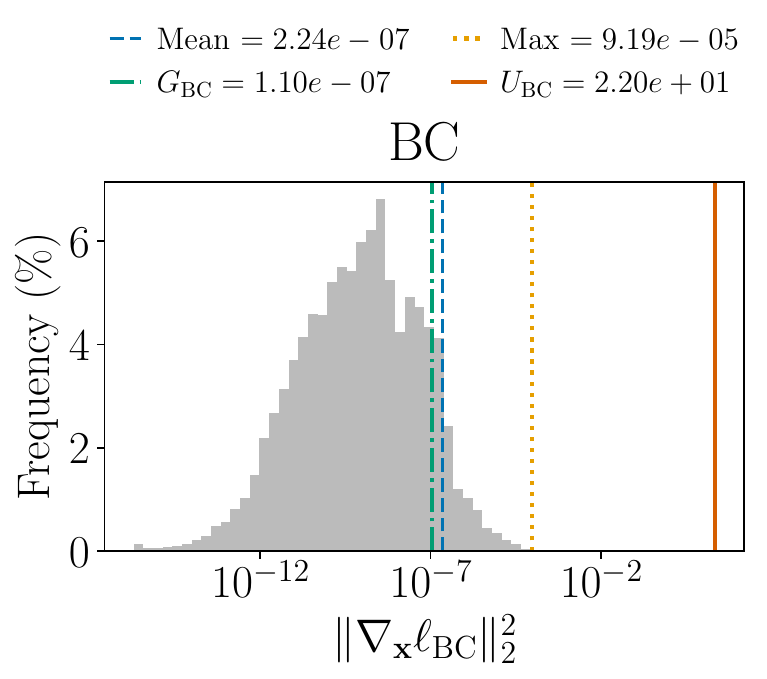}
    \end{subfigure}
    \caption{Empirical input-gradient energy distributions for Burgers, with a common scaling factor of $U_{G,\mathrm{factor}} = 2\times 10^8$ applied to all loss components.}
    \label{fig:gradient_enveloppe_burgers}
\end{figure}

\begin{figure}[h]
    \centering
    \begin{subfigure}{0.4\linewidth}
        \includegraphics[width=\linewidth]{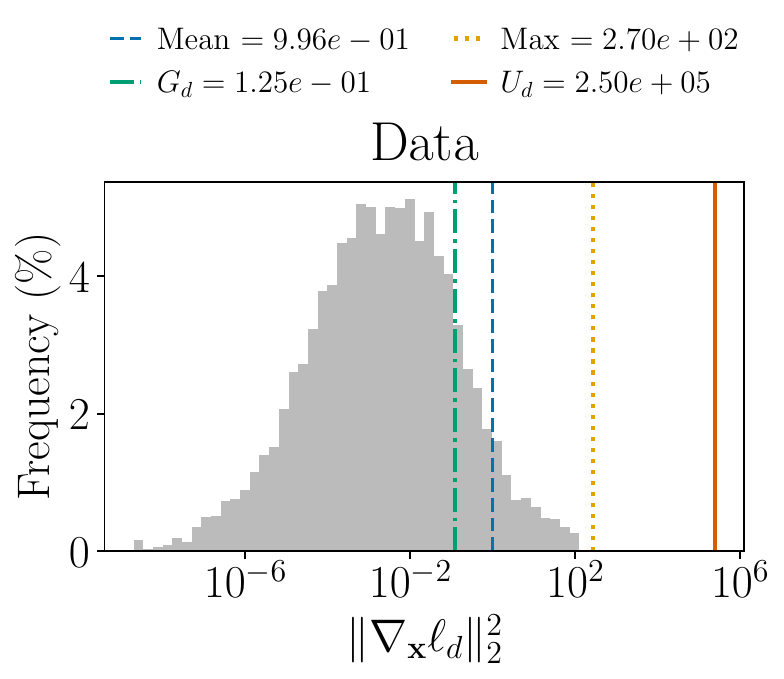}
    \end{subfigure}
    \begin{subfigure}{0.4\linewidth}
        \includegraphics[width=\linewidth]{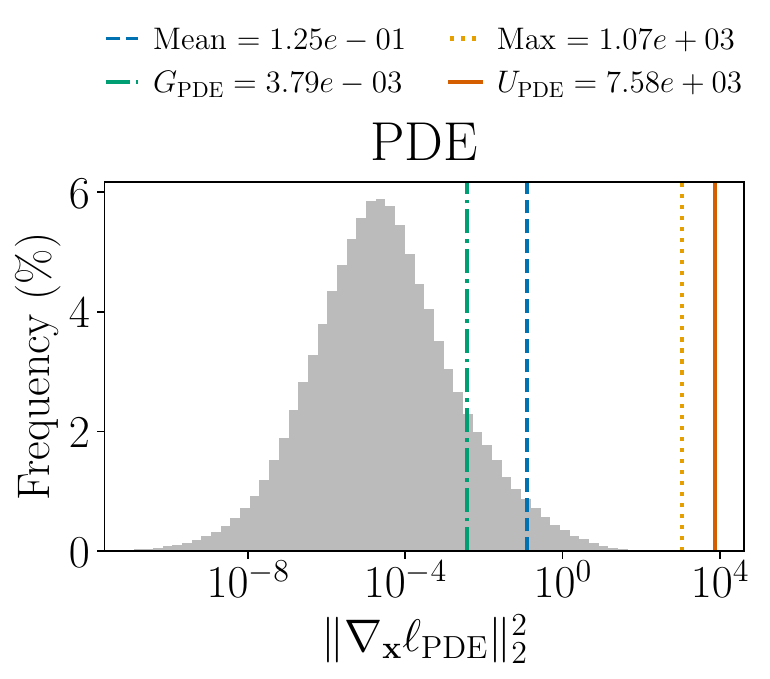}
    \end{subfigure}
    \begin{subfigure}{0.4\linewidth}
        \includegraphics[width=\linewidth]{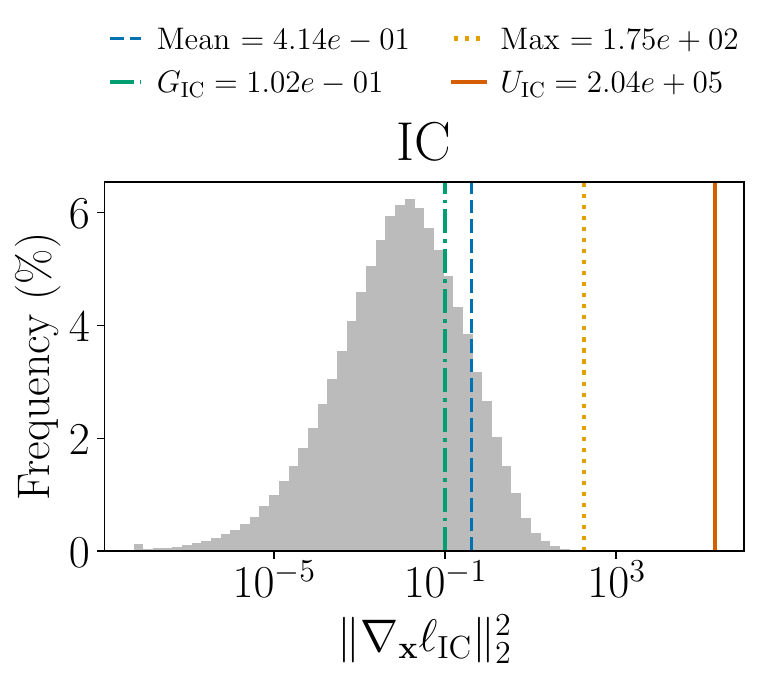}
    \end{subfigure}
    \begin{subfigure}{0.4\linewidth}
        \includegraphics[width=\linewidth]{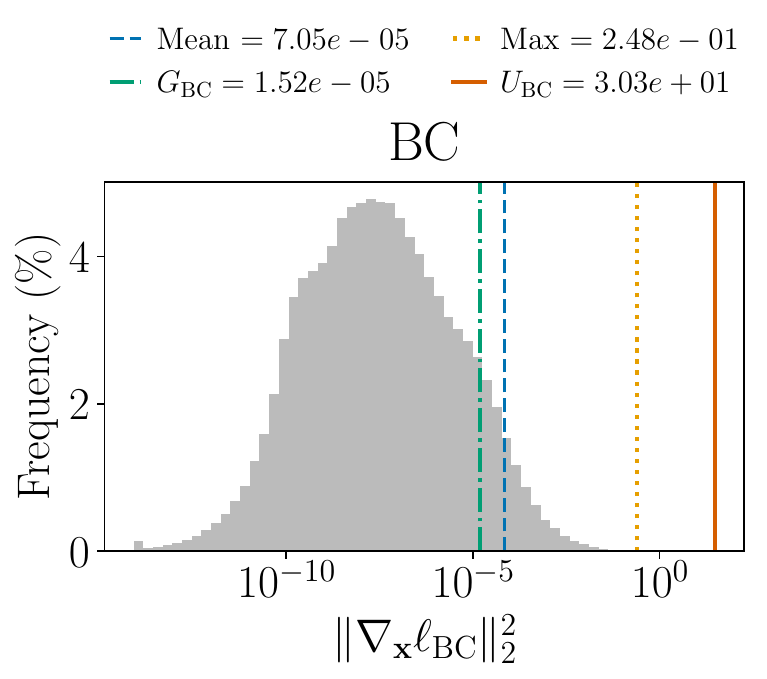}
    \end{subfigure}
    \caption{Empirical input-gradient energy distributions for Heat-2D MS, with a common scaling factor of $U_{G,\mathrm{factor}} = 2\times 10^6$ applied to all loss components.}
    \label{fig:gradient_enveloppe_heat2d_ms}
\end{figure}

\begin{figure}[h]
    \centering
    \begin{subfigure}{0.4\linewidth}
        \includegraphics[width=\linewidth]{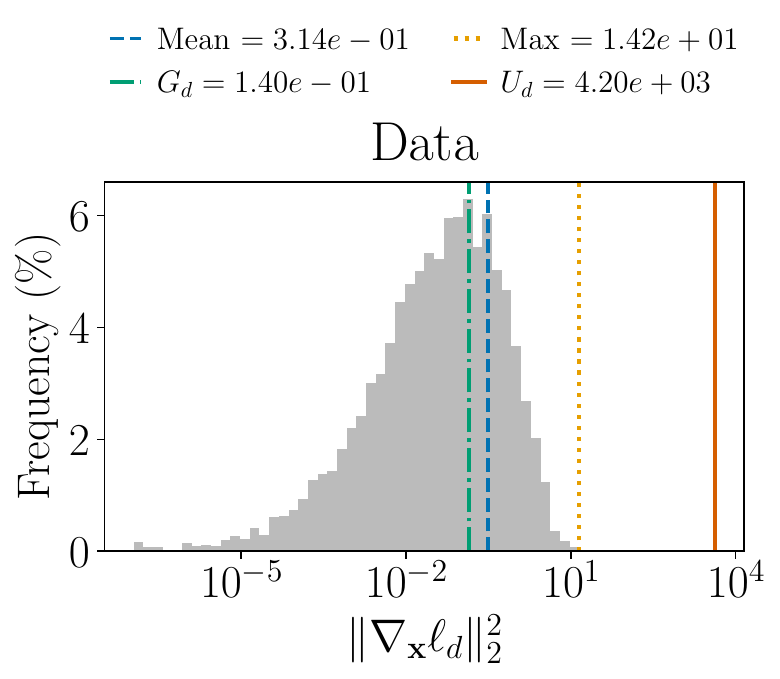}
    \end{subfigure}
    \begin{subfigure}{0.4\linewidth}
        \includegraphics[width=\linewidth]{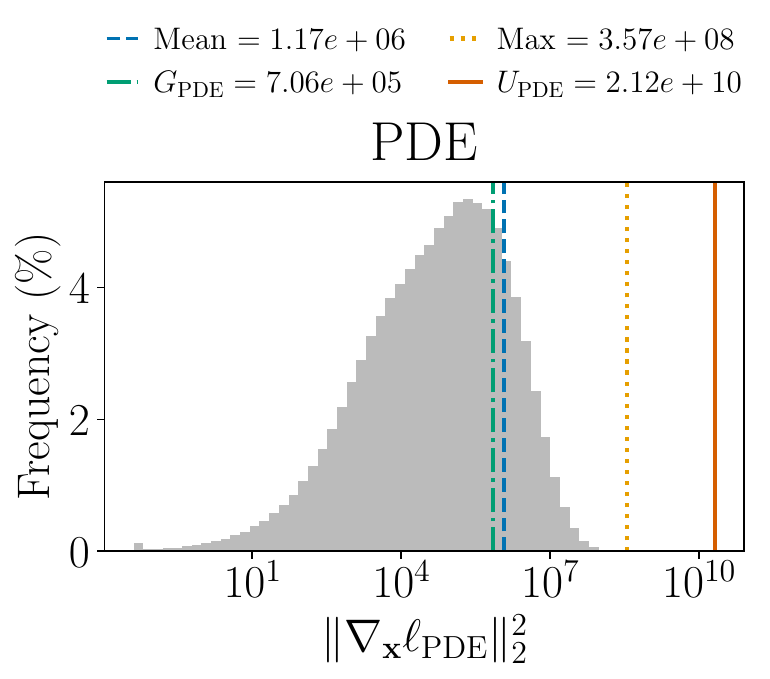}
    \end{subfigure}
    \begin{subfigure}{0.4\linewidth}
        \includegraphics[width=\linewidth]{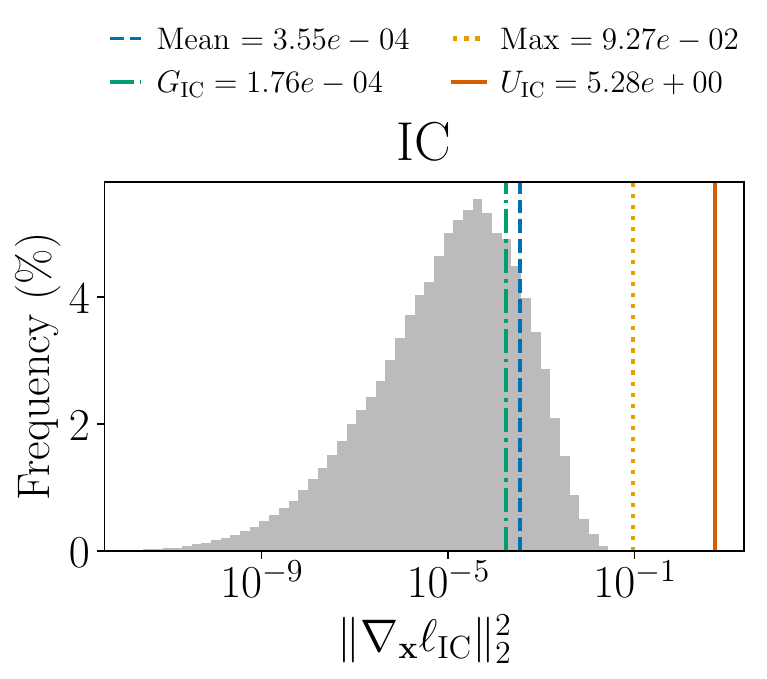}
    \end{subfigure}
    \begin{subfigure}{0.4\linewidth}
        \includegraphics[width=\linewidth]{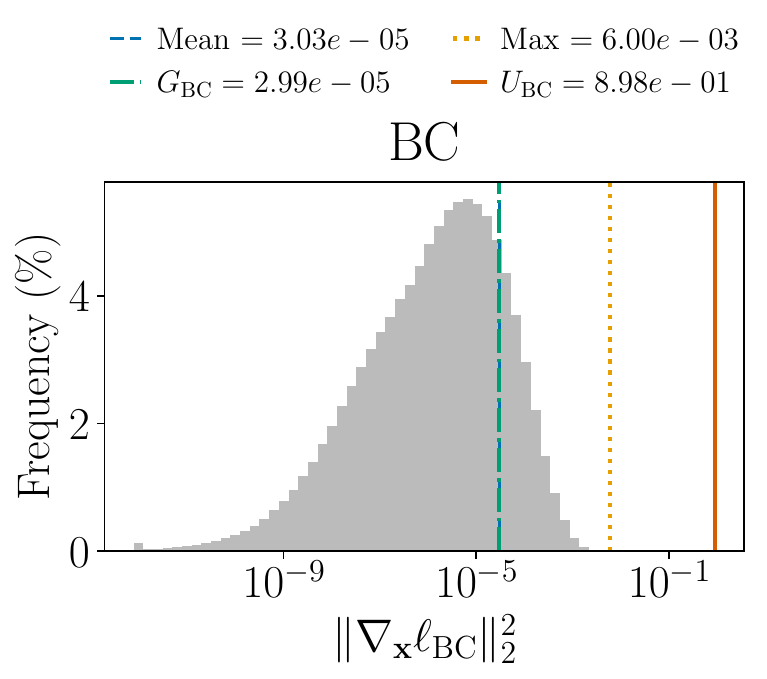}
    \end{subfigure}
    \caption{Empirical input-gradient energy distributions for Helmholtz, with a common scaling factor of $U_{G,\mathrm{factor}} = 3\times 10^4$ applied to all loss components.}
    \label{fig:gradient_enveloppe_helmholtz3d}
\end{figure}

\begin{figure}[h]
    \centering
    \begin{subfigure}{0.4\linewidth}
        \includegraphics[width=\linewidth]{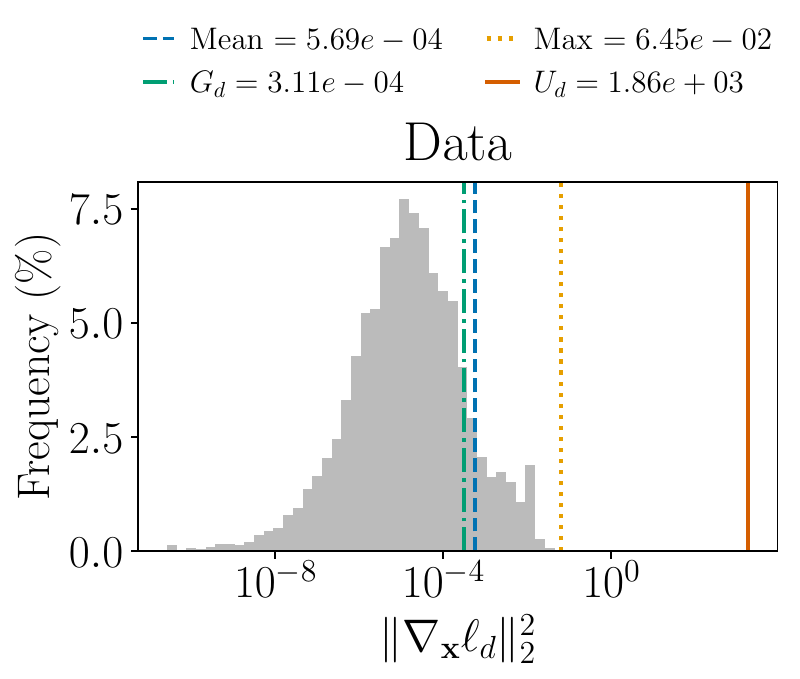}
    \end{subfigure}
    \begin{subfigure}{0.4\linewidth}
        \includegraphics[width=\linewidth]{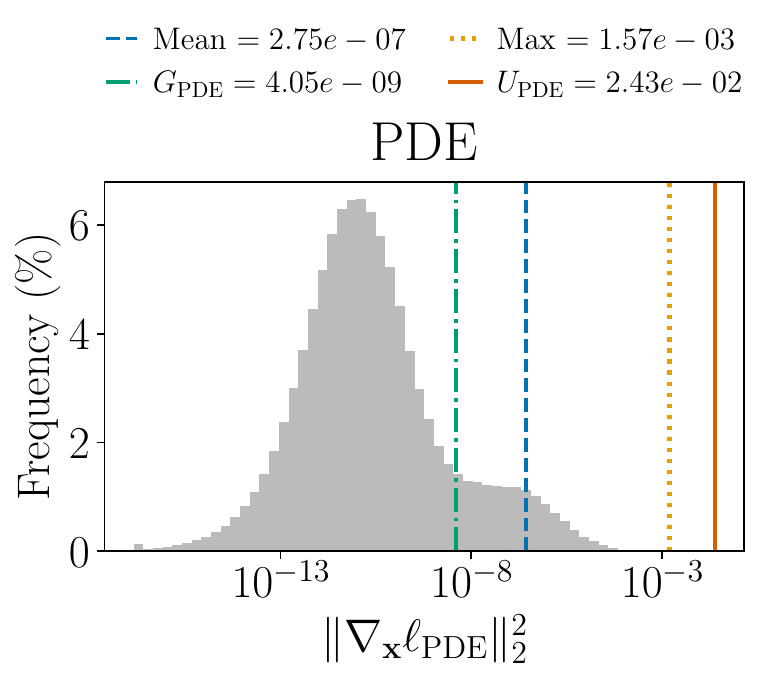}
    \end{subfigure}
    \begin{subfigure}{0.4\linewidth}
        \includegraphics[width=\linewidth]{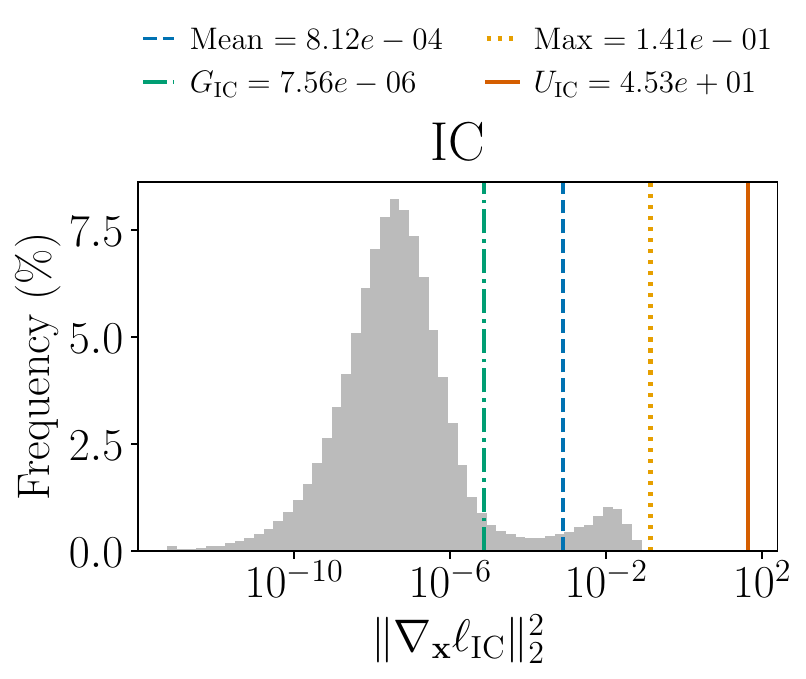}
    \end{subfigure}
    \begin{subfigure}{0.4\linewidth}
        \includegraphics[width=\linewidth]{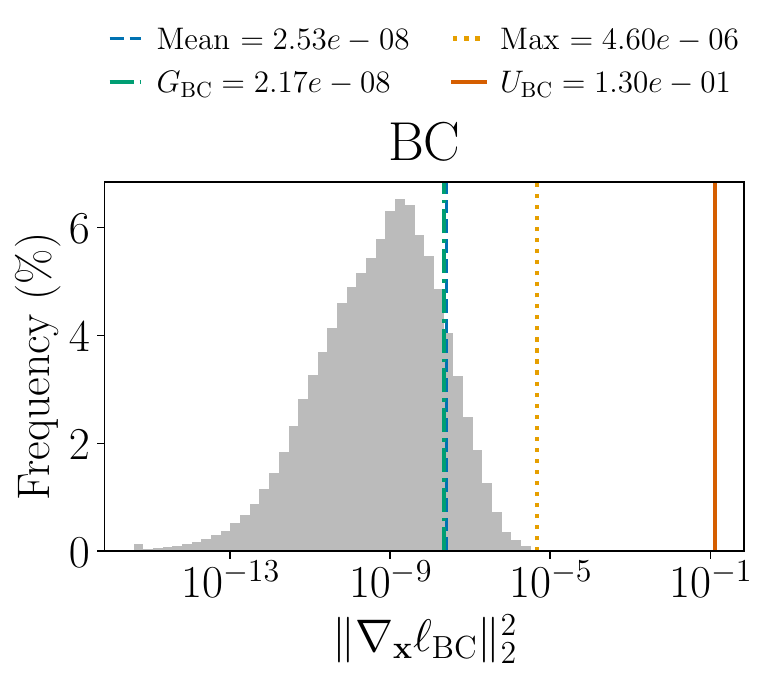}
    \end{subfigure}
    \caption{Empirical input-gradient energy distributions for Flow Mixing-2D, with a common scaling factor of $U_{G,\mathrm{factor}} = 6\times 10^6$ applied to all loss components.}
    \label{fig:gradient_enveloppe_flow_mixing}
\end{figure}

\clearpage

\subsubsection{Empirical Assessment of the Assumption~\ref{ass:calibration_dominance}} \label{apdx:ass_calib_dominance}
As described in Section~\ref{apdx:algo}, we set $K_i = K_{\mathrm{factor}} K_{\mathrm{ref},i}$ and fix $G_i = G_{\mathrm{ref},i}$, selecting $K_{\mathrm{factor}}$ by empirically assessing the calibration-dominance condition. Specifically, we evaluate the empirical ratio $C_S(\rho)/C_{S,K,G}(\rho)$ over a range of $K_{\mathrm{factor}}^{\mathrm{cert}}$ values knowing that assumption~\ref{ass:calibration_dominance} requires this ratio to remain below one. The $K_{\mathrm{factor}}^{\mathrm{cert}}$ is chosen from the prior set $S^{\mathrm{prior}}$ as the smallest value where such ratio is below $1$ for all losses, then we assess the ratio at that value on the held-out dataset $S^{\mathrm{calib}}$ to see whether it also holds on unseen data. Figure~\ref{fig:K_factor_diagnostic} reports the mean ratio across random seeds, together with the corresponding minimum-maximum range. With the exception of Flow Mixing, the empirical ratios remain below one for all loss components over the entire range considered and decrease consistently as $K_{\mathrm{factor}}$ increases. This behavior is expected: as the clipping thresholds increase, the clipped losses progressively approach their unclipped counterparts, increasing the numerator of $C_{K,G}(\rho)$, while its denominator remains clipped at the same threshold, hence resulting in larger $C_{K,G}(\rho)$ and consequently smaller ratio $C(\rho) / C_{K,G}(\rho)$. It is noteworthy mentioning that the $S^{\mathrm{calib}}$-evaluated ratios at the chosen $K_{\mathrm{factor}}$ values are consistently smaller than $1$, indicating that the chosen values allows the assumption to be satisfied on unseen data as well. 

\begin{figure}[ht]
    \centering
    \begin{subfigure}{0.8\linewidth}
    \centering
        \includegraphics[width=\linewidth]{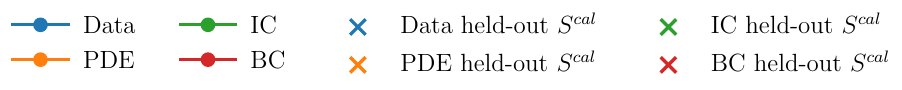}
    \end{subfigure}
    \begin{subfigure}{0.4\linewidth}
        \includegraphics[width=\linewidth]{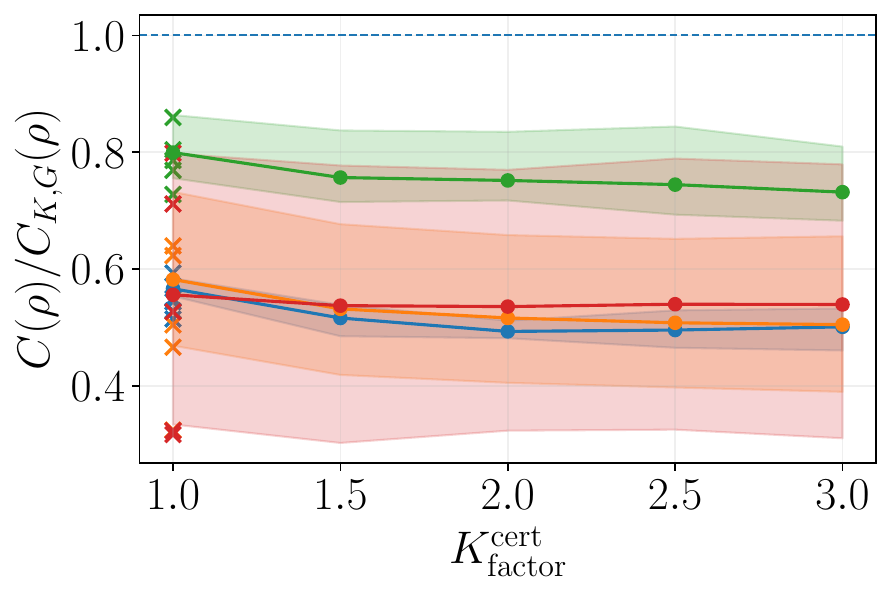}
        \caption{Wave-1D}
    \end{subfigure}
    \begin{subfigure}{0.4\linewidth}
        \includegraphics[width=\linewidth]{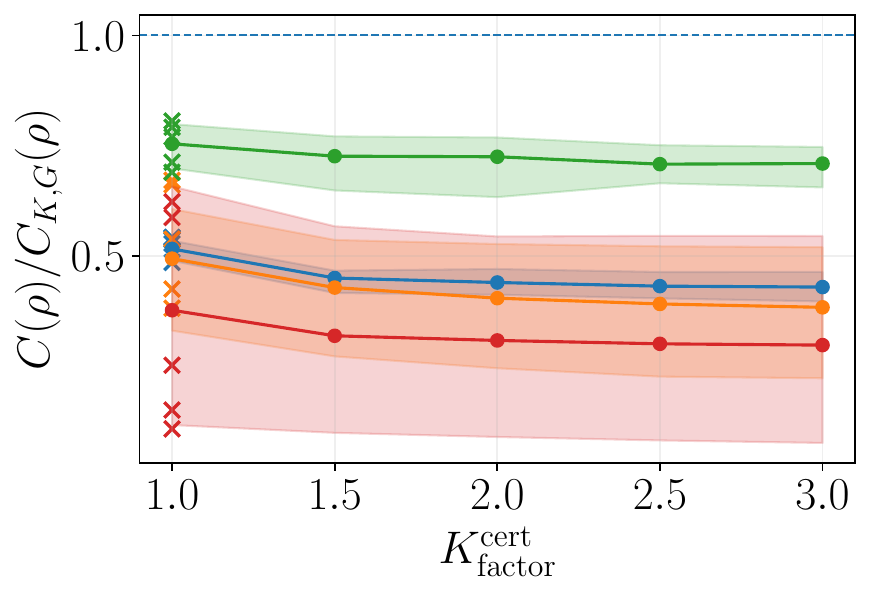}
        \caption{Convection}
    \end{subfigure}
    \begin{subfigure}{0.4\linewidth}
        \includegraphics[width=\linewidth]{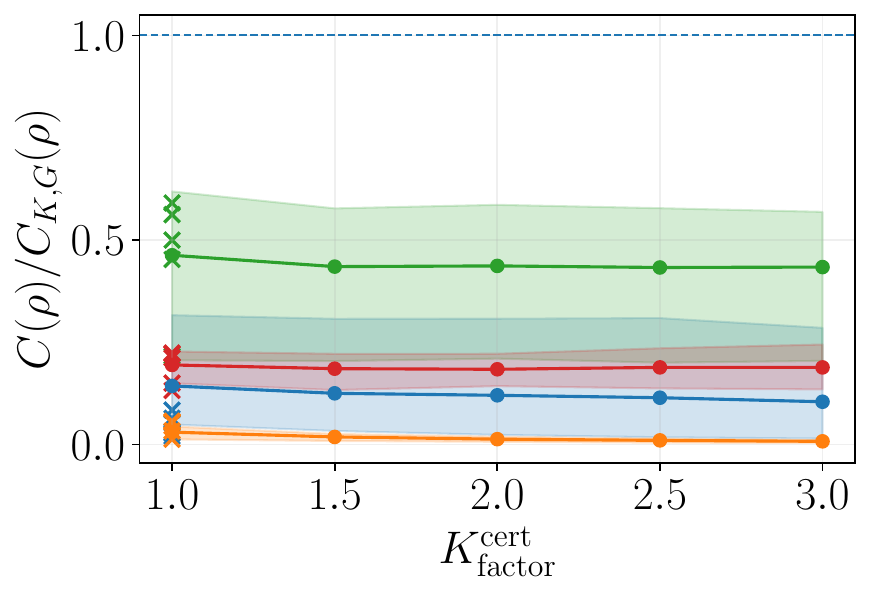}
        \caption{Burgers}
    \end{subfigure}
    \begin{subfigure}{0.4\linewidth}
        \includegraphics[width=\linewidth]{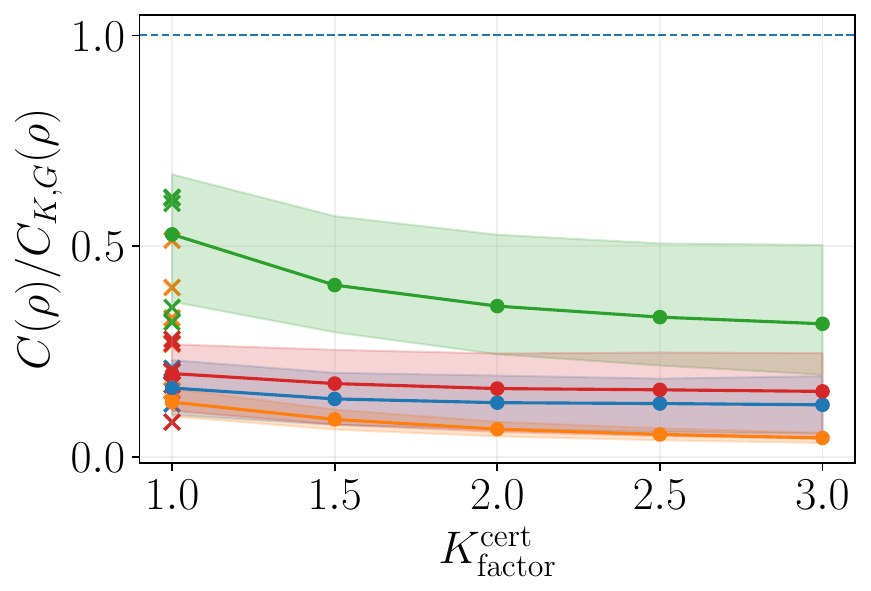}
        \caption{Heat-2D MS}
    \end{subfigure}
    \begin{subfigure}{0.4\linewidth}
        \includegraphics[width=\linewidth]{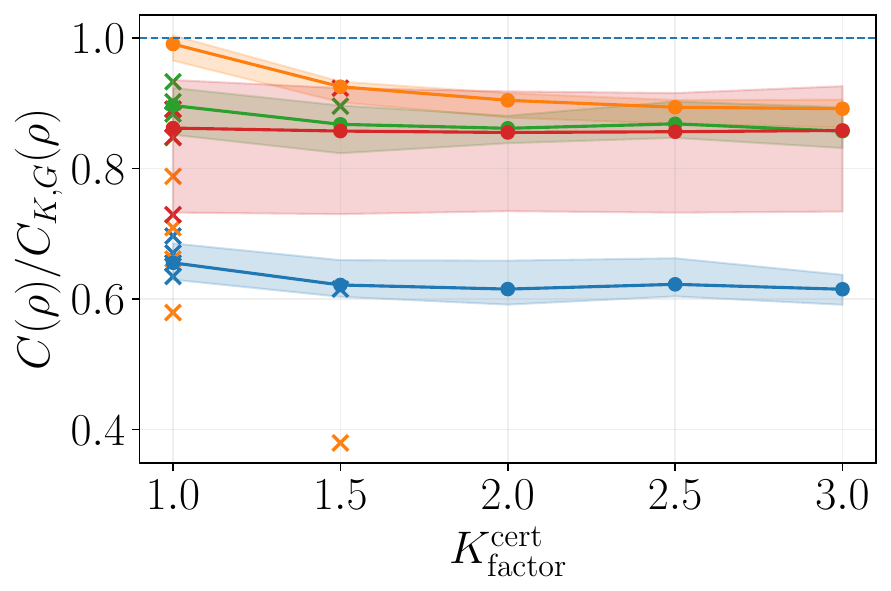}
        \caption{Helmholtz}
    \end{subfigure}
    \begin{subfigure}{0.4\linewidth}
        \includegraphics[width=\linewidth]{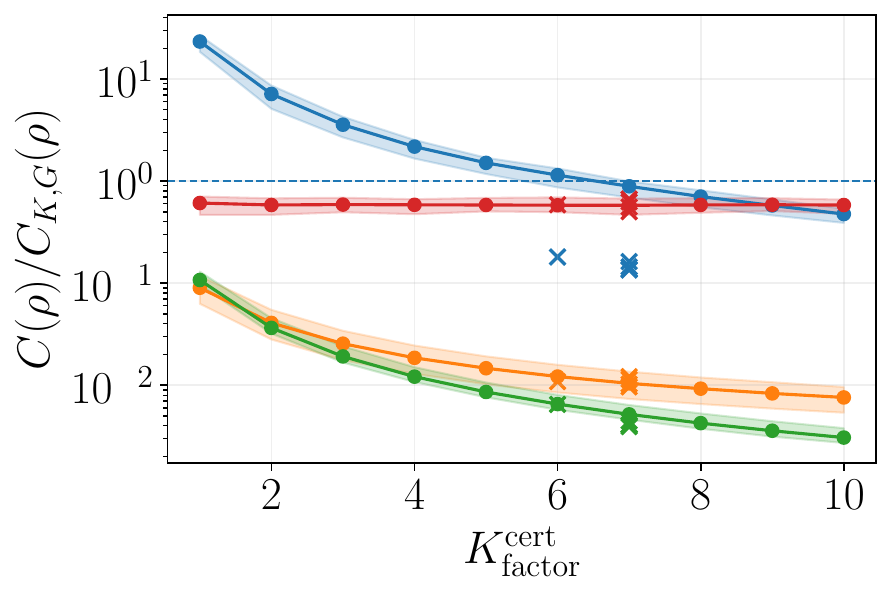}
        \caption{Flow Mixing-2D}
    \end{subfigure}
    \caption{The empirical calibration-dominance ratio $C(\rho)/C_{K,G}(\rho)$ as a function of the $K_{\mathrm{factor}}^{\mathrm{cert}}$ evaluated on $S^{\mathrm{prior}}$, and validated on held-out $S^{\mathrm{calib}}$.}
    \label{fig:K_factor_diagnostic}
\end{figure}

\subsection{Full Main Results} \label{apdx:full_results}

\paragraph{Compared PAC-Bayes bounds.}
We evaluate both certificates developed in Section~\ref{sec:bounds} and compare them with two conventional PAC-Bayes alternatives constructed under the same learned prior and posterior distributions and the same overall confidence level. We denote the component-wise certificate of Theorem~\ref{thm:component_bound} by \emph{ours-union} and its joint multi-task refinement in Theorem~\ref{thm:main} by \emph{ours-joint}. Both proposed bounds rely on the same calibrated input-gradient regularity and control the original, potentially unbounded PIML losses through the empirical input-gradient energies. They both arise from the same theoretical framework, with the latter exploiting the joint multi-task structure to avoid repeatedly paying the PAC-Bayesian complexity across the loss components.

A broad numerical comparison with existing PAC-Bayes bounds is inherently challenging because different results rely on different assumptions on the loss tails, variance, or regularity. Moreover, some bounds involve problem-dependent constants for which a directly computable finite-sample estimation procedure is not specified, so their numerical evaluation would require additional assumptions or researcher-chosen approximations that can materially affect the resulting certificate. We therefore focus on representative baselines whose required quantities can be instantiated under a common experimental protocol, allowing the reported certificates to be compared on equal computational footing. In this respect, our work also contributes to the relatively limited literature on end-to-end computable PAC-Bayes certificates for physics-informed neural networks with potentially unbounded losses: beyond deriving the theoretical bound, we provide an explicit calibration procedure for the problem-dependent quantities entering the final certificate. For comparison, we additionally consider two PAC-Bayes baselines. The first baseline, denoted by \emph{bounded-u}, is a conventional bounded-loss PAC-Bayes certificate adapted to the same sample-centric multi-task construction as our joint bound. For each component, we assume a conservative loss envelope $0\leq\ell_i\leq U_{\ell,i}$. Throughout the experiments, we set $U_{\ell,i}=10^3$ for all components and benchmarks, except for Helmholtz, for which we use $U_{\ell,i}=5\times10^3$ to accommodate its substantially larger PDE losses ($>10^3$). These ranges are chosen conservatively relative to the loss scales observed over the localised hypothesis region, typically exceeding the observed loss values by approximately one order of magnitude. For a jointly treated block $\mathcal I$, we define the common envelope $U_{\mathcal I}:=\max_{i\in\mathcal I}U_{\ell,i}$, such that $\ell_i/U_{\mathcal I}\in[0,1]$ for every $i\in\mathcal I$. Letting $M_{\mathcal I}:=\sum_{i\in\mathcal I}m_i$, the corresponding sample-weighted empirical normalised risk is
\[
\widehat q_{\mathcal I}(\rho)
:=
\frac{1}{M_{\mathcal I}U_{\mathcal I}}
\sum_{i\in\mathcal I}
m_i\widehat{\mathcal R}_i(\rho).
\]
According to \citep{zakerinia2025fast}, applying the binary-$\kl$ PAC-Bayes inequality jointly to the $M_{\mathcal I}$ samples gives, with confidence at least $1-\delta_{\mathcal I}$,
\[
\mathcal R_{\mathcal I}^{\mathcal S}(\rho)
\leq
U_{\mathcal I}
\kl_+^{-1}\!\left(
\widehat q_{\mathcal I}(\rho),
\frac{\KL(\rho\|\pi)+\log\frac{2\sqrt{M_{\mathcal I}}}{\delta_{\mathcal I}}}{M_{\mathcal I}}
\right).
\]
When the components within $\mathcal I$ have equal sample sizes, multiplying this sample-weighted bound by $|\mathcal I|$ recovers their equally weighted population risk. Accordingly, in the balanced regime we apply a single binary-$\kl$ event jointly to all data and physics components. In the data-scarce regime, we bound the observational component separately and apply one joint sample-centric bound to the physical block containing the PDE, IC, and BC losses, with confidence budgets $\delta_d=\delta/4$ and $\delta_{\mathrm{phy}}=3\delta/4$. Thus, \emph{bounded-u} follows the same multi-task aggregation used by our joint construction, but does not exploit input-gradient regularity and instead controls the losses directly through conservative uniform envelopes.

The second baseline, denoted by \emph{sub-gaussian}, similarly follows the sample-centric multi-task construction but retains the calibrated relation between loss concentration and input-gradient regularity. From Corollary~\ref{cor:cgf_grad}, with probability at least $1 - \delta_C$, each component satisfies
\[
\E_{\boldsymbol{\theta}\sim\rho}
[\Lambda_{\boldsymbol{\theta}}^{(i)}(\lambda)]
\leq
\frac{\lambda^2}{2}
\overline C_i(\rho,\delta_C)U_{G,i},
\]
corresponding to the conservative variance proxy
\[
\sigma_i^2(\rho,\delta_C)
:=
\overline C_i(\rho,\delta_C)U_{G,i}.
\]
For a jointly treated block $\mathcal I$, these component-specific proxies are aggregated according to their sample sizes as
\[
\mathcal G_{U,\mathcal I}(\rho,\delta_C)
:=
\sum_{i\in\mathcal I}
m_i\sigma_i^2(\rho,\delta_C)
=
\sum_{i\in\mathcal I}
m_i\overline C_i(\rho,\delta_C)U_{G,i}.
\]
Substituting this quantity into the joint PAC-Bayes-Chernoff argument and optimising over $\lambda>0$ yields the sample-weighted complexity
\[
\frac{1}{M_{\mathcal I}}
\sqrt{
2\,
\mathcal G_{U,\mathcal I}(\rho,\delta_C)
K_{\mathcal I}(\rho,\pi,\delta_G)
}.
\]
As for \emph{bounded-u}, the balanced regime treats all components jointly and rescales the resulting sample-weighted certificate to the equally weighted PIML risk. In the data-scarce regime, the observational loss is treated separately while the physical components are aggregated into a single joint block, using $\delta_d=\delta/4$ and $\delta_{\mathrm{phy}}=3\delta/4$. Within each block, the calibration event receives $\delta_C=\delta_{\mathcal I}/3$, while the remaining $2\delta_{\mathcal I}/3$ is allocated to the PAC-Bayes generalisation event since no empirical-gradient concentration event is required. Hence, \emph{sub-gaussian} shares both the calibrated regularity information and the sample-centric multi-task structure of \emph{ours-joint}, but replaces the posterior-averaged empirical input-gradient energy by its conservative envelope $U_{G,i}$.

Consequently, the four certificates provide a progressively controlled comparison. In particular, \emph{ours-union} is the component-wise certificate of Theorem~\ref{thm:component_bound}, while \emph{ours-joint} is its sample-centric multi-task refinement in Theorem~\ref{thm:main}. The \emph{sub-gaussian} baseline retains this joint sample-centric structure and the calibrated input-gradient regularity, but removes the empirical adaptivity of the gradient complexity. Finally, \emph{bounded-u} retains the joint sample-centric treatment but removes the input-gradient regularity altogether, relying instead on conservative loss envelopes. The resulting comparisons therefore distinguish the gain from the joint refinement, the gain from using empirical rather than worst-case input-gradient information, and the gain from exploiting input-gradient regularity in the first place.

\paragraph{Overall benchmark results.}
Tables~\ref{tab:benchmark_wave1D}--\ref{tab:benchmark_heat2d_ms} report the complete PAC-Bayes results over the six benchmark problems. In the balanced regime, \emph{ours-joint} yields bounds of $6.46\times10^{-1}$, $3.12\times10^{-2}$, $6.01\times10^{-1}$, $6.26\times10^{-3}$, $9.16$, and $3.98\times10^{-2}$ on Wave-1D, Convection, Burgers, Heat-2D MS, Helmholtz, and Flow Mixing, respectively. These are consistently tighter than the bounded-$U$ and sub-Gaussian alternatives, often by one or more orders of magnitude. The improvement is particularly noticeable on Heat-2D MS and Flow Mixing, where the joint certificates are respectively about $71\times$ and $11\times$ tighter than bounded-$U$, and $28\times$ and $36\times$ tighter than the sub-Gaussian baseline. The proposed posteriors also generally attain comparable or lower empirical risks, indicating that the tighter certificates are not obtained by sacrificing predictive performance.

The comparison between \emph{ours-joint} and \emph{ours-union} isolates the benefit of the joint multi-task refinement. The two are nearly identical on Wave-1D, Convection, Burgers, and Helmholtz, while clearer gains appear on Heat-2D MS and Flow Mixing where the observational loss's contribution to the total risk is considerable ($\frac{3.92\times 10^{-4}}{2.05 \times 10^{-3}} \approx 20\%$ for the former and $\frac{1.65 \times 10^{-3}}{3.09 \times 10^{-3}} \approx 53\%$ for the latter). In the balanced regime, the bound decreases from $9.60\times10^{-3}$ to $6.26\times10^{-3}$ on Heat-2D MS and from $5.58\times10^{-2}$ to $3.98\times10^{-2}$ on Flow Mixing, reductions of approximately $35\%$ and $29\%$, respectively. On Heat-2D MS, this tightening is accompanied by a slightly lower test risk ($2.05\times10^{-3}$ versus $2.09\times10^{-3}$).

An important factor governing absolute certificate tightness is the regularity of the underlying solution. Since our complexity depends explicitly on input-gradient energies, problems with sharp or highly varying solution structures naturally incur larger complexity. This is visible for Burgers, whose shock-like region produces a substantially larger gap between empirical risk and certificate than smoother problems such as Wave-1D or Heat-2D MS. Thus, the certificate reflects not only predictive error but also the local regularity of the learned solution.

The data-scarce regime provides a more stringent assessment, restricting the observational component to only $100$ posterior-training and $100$ independent calibration samples. Predictive performance remains remarkably stable: for example, the joint test risk changes from $2.05\times10^{-3}$ to $2.24\times10^{-3}$ on Heat-2D MS and from $3.22\times10^{-3}$ to $3.23\times10^{-3}$ on Burgers. The certificates nevertheless increase to $5.74\times10^{-1}$ and $1.89$, respectively, reflecting the reduced statistical evidence available for both posterior learning and calibration. This separation illustrates that low empirical risk alone does not imply a tight finite-sample certificate.

Despite this deterioration, the proposed certificates remain substantially tighter than the conservative baselines. In the data-scarce regime, \emph{ours-joint} gives $7.77\times10^{-2}$, $1.89$, $5.74\times10^{-1}$, and $9.23$ on Convection, Burgers, Heat-2D MS, and Helmholtz, compared with $65.2$, $65.5$, $65.1$, and $334$ for bounded-$U$, and $6.14\times10^{-1}$, $21.8$, $2.08$, and $178$ for the sub-Gaussian baseline. Flow Mixing is more strongly affected by data scarcity, but its joint bound of $12.1$ remains below both bounded-$U$ ($65.1$) and sub-Gaussian ($41.9$). Thus, the advantage of the input-gradient construction persists even when observational information is severely limited.

Finally, the oracle ``bound w/o tail'' results identify the residual-tail correction as the main source of deterioration under data scarcity. In the balanced regime, removing this contribution reduces the joint bounds to $4.80\times10^{-1}$, $2.13\times10^{-2}$, $3.60\times10^{-2}$, $2.32\times10^{-3}$, $5.34$, and $7.98\times10^{-3}$ across the six problems. Under data scarcity, the corresponding values remain only $4.81\times10^{-1}$, $2.41\times10^{-2}$, $4.50\times10^{-2}$, $4.19\times10^{-3}$, $5.34$, and $7.21\times10^{-2}$, despite much larger full certificates on several problems. This shows that the empirical input-gradient component remains informative with limited data, while the rigorous residual-gradient control becomes the dominant source of conservatism, motivating sharper finite-sample control of this term. The ``bound w/o tail'' therefore also provides an oracle indication of the certificate tightness that could be approached under substantially sharper control of the residual gradient tail.

\begin{table*}[t]
\centering
\small
\caption{PAC-Bayes benchmark results for Wave-1D in the balanced and data-scarce regimes. The balanced regime uses the full available datasets for posterior learning and calibration, whereas the data-scarce regime restricts the observational component to $100$ samples for posterior learning and $100$ independent samples for calibration, while retaining the full physical datasets. For clarity, we highlight the value with \colorbox{blue!20}{blue} for smallest empirical risks and bound values.}
\label{tab:benchmark_wave1D}
\resizebox{0.85\linewidth}{!}{
\begin{tabular}{llcccc}
\toprule
\textbf{Data regime} & \textbf{Quantity} & \textbf{Ours-joint} & \textbf{Ours-union} & \textbf{Bounded-$U$} & \textbf{Sub-Gaussian} \\
\midrule
\multirow{12}{*}{\textbf{Balanced}} & Train risk ($\widehat{\mathcal{R}}_{d}$) & $\mstd{1.51}{0.18}{-4}$ & $\mstd{1.51}{0.18}{-4}$ & $\mstd{1.53}{0.20}{-4}$ & $\mstd{1.55}{0.21}{-4}$ \\
 & Train risk ($\widehat{\mathcal{R}}_{\mathrm{PDE}}$) & $\mstd{4.30}{0.52}{-1}$ & $\mstd{4.30}{0.52}{-1}$ & $\mstd{4.44}{0.53}{-1}$ & $\mstd{4.62}{0.54}{-1}$ \\
 & Train risk ($\widehat{\mathcal{R}}_{\mathrm{IC}}$) & $\mstd{1.19}{0.11}{-4}$ & $\mstd{1.19}{0.11}{-4}$ & $\mstd{1.21}{0.12}{-4}$ & $\mstd{1.23}{0.14}{-4}$ \\
 & Train risk ($\widehat{\mathcal{R}}_{\mathrm{BC}}$) & $\mstd{1.25}{0.12}{-4}$ & $\mstd{1.25}{0.12}{-4}$ & $\mstd{1.27}{0.12}{-4}$ & $\mstd{1.28}{0.13}{-4}$ \\
 & Train risk (total) & {\cellcolor{blue!28} $\mstd{4.31}{0.52}{-1}$} & {\cellcolor{blue!28}$\mstd{4.31}{0.52}{-1}$} & $\mstd{4.44}{0.53}{-1}$ & $\mstd{4.62}{0.54}{-1}$ \\
\cmidrule{2-6}
 & Test risk ($\widehat{\mathcal{R}}_{d}$) & $\mstd{1.51}{0.18}{-4}$ & $\mstd{1.51}{0.18}{-4}$ & $\mstd{1.53}{0.20}{-4}$ & $\mstd{1.55}{0.21}{-4}$ \\
 & Test risk ($\widehat{\mathcal{R}}_{\mathrm{PDE}}$) & $\mstd{4.29}{0.52}{-1}$ & $\mstd{4.29}{0.52}{-1}$ & $\mstd{4.43}{0.53}{-1}$ & $\mstd{4.61}{0.54}{-1}$ \\
 & Test risk ($\widehat{\mathcal{R}}_{\mathrm{IC}}$) & $\mstd{1.19}{0.11}{-4}$ & $\mstd{1.19}{0.11}{-4}$ & $\mstd{1.21}{0.12}{-4}$ & $\mstd{1.23}{0.14}{-4}$  \\
 & Test risk ($\widehat{\mathcal{R}}_{\mathrm{BC}}$) & $\mstd{1.25}{0.11}{-4}$ & $\mstd{1.25}{0.11}{-4}$ & $\mstd{1.27}{0.12}{-4}$ & $\mstd{1.28}{0.13}{-4}$ \\
 & Test risk (total) & {\cellcolor{blue!28}$\mstd{4.30}{0.52}{-1}$} & {\cellcolor{blue!28}$\mstd{4.30}{0.52}{-1}$} & $\mstd{4.43}{0.53}{-1}$ & $\mstd{4.61}{0.54}{-1}$ \\
\cmidrule{2-6}
 & $\mathrm{KL}(\rho\Vert\pi)$ & $\mstd{8.59}{1.67}{-1}$ & $\mstd{8.58}{1.67}{-1}$ & $\mstd{2.31}{0.81}{-1}$ & $\mstd{1.50}{0.57}{-2}$ \\
 & Bound & {\cellcolor{blue!28}$\mstd{6.46}{0.75}{-1}$} & {\cellcolor{blue!28}$\mstd{6.46}{0.75}{-1}$} & $\mstd{1.40}{0.09}{0}$ & $\mstd{8.05}{0.73}{0}$ \\
 & Bound w/o tail & $\mstd{4.80}{0.63}{-1}$ & $\mstd{4.80}{0.63}{-1}$ & - & - \\
\midrule
\multirow{12}{*}{\textbf{Data-scarce}} & Train risk ($\widehat{\mathcal{R}}_{d}$) & $\mstd{1.58}{0.16}{-4}$ & $\mstd{1.58}{0.16}{-4}$ & $\mstd{1.61}{0.17}{-4}$ & $\mstd{1.61}{0.17}{-4}$ \\
& Train risk ($\widehat{\mathcal{R}}_{\mathrm{PDE}}$) & $\mstd{4.31}{0.52}{-1}$ & $\mstd{4.31}{0.52}{-1}$ & $\mstd{4.69}{0.54}{-1}$ & $\mstd{4.62}{0.54}{-1}$ \\
& Train risk ($\widehat{\mathcal{R}}_{\mathrm{IC}}$) & $\mstd{1.20}{0.11}{-4}$ & $\mstd{1.20}{0.11}{-4}$ & $\mstd{1.23}{0.14}{-4}$ & $\mstd{1.23}{0.14}{-4}$ \\
& Train risk ($\widehat{\mathcal{R}}_{\mathrm{BC}}$) & $\mstd{1.25}{0.12}{-4}$ & $\mstd{1.25}{0.12}{-4}$ & $\mstd{1.29}{0.13}{-4}$ & $\mstd{1.28}{0.13}{-4}$ \\
& Train risk (total) & {\cellcolor{blue!28}$\mstd{4.31}{0.52}{-1}$} & {\cellcolor{blue!28}$\mstd{4.31}{0.52}{-1}$} & $\mstd{4.69}{0.54}{-1}$ & $\mstd{4.62}{0.54}{-1}$  \\
\cmidrule{2-6}
& Test risk ($\widehat{\mathcal{R}}_{d}$) & $\mstd{1.51}{0.19}{-4}$ & $\mstd{1.51}{0.19}{-4}$ & $\mstd{1.55}{0.21}{-4}$ & $\mstd{1.55}{0.21}{-4}$ \\
& Test risk ($\widehat{\mathcal{R}}_{\mathrm{PDE}}$) & $\mstd{4.30}{0.52}{-1}$ & $\mstd{4.30}{0.52}{-1}$ & $\mstd{4.68}{0.54}{-1}$ & $\mstd{4.61}{0.54}{-1}$ \\
& Test risk ($\widehat{\mathcal{R}}_{\mathrm{IC}}$) &  $\mstd{1.20}{0.11}{-4}$ & $\mstd{1.20}{0.11}{-4}$ & $\mstd{1.24}{0.14}{-4}$ & $\mstd{1.23}{0.14}{-4}$ \\
& Test risk ($\widehat{\mathcal{R}}_{\mathrm{BC}}$) & $\mstd{1.25}{0.11}{-4}$ & $\mstd{1.25}{0.11}{-4}$ & $\mstd{1.28}{0.13}{-4}$ & $\mstd{1.28}{0.13}{-4}$ \\
& Test risk (total) & {\cellcolor{blue!28}$\mstd{4.31}{0.52}{-1}$} & {\cellcolor{blue!28}$\mstd{4.31}{0.52}{-1}$} & $\mstd{4.68}{0.54}{-1}$ & $\mstd{4.61}{0.54}{-1}$ \\
\cmidrule{2-6}
& $\mathrm{KL}(\rho\Vert\pi)$ & $\mstd{7.94}{1.44}{-1}$ & $\mstd{7.94}{1.44}{-1}$ & $\mstd{1.09}{0.53}{-4}$ & $\mstd{1.45}{0.54}{-2}$  \\
& Bound & {\cellcolor{blue!28}$\mstd{6.57}{0.77}{-1}$} & {\cellcolor{blue!28}$\mstd{6.57}{0.77}{-1}$} & $\mstd{6.61}{0.01}{1}$ & $\mstd{8.08}{0.74}{0}$ \\
& Bound w/o tail & $\mstd{4.81}{0.63}{-1}$ & $\mstd{4.81}{0.63}{-1}$ & - & - \\
\bottomrule
\end{tabular}
}
\end{table*}

\begin{table*}[ht]
\centering
\small
\caption{PAC-Bayes benchmark results for Convection in the balanced and data-scarce regimes.}
\label{tab:benchmark_convection}
\resizebox{0.85\linewidth}{!}{
\begin{tabular}{llcccc}
\toprule
\textbf{Data regime} & \textbf{Quantity} & \textbf{Ours-joint} & \textbf{Ours-union} & \textbf{Bounded-$U$} & \textbf{Sub-Gaussian} \\
\midrule
\multirow{12}{*}{\textbf{Balanced}} & Train risk ($\widehat{\mathcal{R}}_{d}$) & $\mstd{4.90}{0.62}{-4}$ & $\mstd{4.91}{0.63}{-4}$ & $\mstd{5.30}{0.92}{-4}$ & $\mstd{5.30}{0.92}{-4}$ \\
& Train risk ($\widehat{\mathcal{R}}_{\mathrm{PDE}}$) & $\mstd{1.77}{0.21}{-2}$ & $\mstd{1.77}{0.21}{-2}$ & $\mstd{1.96}{0.32}{-2}$ & $\mstd{1.96}{0.33}{-2}$ \\
& Train risk ($\widehat{\mathcal{R}}_{\mathrm{IC}}$) & $\mstd{4.23}{0.96}{-4}$ & $\mstd{4.24}{0.97}{-4}$ & $\mstd{4.69}{1.54}{-4}$ & $\mstd{4.70}{1.54}{-4}$ \\
& Train risk ($\widehat{\mathcal{R}}_{\mathrm{BC}}$) & $\mstd{3.89}{0.58}{-4}$ & $\mstd{3.89}{0.58}{-4}$ & $\mstd{4.08}{0.77}{-4}$ & $\mstd{4.08}{0.77}{-4}$ \\
& Train risk (total) & {\cellcolor{blue!28}$\mstd{1.90}{0.23}{-2}$} & $\mstd{1.91}{0.23}{-2}$ & $\mstd{2.10}{0.34}{-2}$ & $\mstd{2.10}{0.35}{-2}$ \\
\cmidrule{2-6}
& Test risk ($\widehat{\mathcal{R}}_{d}$) & $\mstd{4.92}{0.62}{-4}$ & $\mstd{4.93}{0.63}{-4}$ & $\mstd{5.32}{0.91}{-4}$ & $\mstd{5.32}{0.92}{-4}$ \\
& Test risk ($\widehat{\mathcal{R}}_{\mathrm{PDE}}$) & $\mstd{1.78}{0.22}{-2}$ & $\mstd{1.79}{0.22}{-2}$ & $\mstd{1.97}{0.33}{-2}$ & $\mstd{1.98}{0.33}{-2}$ \\
& Test risk ($\widehat{\mathcal{R}}_{\mathrm{IC}}$) & $\mstd{4.24}{0.96}{-4}$ & $\mstd{4.25}{0.97}{-4}$ & $\mstd{4.70}{1.53}{-4}$ & $\mstd{4.70}{1.54}{-4}$ \\
& Test risk ($\widehat{\mathcal{R}}_{\mathrm{BC}}$) & $\mstd{3.87}{0.58}{-4}$ & $\mstd{3.87}{0.58}{-4}$ & $\mstd{4.07}{0.76}{-4}$ & $\mstd{4.07}{0.76}{-4}$ \\
& Test risk (total) & {\cellcolor{blue!28}$\mstd{1.91}{0.23}{-2}$} & $\mstd{1.92}{0.23}{-2}$ & $\mstd{2.11}{0.35}{-2}$ & $\mstd{2.12}{0.35}{-2}$ \\
\cmidrule{2-6}
& $\mathrm{KL}(\rho\Vert\pi)$ & $\mstd{8.79}{5.92}{-1}$ & $\mstd{8.39}{5.63}{-1}$ & $\mstd{1.51}{0.97}{-2}$ & $\mstd{1.45}{0.83}{-2}$ \\
& Bound & {\cellcolor{blue!28}$\mstd{3.12}{0.43}{-2}$} & $\mstd{3.18}{0.44}{-2}$ & $\mstd{5.21}{0.12}{-1}$ & $\mstd{4.62}{0.55}{-1}$  \\
& Bound w/o tail & $\mstd{2.13}{0.31}{-2}$ & $\mstd{2.15}{0.32}{-2}$ & - & - \\
\midrule
\multirow{12}{*}{\textbf{Data-scarce}} & Train risk ($\widehat{\mathcal{R}}_{d}$) & $\mstd{5.34}{0.85}{-4}$ & $\mstd{5.34}{0.85}{-4}$ & $\mstd{5.45}{0.93}{-4}$ & $\mstd{5.42}{0.91}{-4}$ \\
& Train risk ($\widehat{\mathcal{R}}_{\mathrm{PDE}}$) & $\mstd{1.90}{0.28}{-2}$ & $\mstd{1.90}{0.28}{-2}$ & $\mstd{2.00}{0.35}{-2}$ & $\mstd{1.97}{0.33}{-2}$ \\
& Train risk ($\widehat{\mathcal{R}}_{\mathrm{IC}}$) & $\mstd{4.59}{1.38}{-4}$ & $\mstd{4.59}{1.38}{-4}$ & $\mstd{4.78}{1.68}{-4}$ & $\mstd{4.72}{1.58}{-4}$ \\
& Train risk ($\widehat{\mathcal{R}}_{\mathrm{BC}}$) & $\mstd{4.04}{0.72}{-4}$ & $\mstd{4.04}{0.72}{-4}$ & $\mstd{4.12}{0.80}{-4}$ & $\mstd{4.09}{0.78}{-4}$ \\
& Train risk (total) & {\cellcolor{blue!28}$\mstd{2.04}{0.30}{-2}$} & {\cellcolor{blue!28}$\mstd{2.04}{0.30}{-2}$} & $\mstd{2.15}{0.38}{-2}$ & $\mstd{2.12}{0.36}{-2}$ \\
\cmidrule{2-6}
& Test risk ($\widehat{\mathcal{R}}_{d}$) & $\mstd{5.25}{0.85}{-4}$ & $\mstd{5.25}{0.85}{-4}$ & $\mstd{5.37}{0.95}{-4}$ & $\mstd{5.33}{0.93}{-4}$ \\
& Test risk ($\widehat{\mathcal{R}}_{\mathrm{PDE}}$) & $\mstd{1.92}{0.29}{-2}$ & $\mstd{1.92}{0.29}{-2}$ & $\mstd{2.02}{0.36}{-2}$ & $\mstd{1.99}{0.34}{-2}$ \\
& Test risk ($\widehat{\mathcal{R}}_{\mathrm{IC}}$) & $\mstd{4.59}{1.38}{-4}$ & $\mstd{4.59}{1.38}{-4}$ & $\mstd{4.79}{1.68}{-4}$ & $\mstd{4.73}{1.58}{-4}$ \\
& Test risk ($\widehat{\mathcal{R}}_{\mathrm{BC}}$) & $\mstd{4.02}{0.71}{-4}$ & $\mstd{4.02}{0.71}{-4}$ & $\mstd{4.10}{0.80}{-4}$ & $\mstd{4.08}{0.77}{-4}$ \\
& Test risk (total) & {\cellcolor{blue!28}$\mstd{2.06}{0.31}{-2}$} & {\cellcolor{blue!28}$\mstd{2.06}{0.31}{-2}$} & $\mstd{2.16}{0.38}{-2}$ & $\mstd{2.13}{0.36}{-2}$ \\
\cmidrule{2-6}
& $\mathrm{KL}(\rho\Vert\pi)$ & $\mstd{8.81}{6.35}{-2}$ & $\mstd{8.74}{6.30}{-2}$ & $\mstd{6.57}{0.92}{-6}$ & $\mstd{6.92}{4.03}{-3}$ \\
& Bound & {\cellcolor{blue!28}$\mstd{7.77}{0.72}{-2}$} & $\mstd{7.81}{0.73}{-2}$ & $\mstd{6.52}{0.00}{1}$ & $\mstd{6.14}{0.60}{-1}$ \\
& Bound w/o tail & $\mstd{2.41}{0.40}{-2}$ & $\mstd{2.42}{0.41}{-2}$ & - & - \\
\bottomrule
\end{tabular}
}
\end{table*}

\begin{table*}[h]
\centering
\small
\caption{PAC-Bayes benchmark results for Burgers in the balanced and data-scarce regimes.}
\label{tab:benchmark_burgers}
\resizebox{0.85\linewidth}{!}{
\begin{tabular}{llcccc}
\toprule
\textbf{Data regime} & \textbf{Quantity} & \textbf{Ours-joint} & \textbf{Ours-union} & \textbf{Bounded-$U$} & \textbf{Sub-Gaussian} \\
\midrule
\multirow{12}{*}{\textbf{Balanced}} & Train risk ($\widehat{\mathcal{R}}_{d}$) & $\mstd{9.11}{3.71}{-5}$ & $\mstd{9.11}{3.71}{-5}$ & $\mstd{9.12}{3.91}{-5}$ & $\mstd{9.10}{3.70}{-5}$ \\
& Train risk ($\widehat{\mathcal{R}}_{\mathrm{PDE}}$) & $\mstd{3.21}{1.55}{-3}$ & $\mstd{3.21}{1.55}{-3}$ & $\mstd{3.23}{1.57}{-3}$ & $\mstd{3.23}{1.57}{-3}$ \\
& Train risk ($\widehat{\mathcal{R}}_{\mathrm{IC}}$) & $\mstd{4.67}{0.54}{-5}$ & $\mstd{4.67}{0.54}{-5}$ & $\mstd{4.58}{0.57}{-5}$ & $\mstd{4.67}{0.54}{-5}$ \\
& Train risk ($\widehat{\mathcal{R}}_{\mathrm{BC}}$) & $\mstd{4.87}{0.94}{-5}$ & $\mstd{4.87}{0.94}{-5}$ & $\mstd{4.89}{0.79}{-5}$ & $\mstd{4.88}{0.94}{-5}$ \\
& Train risk (total) & {\cellcolor{blue!28}$\mstd{3.40}{1.57}{-3}$} & {\cellcolor{blue!28}$\mstd{3.40}{1.57}{-3}$} & $\mstd{3.41}{1.58}{-3}$ & $\mstd{3.42}{1.59}{-3}$ \\
\cmidrule{2-6}
& Test risk ($\widehat{\mathcal{R}}_{d}$) & $\mstd{9.59}{4.53}{-5}$ & $\mstd{9.59}{4.53}{-5}$ & $\mstd{9.60}{4.75}{-5}$ & $\mstd{9.58}{4.53}{-5}$ \\
& Test risk ($\widehat{\mathcal{R}}_{\mathrm{PDE}}$) & $\mstd{3.03}{0.99}{-3}$ & $\mstd{3.03}{0.99}{-3}$ & $\mstd{3.05}{0.99}{-3}$ & $\mstd{3.05}{1.00}{-3}$ \\
& Test risk ($\widehat{\mathcal{R}}_{\mathrm{IC}}$) & $\mstd{4.66}{0.54}{-5}$ & $\mstd{4.66}{0.54}{-5}$ & $\mstd{4.58}{0.57}{-5}$ & $\mstd{4.66}{0.54}{-5}$ \\
& Test risk ($\widehat{\mathcal{R}}_{\mathrm{BC}}$) & $\mstd{4.87}{0.94}{-5}$ & $\mstd{4.87}{0.94}{-5}$ & $\mstd{4.89}{0.79}{-5}$ & $\mstd{4.88}{0.94}{-5}$  \\
& Test risk (total) & {\cellcolor{blue!28}$\mstd{3.22}{1.01}{-3}$} & {\cellcolor{blue!28}$\mstd{3.22}{1.01}{-3}$} & $\mstd{3.24}{1.02}{-3}$ & $\mstd{3.24}{1.02}{-3}$ \\
\cmidrule{2-6}
& $\mathrm{KL}(\rho\Vert\pi)$ & $\mstd{1.27}{1.43}{-3}$ & $\mstd{1.23}{1.39}{-3}$ & $\mstd{4.49}{2.69}{-4}$ & $\mstd{3.48}{1.29}{-6}$ \\
& Bound & {\cellcolor{blue!28}$\mstd{6.01}{0.44}{-1}$} & $\mstd{6.14}{0.44}{-1}$ & $\mstd{8.51}{0.09}{-1}$ & $\mstd{1.75}{0.14}{1}$ \\
& Bound w/o tail & $\mstd{3.60}{2.37}{-2}$ & $\mstd{3.62}{2.38}{-2}$ & - & - \\
\midrule
\multirow{12}{*}{\textbf{Data-scarce}} & Train risk ($\widehat{\mathcal{R}}_{d}$) &  $\mstd{2.08}{1.45}{-4}$ & $\mstd{2.08}{1.45}{-4}$ & $\mstd{2.08}{1.47}{-4}$ & $\mstd{2.08}{1.45}{-4}$ \\
& Train risk ($\widehat{\mathcal{R}}_{\mathrm{PDE}}$) & $\mstd{3.23}{1.56}{-3}$ & $\mstd{3.23}{1.56}{-3}$ & $\mstd{3.24}{1.57}{-3}$ & $\mstd{3.23}{1.57}{-3}$ \\
& Train risk ($\widehat{\mathcal{R}}_{\mathrm{IC}}$) & $\mstd{4.66}{0.54}{-5}$ & $\mstd{4.66}{0.54}{-5}$ & $\mstd{4.58}{0.57}{-5}$ & $\mstd{4.67}{0.54}{-5}$ \\
& Train risk ($\widehat{\mathcal{R}}_{\mathrm{BC}}$) & $\mstd{4.88}{0.94}{-5}$ & $\mstd{4.88}{0.94}{-5}$ & $\mstd{4.89}{0.78}{-5}$ & $\mstd{4.88}{0.94}{-5}$ \\
& Train risk (total) & {\cellcolor{blue!28}$\mstd{3.53}{1.69}{-3}$} & {\cellcolor{blue!28}$\mstd{3.53}{1.69}{-3}$} & $\mstd{3.55}{1.70}{-3}$ & $\mstd{3.54}{1.70}{-3}$ \\
\cmidrule{2-6}
& Test risk ($\widehat{\mathcal{R}}_{d}$) & $\mstd{9.58}{4.53}{-5}$ & $\mstd{9.58}{4.53}{-5}$ & $\mstd{9.60}{4.75}{-5}$ & $\mstd{9.58}{4.53}{-5}$ \\
& Test risk ($\widehat{\mathcal{R}}_{\mathrm{PDE}}$) & $\mstd{3.04}{0.99}{-3}$ & $\mstd{3.04}{0.99}{-3}$ & $\mstd{3.06}{1.00}{-3}$ & $\mstd{3.05}{1.00}{-3}$ \\
& Test risk ($\widehat{\mathcal{R}}_{\mathrm{IC}}$) & $\mstd{4.66}{0.54}{-5}$ & $\mstd{4.66}{0.54}{-5}$ & $\mstd{4.58}{0.58}{-5}$ & $\mstd{4.66}{0.54}{-5}$ \\
& Test risk ($\widehat{\mathcal{R}}_{\mathrm{BC}}$) & $\mstd{4.88}{0.94}{-5}$ & $\mstd{4.88}{0.94}{-5}$ & $\mstd{4.89}{0.78}{-5}$ & $\mstd{4.88}{0.94}{-5}$ \\
& Test risk (total) & {\cellcolor{blue!28}$\mstd{3.23}{1.02}{-3}$} & {\cellcolor{blue!28}$\mstd{3.23}{1.02}{-3}$} & $\mstd{3.25}{1.03}{-3}$ & $\mstd{3.24}{1.02}{-3}$ \\
\cmidrule{2-6}
& $\mathrm{KL}(\rho\Vert\pi)$ & $\mstd{1.45}{1.84}{-4}$ & $\mstd{1.44}{1.83}{-4}$ & $\mstd{1.98}{0.35}{-6}$ & $\mstd{2.15}{0.49}{-6}$ \\
& Bound & {\cellcolor{blue!28}$\mstd{1.89}{0.58}{0}$} & $\mstd{1.90}{0.58}{0}$ & $\mstd{6.55}{0.00}{1}$ & $\mstd{2.18}{0.26}{1}$ \\
& Bound w/o tail & $\mstd{4.50}{2.73}{-2}$ & $\mstd{4.50}{2.73}{-2}$ & - & - \\
\bottomrule
\end{tabular}
}
\end{table*}

\begin{table*}[h]
\centering
\small
\caption{PAC-Bayes benchmark results for Heat-2D MS in the balanced and data-scarce regimes.}
\label{tab:benchmark_heat2d_ms}
\resizebox{0.85\linewidth}{!}{
\begin{tabular}{llcccc}
\toprule
\textbf{Data regime} & \textbf{Quantity} & \textbf{Ours-joint} & \textbf{Ours-union} & \textbf{Bounded-$U$} & \textbf{Sub-Gaussian} \\
\midrule
\multirow{12}{*}{\textbf{Balanced}} & Train risk ($\widehat{\mathcal{R}}_{d}$) & $\mstd{3.92}{0.96}{-4}$ & $\mstd{3.98}{0.98}{-4}$ & $\mstd{4.13}{1.04}{-4}$ & $\mstd{4.12}{1.04}{-4}$ \\
& Train risk ($\widehat{\mathcal{R}}_{\mathrm{PDE}}$) & $\mstd{3.69}{0.84}{-4}$ & $\mstd{3.73}{0.83}{-4}$ & $\mstd{3.81}{0.83}{-4}$ & $\mstd{3.81}{0.83}{-4}$ \\
& Train risk ($\widehat{\mathcal{R}}_{\mathrm{IC}}$) & $\mstd{9.33}{3.55}{-4}$ & $\mstd{9.64}{3.80}{-4}$ & $\mstd{1.05}{0.47}{-3}$ & $\mstd{1.05}{0.47}{-3}$ \\
& Train risk ($\widehat{\mathcal{R}}_{\mathrm{BC}}$) & $\mstd{3.52}{0.96}{-4}$ & $\mstd{3.58}{0.97}{-4}$ & $\mstd{3.70}{1.02}{-4}$ & $\mstd{3.69}{1.01}{-4}$ \\
& Train risk (total) & {\cellcolor{blue!28}$\mstd{2.05}{0.55}{-3}$} & $\mstd{2.09}{0.58}{-3}$ & $\mstd{2.22}{0.67}{-3}$ & $\mstd{2.21}{0.67}{-3}$ \\
\cmidrule{2-6}
& Test risk ($\widehat{\mathcal{R}}_{d}$) & $\mstd{3.90}{0.96}{-4}$ & $\mstd{3.96}{0.98}{-4}$ & $\mstd{4.10}{1.04}{-4}$ & $\mstd{4.09}{1.04}{-4}$ \\
& Test risk ($\widehat{\mathcal{R}}_{\mathrm{PDE}}$) & $\mstd{3.76}{0.66}{-4}$ & $\mstd{3.80}{0.68}{-4}$ & $\mstd{3.88}{0.74}{-4}$ & $\mstd{3.88}{0.74}{-4}$ \\
& Test risk ($\widehat{\mathcal{R}}_{\mathrm{IC}}$) & $\mstd{9.41}{3.59}{-4}$ & $\mstd{9.72}{3.84}{-4}$ & $\mstd{1.06}{0.47}{-3}$ & $\mstd{1.06}{0.47}{-3}$ \\
& Test risk ($\widehat{\mathcal{R}}_{\mathrm{BC}}$) & $\mstd{3.48}{0.95}{-4}$ & $\mstd{3.53}{0.96}{-4}$ & $\mstd{3.66}{1.01}{-4}$ & $\mstd{3.65}{1.01}{-4}$ \\
& Test risk (total) & {\cellcolor{blue!28}$\mstd{2.05}{0.54}{-3}$} & $\mstd{2.10}{0.57}{-3}$ & $\mstd{2.23}{0.66}{-3}$ & $\mstd{2.22}{0.66}{-3}$ \\
\cmidrule{2-6}
& $\mathrm{KL}(\rho\Vert\pi)$ & $\mstd{1.91}{1.13}{-1}$ & $\mstd{9.07}{5.86}{-2}$ & $\mstd{8.87}{8.11}{-4}$ & $\mstd{1.58}{1.02}{-3}$ \\
& Bound & {\cellcolor{blue!28}$\mstd{6.26}{1.67}{-3}$} & $\mstd{9.60}{2.33}{-3}$ & $\mstd{4.46}{0.04}{-1}$ & $\mstd{1.76}{0.49}{-1}$ \\
& Bound w/o tail & $\mstd{2.32}{0.60}{-3}$ & $\mstd{2.56}{0.70}{-3}$ & - & - \\
\midrule
\multirow{12}{*}{\textbf{Data-scarce}} & Train risk ($\widehat{\mathcal{R}}_{d}$) & $\mstd{3.77}{0.85}{-4}$ & $\mstd{3.77}{0.85}{-4}$ & $\mstd{3.77}{0.85}{-4}$ & $\mstd{3.77}{0.85}{-4}$ \\
& Train risk ($\widehat{\mathcal{R}}_{\mathrm{PDE}}$) & $\mstd{3.82}{0.83}{-4}$ & $\mstd{3.82}{0.83}{-4}$ & $\mstd{3.82}{0.83}{-4}$ & $\mstd{3.82}{0.83}{-4}$ \\
& Train risk ($\widehat{\mathcal{R}}_{\mathrm{IC}}$) & $\mstd{1.06}{0.48}{-3}$ & $\mstd{1.06}{0.48}{-3}$ & $\mstd{1.07}{0.48}{-3}$ & $\mstd{1.07}{0.48}{-3}$ \\
& Train risk ($\widehat{\mathcal{R}}_{\mathrm{BC}}$) & $\mstd{3.71}{1.02}{-4}$ & $\mstd{3.71}{1.02}{-4}$ & $\mstd{3.71}{1.02}{-4}$ & $\mstd{3.71}{1.02}{-4}$ \\
& Train risk (total) & $\mstd{2.19}{0.66}{-3}$ & $\mstd{2.19}{0.66}{-3}$ & $\mstd{2.20}{0.67}{-3}$ & $\mstd{2.19}{0.67}{-3}$ \\
\cmidrule{2-6}
& Test risk ($\widehat{\mathcal{R}}_{d}$) & $\mstd{4.11}{1.05}{-4}$ & $\mstd{4.11}{1.05}{-4}$ & $\mstd{4.12}{1.05}{-4}$ & $\mstd{4.11}{1.05}{-4}$ \\
& Test risk ($\widehat{\mathcal{R}}_{\mathrm{PDE}}$) & $\mstd{3.89}{0.75}{-4}$ & $\mstd{3.89}{0.75}{-4}$ & $\mstd{3.89}{0.75}{-4}$ & $\mstd{3.89}{0.75}{-4}$ \\
& Test risk ($\widehat{\mathcal{R}}_{\mathrm{IC}}$) & $\mstd{1.07}{0.48}{-3}$ & $\mstd{1.07}{0.48}{-3}$ & $\mstd{1.07}{0.49}{-3}$ & $\mstd{1.07}{0.48}{-3}$ \\
& Test risk ($\widehat{\mathcal{R}}_{\mathrm{BC}}$) & $\mstd{3.67}{1.02}{-4}$ & $\mstd{3.67}{1.02}{-4}$ & $\mstd{3.67}{1.02}{-4}$ & $\mstd{3.67}{1.02}{-4}$ \\
& Test risk (total) & $\mstd{2.24}{0.68}{-3}$ & $\mstd{2.24}{0.68}{-3}$ & $\mstd{2.24}{0.68}{-3}$ & $\mstd{2.24}{0.68}{-3}$ \\
\cmidrule{2-6}
& $\mathrm{KL}(\rho\Vert\pi)$ & $\mstd{3.85}{2.73}{-5}$ & $\mstd{3.83}{2.67}{-5}$ & $\mstd{3.09}{0.34}{-6}$ & $\mstd{1.65}{1.07}{-5}$ \\
& Bound & {\cellcolor{blue!28}$\mstd{5.74}{0.78}{-1}$} & $\mstd{5.76}{0.78}{-1}$ & $\mstd{6.51}{0.00}{1}$ & $\mstd{2.08}{0.29}{0}$ \\
& Bound w/o tail & $\mstd{4.19}{1.17}{-3}$ & $\mstd{4.32}{1.23}{-3}$ & - & - \\
\bottomrule
\end{tabular}
}
\end{table*}

\begin{table*}[h]
\centering
\small
\caption{PAC-Bayes benchmark results for Helmholtz in the balanced and data-scarce regimes.}
\label{tab:benchmark_helmholtz3d_ms}
\resizebox{0.85\linewidth}{!}{
\begin{tabular}{llcccc}
\toprule
\textbf{Data regime} & \textbf{Quantity} & \textbf{Ours-joint} & \textbf{Ours-union} & \textbf{Bounded-$U$} & \textbf{Sub-Gaussian} \\
\midrule
\multirow{12}{*}{\textbf{Balanced}} & Train risk ($\widehat{\mathcal{R}}_{d}$) & $\mstd{3.65}{0.30}{-4}$ & $\mstd{3.65}{0.30}{-4}$ & $\mstd{3.65}{0.30}{-4}$ & $\mstd{3.67}{0.31}{-4}$ \\
& Train risk ($\widehat{\mathcal{R}}_{\mathrm{PDE}}$) & $\mstd{4.66}{0.48}{0}$ & $\mstd{4.66}{0.48}{0}$ & $\mstd{4.66}{0.48}{0}$ & $\mstd{4.77}{0.48}{0}$ \\
& Train risk ($\widehat{\mathcal{R}}_{\mathrm{IC}}$) & $\mstd{3.64}{0.29}{-4}$ & $\mstd{3.64}{0.29}{-4}$ & $\mstd{3.64}{0.29}{-4}$ & $\mstd{3.65}{0.30}{-4}$ \\
& Train risk ($\widehat{\mathcal{R}}_{\mathrm{BC}}$) & $\mstd{4.29}{0.38}{-4}$ & $\mstd{4.29}{0.38}{-4}$ & $\mstd{4.29}{0.37}{-4}$ & $\mstd{4.31}{0.38}{-4}$ \\
& Train risk (total) & $\mstd{4.66}{0.48}{0}$ & $\mstd{4.66}{0.48}{0}$ & $\mstd{4.66}{0.48}{0}$ & $\mstd{4.77}{0.48}{0}$ \\
\cmidrule{2-6}
& Test risk ($\widehat{\mathcal{R}}_{d}$) & $\mstd{3.65}{0.30}{-4}$ & $\mstd{3.65}{0.30}{-4}$ & $\mstd{3.65}{0.30}{-4}$ & $\mstd{3.67}{0.30}{-4}$  \\
& Test risk ($\widehat{\mathcal{R}}_{\mathrm{PDE}}$) & $\mstd{4.62}{0.48}{0}$ & $\mstd{4.62}{0.48}{0}$ & $\mstd{4.62}{0.48}{0}$ & $\mstd{4.72}{0.48}{0}$ \\
& Test risk ($\widehat{\mathcal{R}}_{\mathrm{IC}}$) & $\mstd{3.64}{0.29}{-4}$ & $\mstd{3.64}{0.29}{-4}$ & $\mstd{3.64}{0.29}{-4}$ & $\mstd{3.65}{0.30}{-4}$ \\
& Test risk ($\widehat{\mathcal{R}}_{\mathrm{BC}}$) & $\mstd{4.28}{0.37}{-4}$ & $\mstd{4.28}{0.37}{-4}$ & $\mstd{4.28}{0.37}{-4}$ & $\mstd{4.30}{0.37}{-4}$ \\
& Test risk (total) & $\mstd{4.62}{0.48}{0}$ & $\mstd{4.62}{0.48}{0}$ & $\mstd{4.62}{0.48}{0}$ & $\mstd{4.73}{0.48}{0}$ \\
\cmidrule{2-6}
& $\mathrm{KL}(\rho\Vert\pi)$ & $\mstd{1.77}{0.48}{-1}$ & $\mstd{1.77}{0.48}{-1}$ & $\mstd{1.87}{0.42}{-1}$ & $\mstd{1.70}{0.42}{-3}$ \\
& Bound & {\cellcolor{blue!28}$\mstd{9.16}{0.76}{0}$} & $\mstd{9.17}{0.76}{0}$ & $\mstd{1.08}{0.07}{1}$ & $\mstd{1.78}{0.15}{2}$ \\
& Bound w/o tail & $\mstd{5.34}{0.51}{0}$ & $\mstd{5.34}{0.51}{0}$ & - & - \\
\midrule
\multirow{12}{*}{\textbf{Data-scarce}} & Train risk ($\widehat{\mathcal{R}}_{d}$) & $\mstd{3.58}{0.31}{-4}$ & $\mstd{3.58}{0.31}{-4}$ & $\mstd{3.59}{0.31}{-4}$ & $\mstd{3.59}{0.31}{-4}$ \\
& Train risk ($\widehat{\mathcal{R}}_{\mathrm{PDE}}$) & $\mstd{4.66}{0.48}{0}$ & $\mstd{4.66}{0.48}{0}$ & $\mstd{4.78}{0.48}{0}$ & $\mstd{4.77}{0.48}{0}$ \\
& Train risk ($\widehat{\mathcal{R}}_{\mathrm{IC}}$) & $\mstd{3.64}{0.29}{-4}$ & $\mstd{3.64}{0.29}{-4}$ & $\mstd{3.65}{0.30}{-4}$ & $\mstd{3.65}{0.30}{-4}$  \\
& Train risk ($\widehat{\mathcal{R}}_{\mathrm{BC}}$) & $\mstd{4.29}{0.37}{-4}$ & $\mstd{4.29}{0.37}{-4}$ & $\mstd{4.31}{0.38}{-4}$ & $\mstd{4.31}{0.38}{-4}$ \\
& Train risk (total) & {\cellcolor{blue!28}$\mstd{4.67}{0.48}{0}$} & {\cellcolor{blue!28}$\mstd{4.67}{0.48}{0}$} & $\mstd{4.78}{0.48}{0}$ & $\mstd{4.77}{0.48}{0}$ \\
\cmidrule{2-6}
& Test risk ($\widehat{\mathcal{R}}_{d}$) & $\mstd{3.65}{0.30}{-4}$ & $\mstd{3.65}{0.30}{-4}$ & $\mstd{3.67}{0.30}{-4}$ & $\mstd{3.67}{0.30}{-4}$ \\
& Test risk ($\widehat{\mathcal{R}}_{\mathrm{PDE}}$) & $\mstd{4.62}{0.48}{0}$ & $\mstd{4.62}{0.48}{0}$ & $\mstd{4.73}{0.48}{0}$ & $\mstd{4.72}{0.48}{0}$ \\
& Test risk ($\widehat{\mathcal{R}}_{\mathrm{IC}}$) & $\mstd{3.64}{0.29}{-4}$ & $\mstd{3.64}{0.29}{-4}$ & $\mstd{3.65}{0.30}{-4}$ & $\mstd{3.65}{0.30}{-4}$ \\
& Test risk ($\widehat{\mathcal{R}}_{\mathrm{BC}}$) & $\mstd{4.28}{0.37}{-4}$ & $\mstd{4.28}{0.37}{-4}$ & $\mstd{4.31}{0.37}{-4}$ & $\mstd{4.30}{0.37}{-4}$ \\
& Test risk (total) & {\cellcolor{blue!28}$\mstd{4.62}{0.48}{0}$} & {\cellcolor{blue!28}$\mstd{4.62}{0.48}{0}$} & $\mstd{4.73}{0.48}{0}$ & $\mstd{4.73}{0.48}{0}$ \\
\cmidrule{2-6}
& $\mathrm{KL}(\rho\Vert\pi)$ & $\mstd{1.71}{0.43}{-1}$ & $\mstd{1.71}{0.43}{-1}$ & $\mstd{1.76}{0.26}{-4}$ & $\mstd{1.75}{0.32}{-3}$ \\
& Bound & {\cellcolor{blue!28}$\mstd{9.23}{0.76}{0}$} & {\cellcolor{blue!28}$\mstd{9.23}{0.76}{0}$} & $\mstd{3.34}{0.01}{2}$ & $\mstd{1.78}{0.15}{2}$ \\
& Bound w/o tail & $\mstd{5.34}{0.51}{0}$ & $\mstd{5.34}{0.51}{0}$  & - & - \\
\bottomrule
\end{tabular}
}
\end{table*}

\begin{table*}[h]
\centering
\small
\caption{PAC-Bayes benchmark results for Flow Mixing-2D in the balanced and data-scarce regimes.}
\label{tab:benchmark_flow_mixing}
\resizebox{0.85\linewidth}{!}{
\begin{tabular}{llcccc}
\toprule
\textbf{Data regime} & \textbf{Quantity} & \textbf{Ours-joint} & \textbf{Ours-union} & \textbf{Bounded-$U$} & \textbf{Sub-Gaussian} \\
\midrule
\multirow{12}{*}{\textbf{Balanced}} & Train risk ($\widehat{\mathcal{R}}_{d}$) & $\mstd{1.65}{0.12}{-3}$ & $\mstd{1.65}{0.12}{-3}$ & $\mstd{1.65}{0.12}{-3}$ & $\mstd{1.65}{0.12}{-3}$ \\
& Train risk ($\widehat{\mathcal{R}}_{\mathrm{PDE}}$) & $\mstd{1.13}{0.24}{-5}$ & $\mstd{1.13}{0.24}{-5}$ & $\mstd{1.13}{0.24}{-5}$ & $\mstd{1.13}{0.24}{-5}$ \\
& Train risk ($\widehat{\mathcal{R}}_{\mathrm{IC}}$) & $\mstd{1.34}{0.10}{-3}$ & $\mstd{1.34}{0.10}{-3}$ & $\mstd{1.34}{0.10}{-3}$ & $\mstd{1.34}{0.10}{-3}$ \\
& Train risk ($\widehat{\mathcal{R}}_{\mathrm{BC}}$) & $\mstd{8.95}{1.26}{-5}$ & $\mstd{9.01}{1.30}{-5}$ & $\mstd{9.09}{1.36}{-5}$ & $\mstd{9.12}{1.39}{-5}$ \\
& Train risk (total) & {\cellcolor{blue!28}$\mstd{3.09}{0.22}{-3}$} & {\cellcolor{blue!28}$\mstd{3.09}{0.22}{-3}$} & $\mstd{3.10}{0.22}{-3}$ & $\mstd{3.10}{0.22}{-3}$ \\
\cmidrule{2-6}
& Test risk ($\widehat{\mathcal{R}}_{d}$) & $\mstd{1.74}{0.12}{-3}$ & $\mstd{1.74}{0.12}{-3}$ & $\mstd{1.74}{0.12}{-3}$ & $\mstd{1.74}{0.12}{-3}$ \\
& Test risk ($\widehat{\mathcal{R}}_{\mathrm{PDE}}$) & $\mstd{1.16}{0.24}{-5}$ & $\mstd{1.16}{0.24}{-5}$ & $\mstd{1.16}{0.24}{-5}$ & $\mstd{1.16}{0.24}{-5}$ \\
& Test risk ($\widehat{\mathcal{R}}_{\mathrm{IC}}$) & $\mstd{1.38}{0.11}{-3}$ & $\mstd{1.38}{0.11}{-3}$ & $\mstd{1.38}{0.11}{-3}$ & $\mstd{1.38}{0.11}{-3}$ \\
& Test risk ($\widehat{\mathcal{R}}_{\mathrm{BC}}$) & $\mstd{8.96}{1.26}{-5}$ & $\mstd{9.01}{1.30}{-5}$ & $\mstd{9.09}{1.36}{-5}$ & $\mstd{9.13}{1.39}{-5}$ \\
& Test risk (total) & $\mstd{3.22}{0.23}{-3}$ & $\mstd{3.22}{0.23}{-3}$ & $\mstd{3.22}{0.23}{-3}$ & $\mstd{3.22}{0.23}{-3}$ \\
\cmidrule{2-6}
& $\mathrm{KL}(\rho\Vert\pi)$ & $\mstd{1.05}{0.78}{-3}$ & $\mstd{5.77}{4.13}{-4}$ & $\mstd{6.54}{4.58}{-5}$ & $\mstd{5.74}{4.42}{-6}$ \\
& Bound & {\cellcolor{blue!28}$\mstd{3.98}{0.22}{-2}$} & $\mstd{5.58}{0.33}{-2}$ & $\mstd{4.51}{0.01}{-1}$ & $\mstd{1.44}{0.09}{0}$ \\
& Bound w/o tail & $\mstd{7.98}{0.31}{-3}$ & $\mstd{8.89}{0.38}{-3}$ & - & - \\
\midrule
\multirow{12}{*}{\textbf{Data-scarce}} & Train risk ($\widehat{\mathcal{R}}_{d}$) & $\mstd{1.47}{0.17}{-3}$ & $\mstd{1.47}{0.17}{-3}$ & $\mstd{1.47}{0.17}{-3}$ & $\mstd{1.47}{0.17}{-3}$ \\
& Train risk ($\widehat{\mathcal{R}}_{\mathrm{PDE}}$) & $\mstd{1.13}{0.24}{-5}$ & $\mstd{1.13}{0.24}{-5}$ & $\mstd{1.13}{0.24}{-5}$ & $\mstd{1.13}{0.24}{-5}$ \\
& Train risk ($\widehat{\mathcal{R}}_{\mathrm{IC}}$) & $\mstd{1.34}{0.10}{-3}$ & $\mstd{1.34}{0.10}{-3}$ & $\mstd{1.34}{0.10}{-3}$ & $\mstd{1.34}{0.10}{-3}$ \\
& Train risk ($\widehat{\mathcal{R}}_{\mathrm{BC}}$) & $\mstd{9.13}{1.39}{-5}$ & $\mstd{9.13}{1.39}{-5}$ & $\mstd{9.14}{1.39}{-5}$ & $\mstd{9.14}{1.39}{-5}$ \\
& Train risk (total) & $\mstd{2.91}{0.27}{-3}$ & $\mstd{2.91}{0.27}{-3}$ & $\mstd{2.91}{0.27}{-3}$ & $\mstd{2.91}{0.27}{-3}$ \\
\cmidrule{2-6}
& Test risk ($\widehat{\mathcal{R}}_{d}$) & $\mstd{1.74}{0.12}{-3}$ & $\mstd{1.74}{0.12}{-3}$ & $\mstd{1.74}{0.12}{-3}$ & $\mstd{1.74}{0.12}{-3}$ \\
& Test risk ($\widehat{\mathcal{R}}_{\mathrm{PDE}}$) & $\mstd{1.16}{0.24}{-5}$ & $\mstd{1.16}{0.24}{-5}$ & $\mstd{1.16}{0.24}{-5}$ & $\mstd{1.16}{0.24}{-5}$ \\
& Test risk ($\widehat{\mathcal{R}}_{\mathrm{IC}}$) & $\mstd{1.38}{0.11}{-3}$ & $\mstd{1.38}{0.11}{-3}$ & $\mstd{1.38}{0.11}{-3}$ & $\mstd{1.38}{0.11}{-3}$ \\
& Test risk ($\widehat{\mathcal{R}}_{\mathrm{BC}}$) & $\mstd{9.14}{1.40}{-5}$ & $\mstd{9.14}{1.40}{-5}$ & $\mstd{9.14}{1.40}{-5}$ & $\mstd{9.14}{1.40}{-5}$  \\
& Test risk (total) & $\mstd{3.22}{0.23}{-3}$ & $\mstd{3.22}{0.23}{-3}$ & $\mstd{3.22}{0.23}{-3}$ & $\mstd{3.22}{0.23}{-3}$ \\
\cmidrule{2-6}
& $\mathrm{KL}(\rho\Vert\pi)$ & $\mstd{3.13}{1.06}{-6}$ & $\mstd{3.15}{1.05}{-6}$ & $\mstd{3.37}{1.69}{-6}$ & $\mstd{3.50}{1.55}{-6}$ \\
& Bound & {\cellcolor{blue!28}$\mstd{1.21}{0.16}{1}$} & {\cellcolor{blue!28}$\mstd{1.21}{0.16}{1}$} & $\mstd{6.51}{0.00}{1}$ & $\mstd{4.19}{0.53}{1}$  \\
& Bound w/o tail & $\mstd{7.21}{0.57}{-2}$ & $\mstd{7.22}{0.57}{-2}$ & - & - \\
\bottomrule
\end{tabular}
}
\end{table*}

\subsection{Ablation Study} \label{apdx:ablation}
\subsubsection{Impact of the posterior's training sample size}
We investigate the effect of the observational posterior-training sample size $m_d^{\mathrm{posterior}}$ on Heat-2D MS by varying it from $10$ to the full $20$k samples, while retaining the complete calibration set and all physical samples, thereby isolating the effect of posterior-training data scarcity. As shown in Fig.~\ref{fig:ablation_m_d_posterior}, all certificates become tighter as $m_d^{\mathrm{posterior}}$ increases, with \emph{ours-joint} and \emph{ours-union} improving by more than two orders of magnitude and remaining substantially tighter than the conservative baselines. The gradient-tail contribution accounts for much of the looseness at very small $m_d^{\mathrm{posterior}}$ and rapidly decreases with additional samples, while the total test risk reduces only moderately (but still more notable on ours bounds). This contrast shows that limited observational data have a much stronger effect on the tightness of the certificate than on predictive accuracy: the physics-informed training can still yield an accurate solution, while the small observational sample leads to substantially larger finite-sample generalisation penalties.
\begin{figure}[ht]
    \centering
    \begin{subfigure}{\linewidth}
        \centering
        \includegraphics[width=0.8\linewidth]{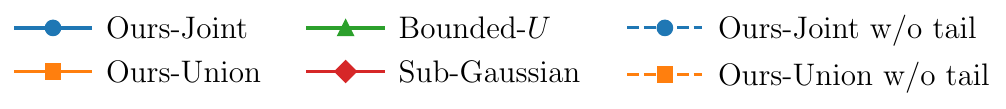}
    \end{subfigure}
    \begin{subfigure}{0.32\linewidth}
        \includegraphics[width=\linewidth]{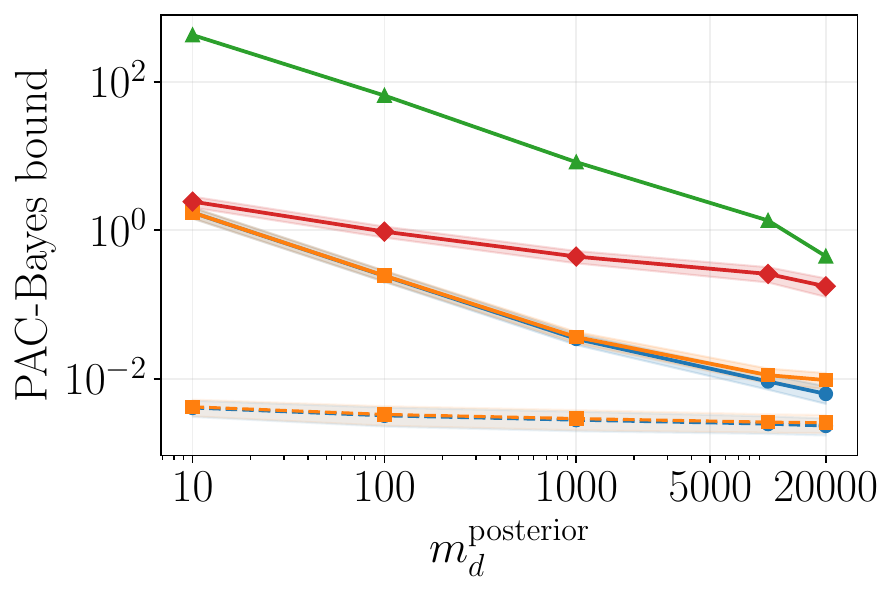}
        \caption{Bound value}
    \end{subfigure}
    \begin{subfigure}{0.32\linewidth}
        \includegraphics[width=\linewidth]{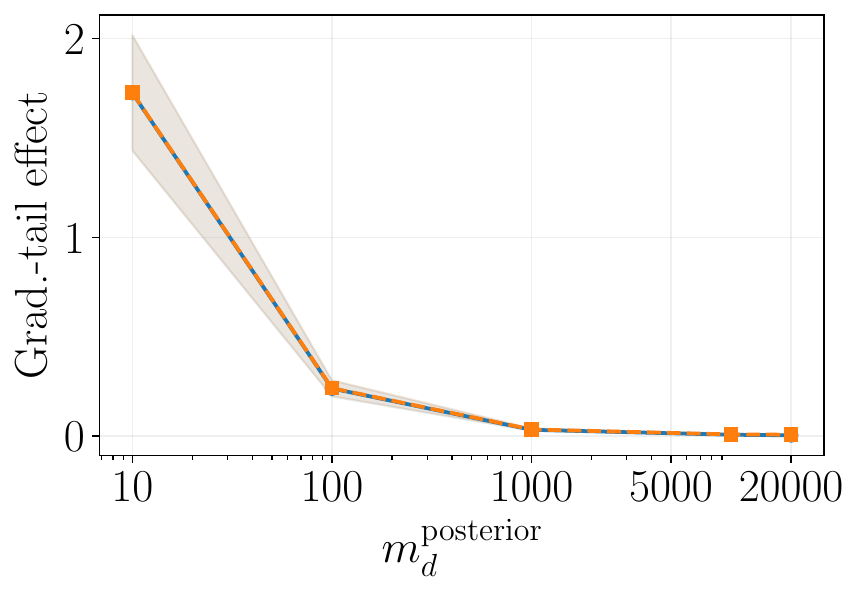}
        \caption{Gradient tail effect}
    \end{subfigure}
    \begin{subfigure}{0.32\linewidth}
        \includegraphics[width=\linewidth]{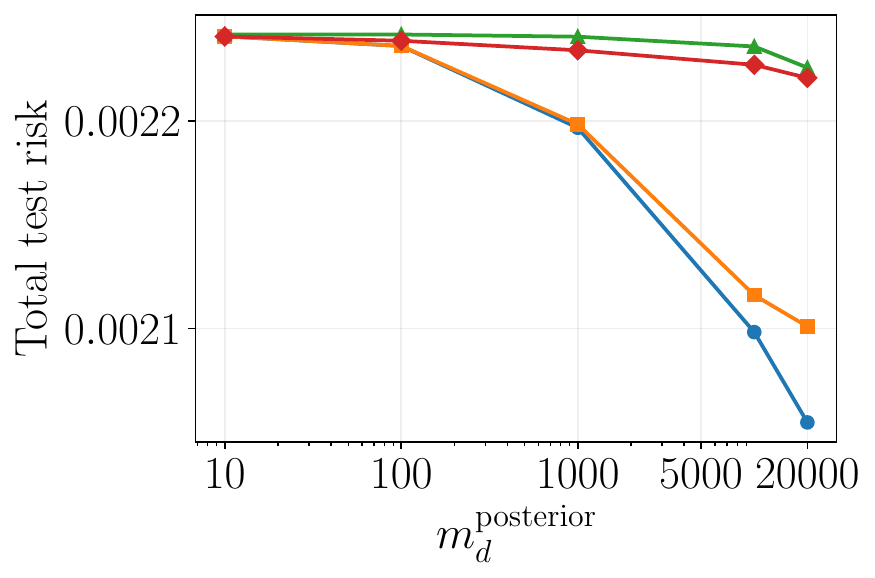}
        \caption{Total test risk}
    \end{subfigure}
    \caption{Impact of the number of observational posterior training samples $m_d^{\mathrm{posterior}}$.}
    \label{fig:ablation_m_d_posterior}
\end{figure}

\subsubsection{Impact of the calibration's sample size}
We next study the effect of the observational calibration sample size $m_d^{\mathrm{calib}}$ on Heat-2D MS by varying it from $10$ to the full $10$k samples, while fixing $m_d^{\mathrm{posterior}}=100$ and retaining all physical samples. For this ablation, the calibrated coefficient $\overline C_d$ is shared by all compared input-gradient bounds, allowing us to isolate the effect of calibration quality independently of posterior-training sample size. As shown in Fig.~\ref{fig:ablation_m_d_calib}, increasing $m_d^{\mathrm{calib}}$ substantially decreases $\overline C_d$, particularly in the small-sample regime, and leads to a corresponding tightening of all certificates. The improvement is most pronounced between $10$ and the first few hundred calibration samples and gradually saturates thereafter, consistently with the convergence of $\overline C_d$ towards a stable scale. The gradient tail's contribution also follows the same variation. The similar trend across the proposed and sub-Gaussian bounds further confirms that the improvement originates primarily from the increasingly informative (tighter) calibration coefficient rather than from differences in their posterior-training complexities.

\begin{figure}[h]
    \centering
    \begin{subfigure}{\linewidth}
    \centering
        \includegraphics[width=\linewidth]{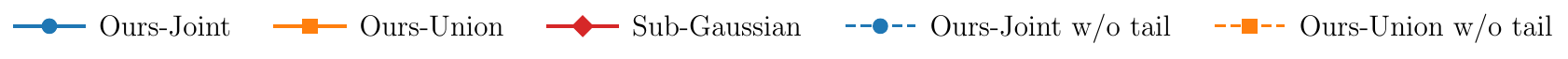}
    \end{subfigure} 
     \par\medskip
    \begin{subfigure}{0.32\linewidth}
        \includegraphics[width=\linewidth]{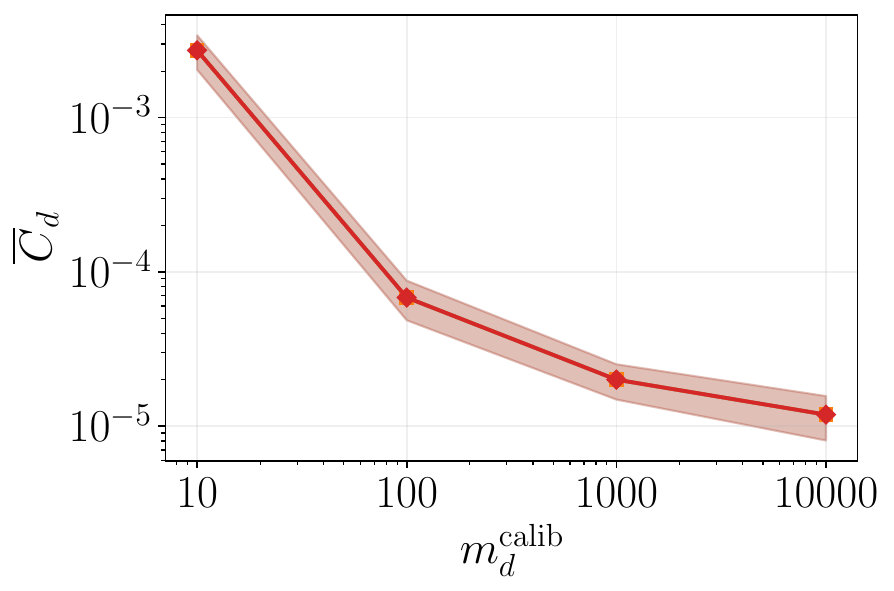}
        \caption{Calibrated $\overline C_d$ value}
    \end{subfigure}
    \begin{subfigure}{0.32\linewidth}
        \includegraphics[width=\linewidth]{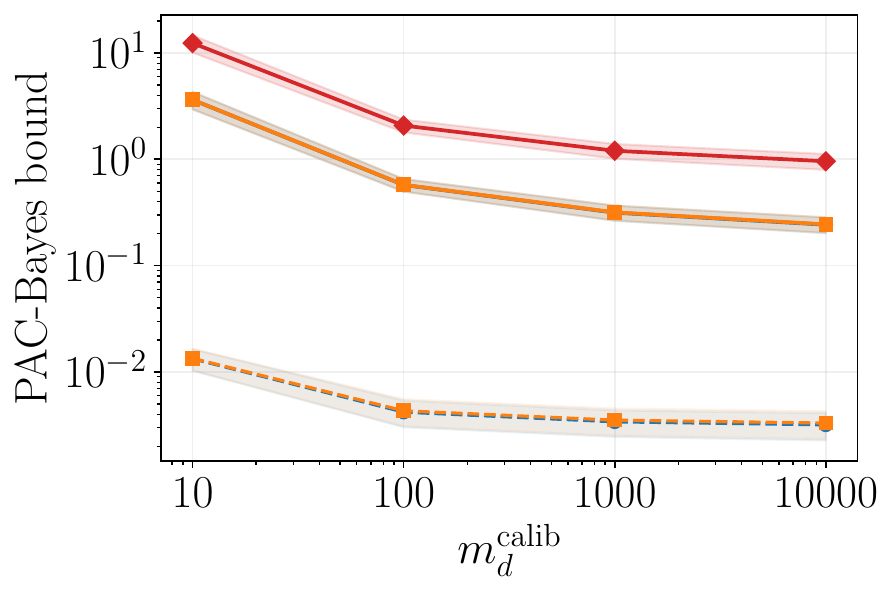}
        \caption{Bound value}
    \end{subfigure}
    \begin{subfigure}{0.32\linewidth}
        \includegraphics[width=\linewidth]{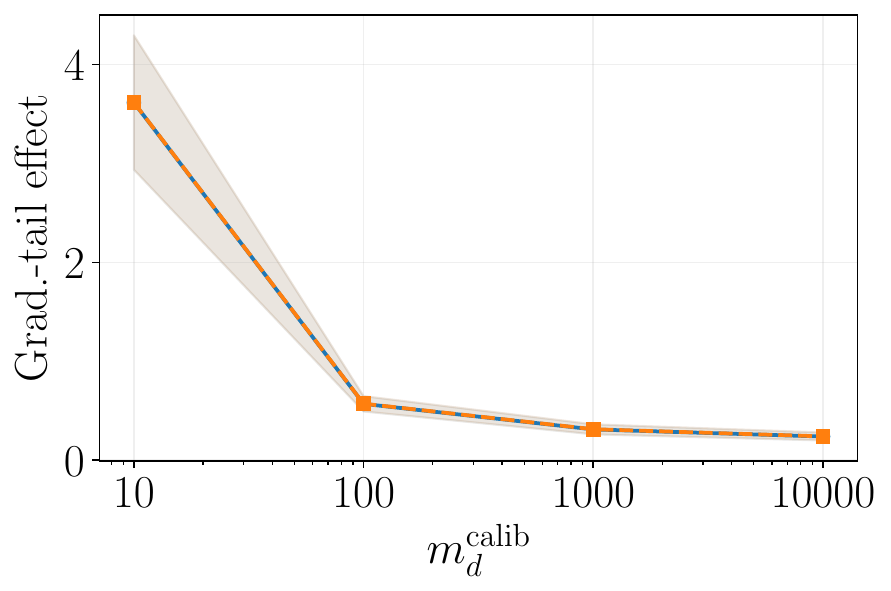}
        \caption{Gradient tail effect}
    \end{subfigure}
    \caption{Impact of the number of observational calibration sample size $m_d^{\mathrm{calib}}$.}
    \label{fig:ablation_m_d_calib}
\end{figure}

\subsubsection{Impact of the gradient envelope}
We further assess the sensitivity to the prescribed gradient envelope by varying $U_{G,\mathrm{factor}}\in\{2\times10^6,10^7,10^8\}$ on Heat-2D MS, while retaining the full posterior-training and calibration datasets and fixing all remaining hyperparameters. As shown in Fig.~\ref{fig:ablation_U_G_factor}, increasing $U_{G,\mathrm{factor}}$ progressively loosens all certificates, with a particularly strong deterioration of the sub-Gaussian bound, which directly uses $U_{G,i}$ to control the complete input-gradient energy. For our bounds, this dependence arises primarily through the residual gradient-tail correction: the oracle certificates without this contribution remain nearly unchanged as $U_{G,\mathrm{factor}}$ increases, whereas the tail effect grows notably with the prescribed envelope. This confirms the benefit of restricting the conservative gradient envelope to the residual-tail control rather than using it as the concentration scale of the entire input-gradient complexity, while also motivating the use of the smallest conservative envelope supported by the empirical diagnostics.
\begin{figure}[ht]
    \centering
    \begin{subfigure}{0.9\linewidth}
    \centering
        \includegraphics[width=\linewidth]{figures/heat2d_ms/ablation_common_legend_m_d_calib.pdf}
    \end{subfigure} 
     \par\medskip
    \begin{subfigure}{0.45\linewidth}
        \includegraphics[width=\linewidth]{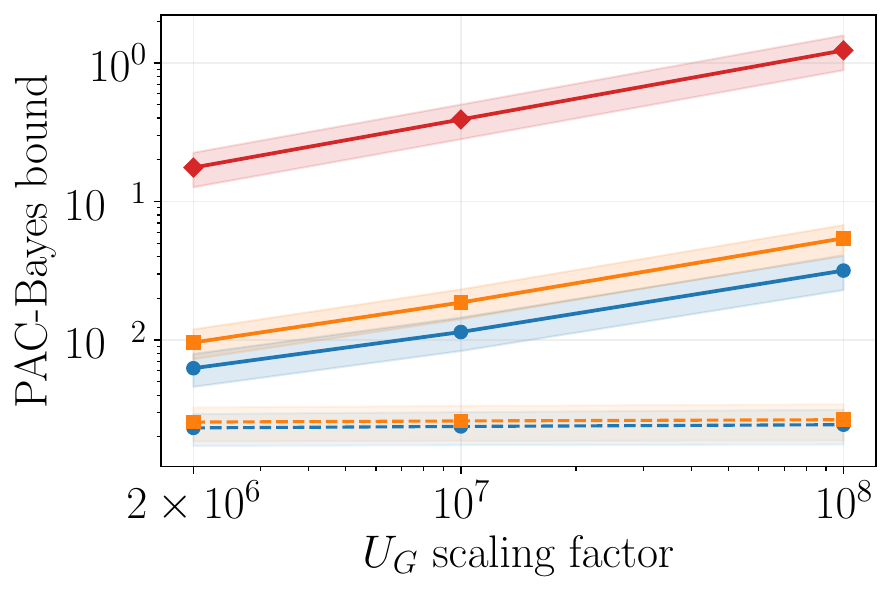}
        \caption{Bound value}
    \end{subfigure}
    \begin{subfigure}{0.45\linewidth}
        \includegraphics[width=\linewidth]{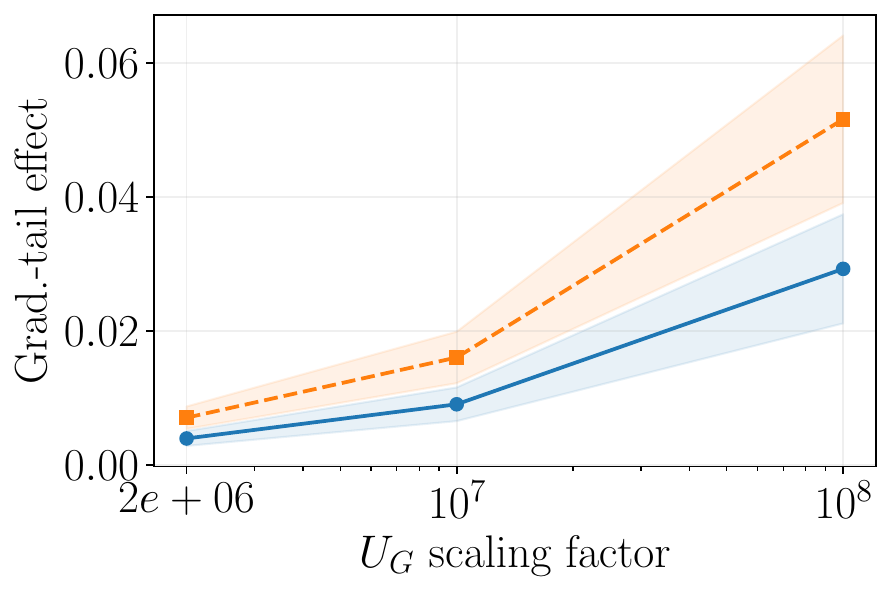}
        \caption{Gradient tail effect}
    \end{subfigure}
    \caption{Impact of the gradient enveloppe through $U_G$'s scaling factor ($U_{G,\mathrm{factor}}$).}
    \label{fig:ablation_U_G_factor}
\end{figure}

\subsubsection{Impact of the localisation radius}
We finally investigate the effect of the localisation radius $R_{\mathrm{loc}}=r_{\mathrm{loc}}\sigma\sqrt{d_\theta}$ on Heat-2D MS by varying $r_{\mathrm{loc}}\in\{1,3,5\}$ while retaining the full posterior-training and calibration datasets. As shown in Fig.~\ref{fig:ablation_radius_factor}, increasing the localisation radius substantially reduces $\KL(\rho\|\pi)$, but simultaneously enlarges the proposed certificates through the input-gradient complexity, as models farther from the well-trained region generally exhibit larger input-gradient magnitudes. One can see that the gradient-tail contribution increasing markedly over the considered range, in particular from $3$ to $5$. Accordingly, the oracle bounds without the tail remain nearly unchanged, whereas the complete \emph{ours-joint}, \emph{ours-union} and sub-Gaussian certificates progressively increase with $r_{\mathrm{loc}}$. In contrast, Bounded-\(U\) is comparatively insensitive to the localisation radius because it does not depend on the observed input-gradient magnitudes.
\begin{figure}[ht]
    \centering
    \begin{subfigure}{0.7\linewidth}
    \centering
        \includegraphics[width=\linewidth]{figures/heat2d_ms/ablation_common_legend.pdf}
    \end{subfigure} 
    \begin{subfigure}{0.4\linewidth}
        \includegraphics[width=\linewidth]{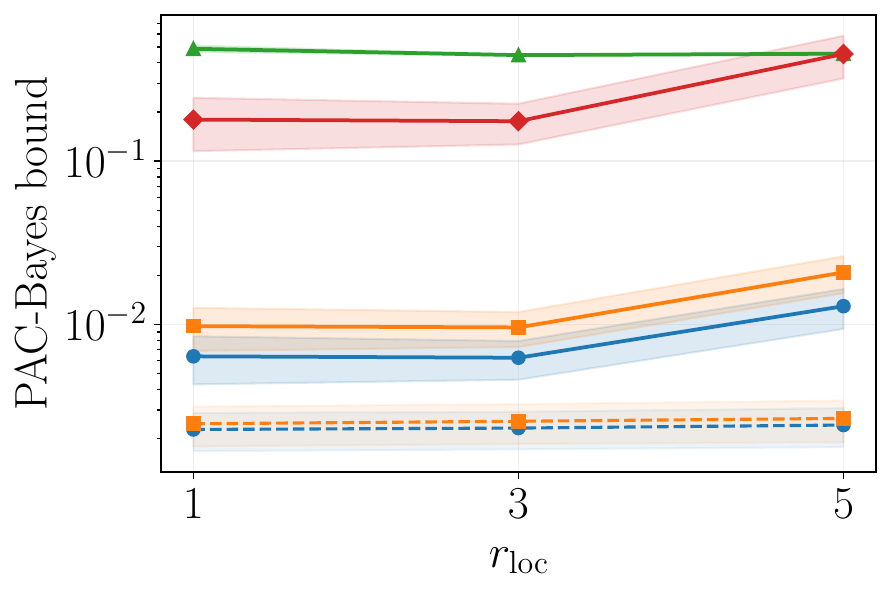}
        \caption{Bound value}
    \end{subfigure}
     \begin{subfigure}{0.4\linewidth}
        \includegraphics[width=\linewidth]{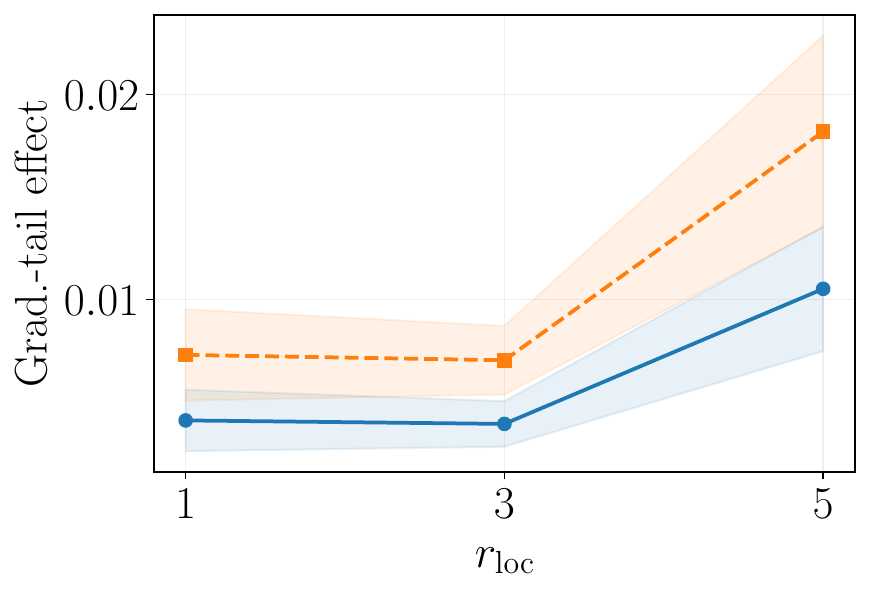}
        \caption{Gradient tail effect}
    \end{subfigure}
    \caption{Impact of the truncation radius through the radius factor $r_{\mathrm{loc}}$.}
    \label{fig:ablation_radius_factor}
\end{figure}

\section{Limitations}\label{sa:limitations}
This paper presents a PAC-Bayesian analysis for physics-informed machine learning, providing both generalisation guarantees and practical algorithms supported by extensive experiments on six benchmarks. Despite promising results, including substantially tighter certificates than the considered baselines and intuitive empirical observations, several limitations remain. Regarding the theoretical analysis, the use of a gradient envelope to control the residual gradient contribution can still lead to deteriorated bounds, particularly in data-scarce regimes and when the underlying solution is highly irregular. Therefore, deriving sharper finite-sample control of the residual gradient represents a promising direction for further tightening the resulting certificates. In addition, developing a calibration procedure for Theorem~\ref{thm:variance_gradient_calibration} that avoids the use of a union bound could be particularly advantageous in data-scarce scenario. Besides, it remains unclear how the incorporation of physical constraints affects the generalisation behavior of the data-fidelity loss, which constitutes an important direction for future theoretical investigation. From an experimental perspective, an important direction is to investigate how the proposed bounds scale to more complex multi-output problems (\eg the Navier-Stokes equations), where several potentially high-dimensional and coupled PDE residual terms must be handled simultaneously. Another promising direction for mitigating the gradient-envelope issue is to enforce input-gradient regularity by construction, for instance through Lipschitz-constrained neural networks, and to investigate whether such architectures can yield tighter certificates in an end-to-end manner, from learning a physics-informed prior under these constraints to directly minimising the final PIML bound.  

\section{Related works}\label{sa:relatedwork}
\paragraph{Generalisation of PIML.} Despite tremendous empirical success, the generalisation of physics-informed machine learning (PIML) methods remains an open question that requires further study. A common line of work focuses on the convergence of physics informed models to the PDE solution \citep{Shin_MLMC22,DoumecheBernouilli25}. In particular, \citep{DoumecheBernouilli25} addresses overfitting of PINNs by introducing a Sobolev regularisation term to enforce estimator regularity and show consistency and strong convergence properties. Other works derive approximation error bounds under stability assumptions on the underlying PDEs \citep{DBLP:conf/nips/RyckM22,Mishra22_IJNMA}. Several papers study generalisation via (local) Rademacher complexity \citep{Jiao_CP22,LuICLRB22,XuLiHuangICML2025}; notably, \citep{XuLiHuangICML2025} obtains generalisation bounds for second order elliptic PDEs by adopting a multi-task perspective and assuming the underlying solutions lie in Barron or Sobolev spaces. \citep{DBLP:conf/colt/DoumecheBBB24,DBLP:journals/jmlr/Doumeche0BB25} treat the PIML problem as a kernel regression task and use Fourier methods to construct tractable estimators with quantified convergence rates. \citep{scampicchio2026physicsinformed} views physical loss terms as a form of regularisation in addition to data-fidelity loss and shows that, under a knowledge alignment assumption (i.e., the regulariser is approximately zero at the true solution), the estimator converges at a faster rate. Together, these works highlight the importance of encoding prior knowledge about the physical system—via regularity, stability, or data-physics alignment—to establish generalisation guarantees for approximation error or the total PIML loss. In the same vein, we connect generalisation, data, and physical regularity through a PAC-Bayesian framework that yields computable input-gradient complexity terms, and empirically show that the resulting certificates can be meaningfully evaluated across PDE benchmarks and remain substantially tighter than representative alternatives under observational data scarcity.

\paragraph{PAC-Bayesian analysis for PIML.} PAC-Bayes theory has emerged as a powerful framework for analysing modern machine learning models through data-dependent probabilistic guarantees that balance empirical risk and model complexity via information-theoretic divergences~\citep{Alquier24-PB,guedj2019primer,hellstrom2023generalisation}. While early PAC-Bayesian analyses primarily focused on bounded-loss classification settings, more recent works have extended the framework to regression and heavy-tailed losses through higher-order moment control~\citep{haddouche2023pacbayes,NEURIPS2019_3a20f62a}, assumptions on the CGF~\citep{casado2024pacbayeschernoff,JMLR:v25:23-1360}, and structural assumptions on the data distribution~\citep{Shalaeva2019ImprovedPB,guo2025pacbayes}. These advances are particularly relevant for PIML, where losses induced by PDE residuals or physical constraints are typically unbounded and may exhibit complex tail behaviour. Moreover, PAC-Bayes naturally accommodates prior physical knowledge through structured priors and constrained hypothesis spaces. More recently, model-dependent analyses have connected generalisation to regularity measures based on parameter norms~\citep{NIPS2009_250cf8b5}, parameter gradients~\citep{haddouche:hal-04455639}, and input gradients~\citep{gat2022on,casado2024pacbayeschernoff}. In particular, input-gradient information provides a natural connection between statistical generalisation and the regularity of learned functions, and has been incorporated into PAC-Bayesian bounds under different concentration assumptions~\citep{gat2022on,casado2024pacbayeschernoff}. Building on this broader perspective, we develop a PIML-specific framework based on a population CGF-to-gradient control, together with a PAC-Bayes calibration procedure and a tail-aware decomposition that yields computable input-gradient complexities without bounding the original losses. We further exploit the multi-loss structure of PIML to derive a refined joint certificate with a single KL complexity, and introduce a bound-aware learning procedure for optimising the resulting guarantees. To the best of our knowledge, our work is among the first to provide an end-to-end PAC-Bayesian framework for both learning and certifying PIML models in the presence of potentially unbounded losses.


\end{document}